\documentclass[twoside,11pt]{article}

\usepackage{jmlr2e}
\usepackage[colorinlistoftodos]{todonotes}

\newcommand{\comment}[1]{}
\usepackage{tikz}
\usetikzlibrary{arrows.meta}
\usetikzlibrary{arrows}
\usepackage{amsmath}
\usepackage[mathscr]{euscript}
\usepackage{bm}
\usepackage{esvect}
\usepackage{amssymb}
\usepackage{etoolbox}
\usepackage[mathscr]{euscript}
\usepackage{amsmath}
\usepackage{bbm}
\usepackage{algorithm}
\usepackage[noend]{algpseudocode}
\usepackage{caption}
\usepackage{url}
\usepackage{multirow}

\usepackage{comment}
\usetikzlibrary{shapes,snakes}

\newcommand{\SET}[1]{\mathbf{#1}}
\renewcommand{\v}{v}				
\newcommand{\V}{\mathbf{V}}		

\newcommand{\E}{\mathbf{E}}

\renewcommand{\line}{\;\rule[.4ex]{.9em}{.5pt}\;}

\newcommand{\Rcup}{\sqcup}

\firstpageno{1}

\begin{document}


\title{Fast Causal Orientation Learning in Directed Acyclic Graphs}

\author{\name Ramin Safaeian  \email ramin.safaeian@ee.sharif.edu \\
       \addr Department of Electrical Engineering\\
       Sharif University of Technology\\
       Tehran, Iran
       \AND
       \name Saber Salehkaleybar \email saleh@sharif.edu \\
       \addr Department of Electrical Engineering\\
       Sharif University of Technology\\
       Tehran, Iran
       \AND
       \name Mahmoud Tabandeh \email tabandeh@sharif.edu \\
       \addr Department of Electrical Engineering\\
       Sharif University of Technology, \\ Tehran, Iran}


\maketitle

\begin{abstract}
Causal relationships among a set of variables are commonly represented by a directed acyclic graph. The orientations of some edges in the causal DAG can be discovered from observational/interventional data. Further edges can be oriented by iteratively applying so-called Meek rules. Inferring edges' orientations from some previously oriented edges, which we call Causal Orientation Learning (COL), is a common problem in various causal discovery tasks. In these tasks, it is often required to solve multiple COL problems and therefore applying Meek rules could be time consuming. Motivated by Meek rules, we introduce Meek functions that can be utilized in solving COL problems. In particular, we show that these functions have some desirable properties, enabling us to speed up the process of applying Meek rules. In particular, we propose a dynamic programming (DP) based method to apply Meek functions. Moreover, based on the proposed DP method, we present a lower bound on the number of edges that can be oriented as a result of intervention. We also propose a method to check whether some oriented edges belong to a causal DAG. Experimental results show that the proposed methods can outperform previous work in several causal discovery tasks in terms of running-time.

\end{abstract}


\section{Introduction}

Recovering causal relationships from data is one of the ultimate goal in the empirical sciences. The behavior of a complex system is not fully understood unless one can infer how different variables in the system influence each other. Without such complete understanding, it might be infeasible to predict how the system behaves in response to some target interventions.

In the literature, the causal relationships among a set of variables in a system are commonly represented by a directed acyclic graph (DAG) where there is a directed edge from variable $X$ to variable $Y$ if $X$ is the direct cause of $Y$. From observational data, under causal sufficiency and faithfulness assumptions, the true causal graph can be identified up to a Markov equivalence class (MEC) which is the set of all DAGs representing the same set of conditional independences among
the variables. It has been shown that all DAGs in an MEC has the same skeleton\footnote{Two directed graphs have the same skeleton
	if they have the same set of vertices and edges regardless of
	their orientations.} and the set of v-structures\footnote{Three variables $X,Y,$ and $Z$ form a v-structure in a graph if $X$ and $Y$ are the direct cause of $Z$ and they are not connected to each other.} \citep{Pearl92}. Given the skeleton and set of v-structures, \cite{Meek95} proposed four rules, called Meek rules, to orient further edges in the graph by making no directed cycles or new v-structures.  These rules are applied iteratively on the causal graph until no further edges can be oriented by applying any of these rules.  However, some edges might be remained undirected even after applying Meek rules. The resulted graph is commonly called the essential graph which represented a summary of all DAGs in MEC in the sense that if there is a directed edge from $X$ to $Y$ in the essential graph, the same orientation holds in all DAGs in MEC. Otherwise, there exist at least two DAGs with the opposite directions of edge between $X$ and $Y$.

 In order to identify the whole causal structure, the golden standard procedure is to perform interventions in the system. For instance, by performing perfect randomized intervention on a variable, we are able to recover orientations of all edges incident with the intervened variable \citep{He08}. Based on these new oriented edges, we can apply Meek rules again to discover further edges' orientations. 

In both scenarios mentioned above (obtaining essential graph from observational data or discovering orientation of undirected edges from interventional data), we try to solve the problem of inferring new edges' orientations based on our current knowledge about the causal structure which might come from observational data, interventional data, or some prior knowledge. The discovered orientations should not make any new cycle or v-structure in the causal graph. We call this problem, ``causal orientation learning (COL)" problem. Meek rules can be utilized as a basis to solve COL problem. However, in various applications of causal structure learning (see the related work), it is required to solve several COL problems. Unfortunately, executing Meek rules for solving multiple COL problems might be time-consuming since these rules are applied for each problem separately in an iterative manner. Moreover, in some of the applications, the COL problems have structural similarities and the solution for one of the problems can be reused in other ones. In this paper, our main goal is to devise a solution to solve multiple COL problems efficiently in time. Moreover, the proposed solution can be used as a subroutine in various settings for causal structure learning. In the next part, we review some of these settings in which COL problem arises along with the solutions that have been proposed in the literature. Two main problems that we review here, are counting size of MEC and experiment design.

First, we consider the problem of counting size of MEC. The number of DAGs in an MEC can give an indication of the ``causal complexity" of the underlying true causal structure. For instance, in order to determine the range of possible causal strengths of  variables on a target variable, in one of the algorithms in \citep{Maathuis2009}, it has been proposed to enumerate all DAGs in an MEC and estimate the causal strength of each parent of the target variable in the corresponding DAG. Herein, the size of MEC shows the time complexity of the estimation procedure. As another example, \cite{He08} proposed multiple scores, based on the possible sizes of an MEC after performing intervention on each variable. The variable with the highest score is selected for doing intervention (we will discuss this problem in more detail in the next part). In fact, remaining too many DAGs in the resulted MEC after performing an intervention indicates the candidate variable is not suitable to be intervened on.
	
	\cite{He15} proposed the first efficient algorithm for counting the size of MEC instead of enumerating all DAGs. They first presented five closed-form formulas that can be used to calculate size of MECs having some specific relations between the number of nodes and edges. Next, they showed that counting size of an MEC can be partitioned into smaller sub-classes and the size of each of them can be computed in an iterative manner. In particular, the components corresponding to sub-classes (which are commonly called chain components) are obtained by removing the directed edges from the essential graph of MEC. In order to find the size of each sub-class, it is required to consider each variable in the corresponding component to be root and solve a COL problem to obtain further chain components in it. This procedure should be done recursively on each variable in the resulted components. Thus, we need to solve several COL problems and it may become time-consuming as the number of nodes increases. \cite{He15} proposed an efficient algorithm to solve each COL problem separately  which resembles applying Meek rules carefully according to properties of sub-classes. However, this solution still needs to solve multiple COL problems and does not reuse the results of a COL problem in the another one. \cite{He16} introduced a special structure of ``core graphs" and showed that the size of MEC is a polynomial of number of vertices given its core graph. Moreover, they proposed an algorithm in order to obtain this polynomial. Inside this algorithm, similar to their previous work, it is required to solve multiple COL problems in order to find chain components in the essential graphs. \cite{Saleh19} and \cite{Talviti19} proposed a dynamic programming based solution in order to compute the size of MEC. In this solution, whenever size of a chain component is computed, it is stored in memory and will be used later if we encounter the same chain component. In both works, it is necessary to solve multiple COL problems to obtain chain components and clique trees were utilized to efficiently orient edges in each COL problem.
	
Another important problem related to COL problem is experiment design problem. The whole causal structure can be recovered with sufficient number of interventions if it is feasible to intervene in the system. However, in most applications, performing interventions is too time-consuming or costly. The problem of experiment design is to devise a set of experiments where each experiment consists of multiple interventions performed simultaneously. The maximum number of interventions in each experiment is usually set to a fixed value (in some previous work, this value is equal to one \citep{Hauser14,He08,Saleh18}). The experiment design problem can be studied in two settings of passive or active learning. In the active setting, the experiments are performed sequentially and the result of an experiment is used in designing latter ones. More specifically, it is commonly assumed that the orientations of all edges incident with an intervened variable can be recovered by performing perfect randomized interventions. Based on these newly oriented edges, we can solve a COL problem to find resulted chain components after doing intervention. The goal is to recover the whole causal structure with minimum number of interventions.
	 In contrary, in the passive setting, we have a limited budget for performing experiments. All the experiments are designed based on the MEC obtained from observational data and performed in parallel. According to the results of experiments, we solve a COL problem to orient as much edges as possible. Here, the goal is to design a set of limited number of experiments (the number of interventions in each experiment is usually set to one) to recover the most part of the causal structure. 
	 
	 In the passive setting, \cite{He08} proposed a naive approach to enumerate all DAGs in an MEC and select a set of variables for interventions with minimum size such that all possible underlying DAGs can be recovered from the results of interventions after solving a COL problem. \cite{Saleh18} proposed a greedy algorithm to select the target interventions by sampling DAGs from the MEC and estimating the average number of edges oriented as a result of intervention. Again, in the sampling procedure, it is required to solve several COL problems similar to counting size of MECs. \cite{Kocaoglu17} proposed an algorithm for experiment design where there exists a cost for intervening each variable. They showed that the optimal solution can be obtained in polynomial time in the case of causal strucutre being a tree or a clique tree.

	 In the active setting, \cite{He08} introduced an entropy based score to select the intervened variable in each step. In this score, it is required to compute the size of MEC for any possible orientations of edges incident with a given variable. \cite{Hauser14} proposed a minimax criterion in order to select the target intervention. 
	 First, for any given variable, it is required to obtain the maximum number of undirected edges in the essential graphs that are consistent with different possible orientations of edges incident with the given variable. Next, from these variables, the one that minimizes this criterion is selected. In order to compute this criterion, \cite{Hauser14} presented an algorithm where, for any feasible orientations of edges connected to a given variable, the number of oriented edges as a result of these newly oriented edges is obtained by solving a COL problem through running LexBFS algorithm. In fact, one can utilize LexBFS algorithm to find an ordering over a chain component and orient edges based on this ordering without making directed cycles or v-structures. \cite{Shanmugam15} proposed a score based on graph coloring and separating system in order to select the node for the intervention. After observing the result of an intervention, a COL problem is solved to orient further edges by simply applying Meek rules. \cite{Ness17} focused on Bayesian approach to learn the causal structure based on a prior probability distribution on the underlying graph. They introduced an information gain to select the variables for performing interventions. 
	  Similarly, \cite{Uhler19} studied the problem of experiment design from Bayesian perspective considering limited number of samples. They proposed a greedy algorithm with guarantee on approximation quality. \cite{Kristjan19} also considered a Bayesian approach with the assumption of all chain components being trees. They proposed an algorithm with an approximation ratio of $2$, in the sense that the number of interventions is at most twice the minimum achievable number. \cite{Teshnizi20} proposed an efficient algorithm, called LazyIter, for iterating over MECs for all possible orientations of edges incident with a variable. Hence, it is required to solve multiple COL problems simultaneously. The authors exploited the similarities among possible essential graphs to avoid the recalculation of edges' orientations for each of them separately. The proposed algorithm has been utilized to solve problems of counting size of MEC and experiment design.
Besides counting size of MEC and experiment design problems, there exist some other related works to COL problem. For instance, \cite{Chickering1995} proposed an algorithm to recover oriented edges in the MEC of a DAG, given its skeleton and v-structures. \cite{Ramsey2016} considered a limited version of faithfulness assumption and presented a search algorithm to considerably speed up the process of finding the essential graph from the samples,  at the expense of allowing graphs to violate Markov factorization, i.e., for the conditional dependence and independence relations estimated from the samples. They also evaluated this model in large graphs with multi-million nodes.


The contributions of this paper are threefold:
\begin{itemize}
	\item We introduce rigorous mathematical representations of Meek rules, which we call Meek functions. Equipped with these functions, we present key properties of these functions that enable us to execute them separately on a causal graph and aggeragate their outputs in order to obtain the final result. Based on these properties, we can speed up the process of applying Meek rules for solving a COL problem.
	\item As mentioned before, in some cases, we are required to solve multiple COL problems with similar causal structures (such as COL problems in the experiment design problem). By exploiting Meek functions' properties, we can keep the results of some computations and reuse them in other similar COL problem. Based on this idea, we propose a method to solve experiment design problem in the active setting in an efficient manner. 
	\item We present a method to check efficiently whether a set of edges' orientations (which might come from some prior knowledge) are consistent in the sense that there is a DAG with the same orientations for the given edges. 
	\item We propose an efficient method to compute a lower bound on the number of oriented edges as a result of intervention. This lower bound can be utilized to select variables for intervention according to the minimax criterion \citep{Hauser14}. Experiments show that the gap between the lower bound and the true value is small and the same set of variables are selected if we use lower bounds instead of true ones.  
\end{itemize}

The rest of this paper is organized as follows: We begin with some background and terminologies in Section \ref{sec:Background}. We then introduce a rigorous representation of Meek rules and present some their key properties in Section \ref{sec:Meek rules}. In Section \ref{sec:applications}, based on these properties, we present methods for applying Meek rules on causal graphs and checking the consistency of some edges' orientations in a DAG. We also propose a lower bound on the number of oriented edges as a result of intervention. In Section \ref{sec:Experiments}, we report our experimental results for the proposed methods and the lower bound. Finally, in Section \ref{sec:Conclusion}, we conclude the paper and discuss some possible directions for future research.


\section{Background and Terminology}
\label{sec:Background}
A graph $G = (\V,\E)$ is represented by a set of nodes $\V$ and a set of edges $\E$. We denote the  undirected edge between a pair of nodes $\v_1,\v_2 \in \V$, by $\v_1 \line \v_2$. 
 A directed edge from node $\v_1$ to node $\v_2$ is also denoted by $\v_1 \rightarrow \v_2$. We assume that there exists at most one edge either directed or undirected  between each two nodes. We denote the set of all nodes that are connected to node $\v$ by $neigh(\v)$. We also define $Neigh(\v) = neigh(\v) \cup \{\v\}$. We call a subset of nodes of an undirected graph as  clique if every two nodes in that subset are adjacent. A clique is maximal if it is not a subset of a larger clique. We denote the set of maximal cliques in the neighborhood of node $\v$ by $\SET{C}(\v)$.

We say a graph $G$ is partially directed acyclic graph (PDAG) if it does not have any directed cycle \citep{Peters2017}. A graph $G$ is called directed acyclic graph (DAG) if it is PDAG and all the edges in $G$ are directed. We say that node $v_1$ is a parent of node $v_2$ if there is a directed edge from $v_1$ to $v_2$ in graph $G$. Moreover, the descendants of a node $v_1$ is the set of nodes that there are directed paths from $v_2$ to them. An induced sub-graph of graph $G$ containing vertices $\SET{T} \subseteq \V$  is defined to be a set of edges that contains all those edges in $\E$ with both end points in $\SET{T}$ and it is denoted by $\E[\SET{T}]$. We call an induced sub-graph in the form of $\v_1 \rightarrow \v_2 \leftarrow \v_3$ as a v-structure. If a graph has no partially directed cycle, it is called a chain graph. Removing all directed edges in a chain graph leaves some undirected disjoint chain graphs. These chain graphs are called, chain components. An undirected graph is said to be chordal if there exists a cord in any cycle with length more than three. 


Consider the set of variables $\mathcal{X}=\{X_1,\cdots,X_n\}$ in the system. Causal relationships between these variables are usually modeled by a DAG $G$, where each variable in $\mathcal{X}$ is mapped to one of the nodes in $G$, and an arrow between two nodes, like $\v_1 \rightarrow \v_2 $, shows the corresponding variable of $v_1$ in  $\mathcal{X}$ as a direct cause of the corresponding variable of $v_2$ in $\mathcal{X}$. We call this DAG, which represents causal relationships between variables, as a ``causal DAG". We say that a joint distribution $P$ over $\mathcal{X}$ satisfies Markov property with respect to $G$ if any variable of $G$ is independent of its non-descendants
given its parents. Under causal sufficiency and faithfulness
assumptions, any conditional independence in $P$ can be
inferred by Markov property \citep{Spirtes00}.
The set of all DAGs that encode the same conditional independence assertions are called Markov equivalence class (MEC). An essential graph corresponding to an MEC is a graph that has the same skeleton as all DAGs in that MEC and an edge is directed in essential graph if that edge has the same direction in those all DAGs.  It can be shown that the essential graph is a chain graph with chordal chain components, where each chain component is an undirected and connected chordal graph (UCCG for short) \citep{Andersson}.  We denote chain components of a graph $G$ by $\mathcal{C}(G)$. Note that the chain components are obtained by removing all the directed edges from the essential graph, where whole v-structures are identified. Thus, in any chain component, there is no undiscovered v-structure. In this manuscript, we use the term UCCG for denoting one of the chain components of an essential graph.

From the observational distribution $P$, the essential graph of underlying ground truth DAG can be recovered by performing conditional independence tests \citep{Spirtes00}. For further orienting the undirected edges, we need to intervene in the system. We use the same notion of hard intervention as in \cite{Eberhardt05} and \cite{Pearl09}. The procedure of intervention on a random variable $X$ is to force this variable to get its values from an independent randomized distribution, regardless of the values of its parents. For a set of interventions $I$, two DAGs $G_1$ and $G_2$ are called $\mathcal{I}$-Markov equivalent if they are statistically indistinguishable under interventions in $I$. Interventional MEC ($\mathcal{I}$-MEC) is the set of all DAGs that are $\mathcal{I}$-Markov equivalent. Moreover, the summary of all DAGs in an $\mathcal{I}$-MEC can be presented by an essential graph, which we call $\mathcal{I}$-essential graph.


In this paper, we assume that infinite samples from any observational or interventional distributions are available. Thus, the true essential graph of underlying causal model is available.

\section{Causal Edge Orientation}
\label{sec:Meek rules}
In this section, we first describe Meek rules \citep{Meek95}. These rules enable us to orient further edges since we know the underlying causal graph is a DAG. Then, in the next section, we present a mathematical representation for Meek rules. Specifically, we define Meek functions that help us to formulate the Meek rules in more rigorous way. In the last part, we present some properties of Meek functions. As we will see later, these properties can be utilized to solve COL problems in an efficient manner.

\subsection{Meek rules}

As we mentioned in the previous section, the essential graph is a partially directed acyclic graph that will be identified from the observational distribution \citep{Spirtes00}.
 Note that the skeleton and all the v-structures of the essential graph are the same as the skeleton and v-structures of underlying causal DAG and they can be identified by performing some conditional independence tests \citep{Pearl92}. Additional edges in the essential graph can be oriented based on these two facts: (a) there is no more undiscovered v-structure, and (b) the underlying causal DAG is acyclic. Every such additional edges can be identified by repeatedly applying Meek rules \citep{Meek95}. It can be shown that Meek rules are complete and sound in recovering the essential graph \citep{Meek95}. 
 
 There are four Meek rules that are given in Table \ref{table:Meek}. In this table, we illustrate how an edge will be oriented after applying each Meek rule. Each column corresponds to one of the Meek rules. The graph in the first row in each column will be converted to the graph in the second row in that column after applying the corresponding Meek rule. Thus, some undirected edges in the graph in the first row have been oriented in the graph in the second row. By considering all possible orientations for corresponding sub-graphs of Meek rules 3 and 4 , it can be shown that there is no need to consider orientations of dashed lines for applying these Meek rules as long as the dashed lines do not belong to a v-structure.
\usetikzlibrary{decorations.markings,arrows.meta}
\tikzset
{midarrow/.style={decoration={markings,mark=at position 0.6 with
			{\arrow[xshift=2pt]{Latex[length=8pt,#1]}}},postaction={decorate}}
}

\begin{table}[ht]
	\begin{center}
	\caption{Meek rules}
	\begin{tabular}{|c|c|c|c|c|}
		\hline
		Rule &Meek Rule 1 & Meek Rule 2 & Meek Rule 3 & Meek Rule 4 \\ 
		\hline
		\multirow{5}{*}{}
    	&
		\multirow{5}{*}{
		\begin{tikzpicture}[thick, scale=1.3]
		\node ($\v_i$) at (0,1.2) {$\v_i$};
		\node ($\v_k$) at (0,-.2) {$\v_k$};
		\node ($\v_j$) at (1,-.2) {$\v_j$};
		\draw[fill=black] (0,1) circle (1.2pt);
		\draw[fill=black] (0,0) circle (1.2pt);
		\draw[fill=black] (1,0) circle (1.2pt);
		\draw[midarrow] (0,1) -- (0,0);
		\draw[]			(0,0) -- (1,0);
		\end{tikzpicture}}
		&		
		\multirow{5}{*}{
		\begin{tikzpicture}[thick, scale=1.3]
		\node ($\v_i$) at (0,1.2) {$\v_i$};
		\node ($\v_k$) at (0,-.2) {$\v_k$};
		\node ($\v_j$) at (1,-.2) {$\v_j$};
			\draw[fill=black] (0,1) circle (1.2pt);
			\draw[fill=black] (0,0) circle (1.2pt);
			\draw[fill=black] (1,0) circle (1.2pt);
			\draw[midarrow] (0,1) -- (0,0);
			\draw[midarrow] (0,0) -- (1,0);
			\draw[]  		(0,1) -- (1,0);
			\end{tikzpicture}
			}
			&
		\multirow{5}{*}{
				\begin{tikzpicture}[thick, scale=.8]
			\node ($\v_i$) at (1,2.3) {$\v_i$};
			\node ($\v_k$) at (1,-.3) {$\v_k$};
			\node ($\v_l$) at (-.3,1) {$\v_l$};
			\node ($\v_j$) at (2.3,1) {$\v_j$};
			\draw[fill=black] (1,2) circle (1.5pt);
			\draw[fill=black] (1,0) circle (1.5pt);
			\draw[fill=black] (0,1) circle (1.5pt);
			\draw[fill=black] (2,1) circle (1.5pt);
			\draw[dashed]     (1,2) -- (2,1);
			\draw[dashed]     (1,2) -- (0,1);
			\draw[]			  (1,2) -- (1,0);
			\draw[midarrow]   (2,1) -- (1,0);
			\draw[midarrow]   (0,1) -- (1,0);
			\end{tikzpicture}
			}
			& 
			
		\multirow{5}{*}{
			\begin{tikzpicture}[thick, scale=.8]
			\node ($\v_i$) at (1,2.3) {$\v_i$};
			\node ($\v_k$) at (1,-.3) {$\v_k$};
			\node ($\v_l$) at (-.3,1) {$\v_l$};
			\node ($\v_j$) at (2.3,1) {$\v_j$};
			\draw[fill=black] (1,2) circle (1.5pt);
			\draw[fill=black] (1,0) circle (1.5pt);
			\draw[fill=black] (0,1) circle (1.5pt);
			\draw[fill=black] (2,1) circle (1.5pt);
			\draw[dashed]     (1,2) -- (2,1);
			\draw[]           (1,2) -- (0,1);
			\draw[dashed]	  (1,2) -- (1,0);
			\draw[midarrow]	  (2,1) -- (1,0);
			\draw[midarrow]   (1,0) -- (0,1);
			\end{tikzpicture}
			} 
		\\ &&&& 
		\\ Corresponding &&&& 
		\\ sub-graph &&&& 
		\\ &&&& 
		\\ &&&& 
		\\
		\hline
		\multirow{5}{*}{}
		&
				\multirow{5}{*}{
			\begin{tikzpicture}[thick, scale=1.3]
		\node ($\v_i$) at (0,1.2) {$\v_i$};
		\node ($\v_k$) at (0,-.2) {$\v_k$};
		\node ($\v_j$) at (1,-.2) {$\v_j$};
		\draw[fill=black] (0,1) circle (1.2pt);
		\draw[fill=black] (0,0) circle (1.2pt);
		\draw[fill=black] (1,0) circle (1.2pt);
		\draw[midarrow] (0,1) -- (0,0);
		\draw[midarrow] (0,0) -- (1,0);
		\end{tikzpicture}
		}
			&
				\multirow{5}{*}{
			\begin{tikzpicture}[thick, scale=1.3]
		\node ($\v_i$) at (0,1.2) {$\v_i$};
		\node ($\v_k$) at (0,-.2) {$\v_k$};
		\node ($\v_j$) at (1,-.2) {$\v_j$};
			\draw[fill=black] (0,1) circle (1.2pt);
			\draw[fill=black] (0,0) circle (1.2pt);
			\draw[fill=black] (1,0) circle (1.2pt);
			\draw[midarrow] (0,1) -- (0,0);
			\draw[midarrow] (0,0) -- (1,0);
			\draw[midarrow] (0,1) -- (1,0);
			\end{tikzpicture}	
			}
		& 
				\multirow{5}{*}{
				\begin{tikzpicture}[thick, scale=.8]
			\node ($\v_i$) at (1,2.3) {$\v_i$};
			\node ($\v_k$) at (1,-.3) {$\v_k$};
			\node ($\v_l$) at (-.3,1) {$\v_l$};
			\node ($\v_j$) at (2.3,1) {$\v_j$};
			\draw[fill=black] (1,2) circle (1.5pt);
			\draw[fill=black] (1,0) circle (1.5pt);
			\draw[fill=black] (0,1) circle (1.5pt);
			\draw[fill=black] (2,1) circle (1.5pt);
			\draw[dashed]     (1,2) -- (2,1);
			\draw[dashed]     (1,2) -- (0,1);
			\draw[midarrow]	  (1,2) -- (1,0);
			\draw[midarrow]   (2,1) -- (1,0);
			\draw[midarrow]   (0,1) -- (1,0);
			\end{tikzpicture}
			}
			&  
					\multirow{5}{*}{
		\begin{tikzpicture}[thick, scale=.8]
		\node ($\v_i$) at (1,2.3) {$\v_i$};
		\node ($\v_k$) at (1,-.3) {$\v_k$};
		\node ($\v_l$) at (-.3,1) {$\v_l$};
		\node ($\v_j$) at (2.3,1) {$\v_j$};
		\draw[fill=black] (1,2) circle (1.5pt);
		\draw[fill=black] (1,0) circle (1.5pt);
		\draw[fill=black] (0,1) circle (1.5pt);
		\draw[fill=black] (2,1) circle (1.5pt);
		\draw[dashed]     (1,2) -- (2,1);
		\draw[midarrow]	  (1,2) -- (0,1);
		\draw[dashed]     (1,2) -- (1,0);
		\draw[midarrow]	  (2,1) -- (1,0);
		\draw[midarrow]   (1,0) -- (0,1);
		\end{tikzpicture}
		}
		\\ After applying &&&& 
		\\ Meek rule on &&&& 
		\\ corresponding &&&& 
		\\ sub-graph &&&& 
		\\ &&&& 
		\\
		\hline	
		
		\multirow{5}{*}{}
    	&
		\multirow{5}{*}{
		\begin{tikzpicture}[thick, scale=1.3]
		\node ($\v_i$) at (0,1.2) {$\v_i$};
		\node ($\v_k$) at (0,-.2) {$\v_k$};
		\node ($\v_j$) at (1,-.2) {$\v_j$};
		\draw[fill=black] (0,1) circle (1.2pt);
		\draw[fill=black] (0,0) circle (1.2pt);
		\draw[fill=black] (1,0) circle (1.2pt);
		\draw[midarrow] (0,1) -- (0,0);
		\draw[]			(0,0) -- (1,0);
		\end{tikzpicture}}
		&		
		\multirow{5}{*}{
		\begin{tikzpicture}[thick, scale=1.3]
		\node ($\v_i$) at (0,1.2) {$\v_i$};
		\node ($\v_k$) at (0,-.2) {$\v_k$};
		\node ($\v_j$) at (1,-.2) {$\v_j$};
			\draw[fill=black] (0,1) circle (1.2pt);
			\draw[fill=black] (0,0) circle (1.2pt);
			\draw[fill=black] (1,0) circle (1.2pt);
			\draw[midarrow] (0,1) -- (0,0);
			\draw[midarrow] (0,0) -- (1,0);
			\draw[]  		(0,1) -- (1,0);
			\end{tikzpicture}
			}
			&
		\multirow{5}{*}{
				\begin{tikzpicture}[thick, scale=.8]
			\node ($\v_i$) at (1,2.3) {$\v_i$};
			\node ($\v_k$) at (1,-.3) {$\v_k$};
			\node ($\v_l$) at (-.3,1) {$\v_l$};
			\node ($\v_j$) at (2.3,1) {$\v_j$};
			\draw[fill=black] (1,2) circle (1.5pt);
			\draw[fill=black] (1,0) circle (1.5pt);
			\draw[fill=black] (0,1) circle (1.5pt);
			\draw[fill=black] (2,1) circle (1.5pt);
			\draw[dashed]     (1,2) -- (2,1);
			\draw[dashed]     (1,2) -- (0,1);
			\draw[]			  (1,2) -- (1,0);
			\draw[midarrow]   (2,1) -- (1,0);
			\draw[midarrow]   (0,1) -- (1,0);
			\end{tikzpicture}
			}
			& 
			
		\multirow{5}{*}{
			\begin{tikzpicture}[thick, scale=.8]
			\node ($\v_i$) at (1,2.3) {$\v_i$};
			\node ($\v_k$) at (1,-.3) {$\v_k$};
			\node ($\v_l$) at (-.3,1) {$\v_l$};
			\node ($\v_j$) at (2.3,1) {$\v_j$};
			\draw[fill=black] (1,2) circle (1.5pt);
			\draw[fill=black] (1,0) circle (1.5pt);
			\draw[fill=black] (0,1) circle (1.5pt);
			\draw[fill=black] (2,1) circle (1.5pt);
			\draw[dashed]     (1,2) -- (2,1);
			\draw[]	          (1,2) -- (0,1);
			\draw[dashed]	  (1,2) -- (1,0);
			\draw[midarrow]	  (2,1) -- (1,0);
			\draw[dashed]     (1,0) -- (0,1);
			\end{tikzpicture}
			} 
		\\ Corresponding &&&& 
		\\ candidate &&&& 
		\\ sub-graph &&&& 
		\\ &&&& 
		\\ &&&& 
		\\
		\hline	
	\end{tabular}
		\label{table:Meek}
	\end{center}
\end{table}


In addition to applying Meek rules to obtain essential graph, we can utilize them to orient further edges after performing an intervention on a node to obtain  $\mathcal{I}$-essential graph.
 In this case, after performing a perfect randomized intervention, the orientations of all edges between the intervened node and its neighbors will be recovered \citep{He08}. Applying Meek rules on the graph with these new oriented edges, results in orienting  additional edges. This procedure will be continued till no edge can be discovered. The resulted graph would be the $\mathcal{I}$-essential graph after the intervention.

To wrap up, applying Meek rules on a graph with some new directed edges (for instance, oriented edges in the neighborhood of a node after performing an intervention), may recover the orientations of further edges in the graph. 
Having  more directed edges means that more causal relations have been revealed from the underlying causal graph. 
Thus, Meek rules can be seen as a tool to infer new causal orientations based on some prior knowledge about the causal graph.
%
 We will present a more formal representation of these rules in the next part. 

\subsection{Meek Functions}
In this part, we define Meek functions. These functions take a set of edges as their inputs and return another set containing already directed edges in the input set and some new edges that are oriented as a result of repeatedly applying one of the Meek rules. 
Furthermore, we define a ``candidate sub-graph'' for each of these functions. Candidate sub-graph is an induced sub-graph, with minimal number of nodes, that applying particular Meek rule on that sub-graph, results in orienting one or more edges inside that sub-graph. We give a concise definition of candidate sub-graphs for each of the Meek functions in the following. We first provide some definitions.

\begin{definition}[Mixed-edge union operator] Let $\SET{A}$ and $\SET{B}$ be sets of some directed and undirected edges. We define $\SET{A} \sqcup \SET{B}$ as a set of edges with the following property:
\begin{align*}
\SET{A} \Rcup \SET{B} = \SET{C} \backslash \{v_i \line v_j | v_i \rightarrow v_j \in \SET{C} \mbox{ or } v_j \rightarrow v_i \in \SET{C}, \forall i,j \in \V  \},
\end{align*}
where $\SET{C} = \SET{A} \cup \SET{B}$. In other words, $\SET{A} \sqcup \SET{B}$ keeps the directed edge $v_i \rightarrow v_j$ if it exists in $\SET{A}$ or $\SET{B}$. Moreover, it keeps the undirected edge $v_i \line v_j$ if this edge has not been oriented in either $\SET{A}$ or $\SET{B}$.

\end{definition}
\begin{definition}[Skeleton of a graph] Let $G = (\V,\E)$ be a PDAG, $\V$ be the set of nodes and $\E$ be the set of edges. Assume that we replace all directed edges in this partially directed graph with the undirected ones. We denote the set of all edges in this undirected graph with $\overline{\E}$. Thus, we have:
	\label{def:Undir} 
		\begin{align*}
		&\overline{\E} = \left \{\v_i \line \v_j | v_i \rightarrow\v_j \in \E \text{ or }  \v_i \line \v_j \in \E, \forall i,j \in \V \right \}.
		\end{align*}
\end{definition}
\begin{definition}[Set of directed edges] Let $G = (\V,\E)$ be a PDAG. We denote the set of all directed edges in $\E$ by $\overrightarrow{\E}$:
		\begin{align*}
		&\overrightarrow{\E} = \left \{\v_i \rightarrow \v_j | \v_i \rightarrow\v_j \in \E, \forall i,j \in \V \right \}.
		\end{align*}
\end{definition}

Now, we define a candidate sub-graph for each Meek rule. These candidate sub-graphs are depicted in the third row in Table \ref{table:Meek}. Examples for Meek candidate sub-graphs are shown in Figure \ref{Fig:candidate}. The candidate sub-graphs for Meek rules 1, 2 and 3 are as the same of their corresponding  sub-graphs, while Meek rule 4 candidate sub-graph deviates from its corresponding sub-graph. In the following, we show that applying Meek rules are sound and complete. The proof of all lemmas and theorems are given in the appendix.

\begin{theorem}[Orientation soundness]
	\label{thm:Orientationsoundness}
	Considering candidate sub-graphs in Table \ref{table:Meek}, the four orientation rules are sound.
\end{theorem}

Based on the above definitions, we are now ready to define Meek functions. 
More specifically, let $G=(\V,\E)$ be a PDAG with set of nodes $\V$ and set of edges  $\E$.
 \begin{definition} [Meek functions]
 We define $M_i$ as a function that takes $\E$ as input and returns the set of edges $\E_{M_i}$ as the result of applying Meek rule i repeatedly on the corresponding candidate sub-graphs of Meek rule i in $\E$ until no candidate sub-graph can be found and as a result no more edge can be oriented by this rule. In other words, we have: $M_i(\E) = \E_{M_i}$.
 \end{definition}
\begin{remark}
  Since there exists no candidate sub-graph for Meek rule i 
  in the result of  $M_i(\E)$, we have: $M_i(M_i(\E)) = M_i(\E)$. 
\end{remark}
  






Furthermore, in some parts of the paper, we overload the notation of Meek function $M_1$ with the second argument $\E_f$, which is called as ``forbidden edges". The overloaded function is in the form of $M_1(\E,\E_f)$. Forbidden edges are the edges that Meek function $M_1$ does not allow to orient those edges. We use this type of function $M_1$ for the recovering essential graphs in Section \ref{sec:Experiments}. In the following, we give examples of both types of function $M_1$.

\begin{example}
	\label{Ex:exampleMeekFunctionM1}
    Given a graph $G=(\V,\E)$, where $\V = \{\v_1,\v_2,\v_3,\v_4\}$ and $\E =\{ \v_1 \rightarrow \v_2,\v_2 \line \v_3,\v_3 \line \v_4 \}$, we have:
    	\begin{align*}
&M_{1} (\E) = \{ \v_1 \rightarrow \v_2 ,\v_2 \rightarrow \v_3,\v_3 \rightarrow \v_4 \}\\
    &M_{1} (\E,\{\v_2 \line \v_3\}) = \{ \v_1 \rightarrow \v_2,\v_2 \line \v_3,\v_3 \line \v_4 \}
	\end{align*}
\end{example}

\begin{figure}
    \begin{center}
		\begin{tabular}{ccccc}
				\begin{tabular}{c}
				\begin{tikzpicture}[thick, scale=.8]
				\node (a) at (1,2.3) {$v_a$};
				\node (b) at (1,-.3) {$v_b$};
				\node (c) at (-.3,1) {$v_c$};
				\node (d) at (2.3,1) {$v_d$};
				\node (e) at (3,2.3) {$v_e$};
				\draw[fill=black] (1,2) circle (1.5pt);
				\draw[fill=black] (1,0) circle (1.5pt);
				\draw[fill=black] (0,1) circle (1.5pt);
				\draw[fill=black] (2,1) circle (1.5pt);
				\draw[fill=black] (3,2) circle (1.5pt);
				\draw[midarrow]           (2,1) -- (1,2);
				\draw[midarrow]	  (0,1) -- (1,2);
				\draw[]			  (1,2) -- (1,0);
				\draw[midarrow]	  (1,0) -- (2,1) ;
				\draw[midarrow]   (0,1) -- (1,0);
				\draw[]   (1,2) -- (3,2);
				\draw[midarrow]   		  (2,1) -- (3,2);
				\end{tikzpicture}
				
			\end{tabular}		
		
		&
				\begin{tabular}{c}
				\begin{tikzpicture}[thick, scale=.8]
				
				\node (a) at (1,2.3) {$v_a$};
				\node (c) at (-.3,1) {$v_c$};
				\node (e) at (2.5,2.3) {$v_e$};
				
				\draw[fill=black] (1,2) circle (1.5pt);
				\draw[fill=black] (0,1) circle (1.5pt);
				\draw[fill=black] (2.5,2) circle (1.5pt);
				\draw[midarrow]	  (0,1) -- (1,2);
				\draw[]   (1,2) -- (2.5,2);
				\end{tikzpicture}
			\end{tabular}	
			&
						\begin{tabular}{c}
				\begin{tikzpicture}[thick, scale=.8]

				\node (a) at (1,2.3) {$v_a$};
				\node (b) at (1,-.3) {$v_b$};
				\node (d) at (2.3,1) {$v_d$};
				\draw[fill=black] (1,2) circle (1.5pt);
				\draw[fill=black] (1,0) circle (1.5pt);
				\draw[fill=black] (2,1) circle (1.5pt);
				\draw[midarrow]           (2,1) -- (1,2);
				\draw[]			  (1,2) -- (1,0);
				\draw[midarrow]	  (1,0) -- (2,1) ;
				\end{tikzpicture}
				
			\end{tabular}	
			&				\begin{tabular}{c}
				\begin{tikzpicture}[thick, scale=.8]
				
				\node (a) at (1,2.3) {$v_a$};
				\node (b) at (1,-.3) {$v_b$};
				\node (c) at (-.3,1) {$v_c$};
				\node (d) at (2.3,1) {$v_d$};

				\draw[fill=black] (1,2) circle (1.5pt);
				\draw[fill=black] (1,0) circle (1.5pt);
				\draw[fill=black] (0,1) circle (1.5pt);
				\draw[fill=black] (2,1) circle (1.5pt);
				\draw[midarrow]           (2,1) -- (1,2);
				\draw[midarrow]	  (0,1) -- (1,2);
				\draw[]			  (1,2) -- (1,0);
				\draw[midarrow]	  (1,0) -- (2,1) ;
				\draw[midarrow]   (0,1) -- (1,0);
				\end{tikzpicture}
				
			\end{tabular}		
		
		&
		\begin{tabular}{c}
				\begin{tikzpicture}[thick, scale=.8]
				\node (a) at (1,2.3) {$v_a$};
				\node (b) at (1,-.3) {$v_b$};
				\node (d) at (2.3,1) {$v_d$};
				\node (e) at (3,2.3) {$v_e$};
				\draw[fill=black] (1,2) circle (1.5pt);
				\draw[fill=black] (1,0) circle (1.5pt);
				\draw[fill=black] (2,1) circle (1.5pt);
				\draw[fill=black] (3,2) circle (1.5pt);
				\draw[midarrow]           (2,1) -- (1,2);
				\draw[]			  (1,2) -- (1,0);
				\draw[midarrow]	  (1,0) -- (2,1) ;
				\draw[]   (1,2) -- (3,2);
				\draw[midarrow]   		  (2,1) -- (3,2);
				\end{tikzpicture}
				
			\end{tabular}	
			\\
			(a) & (b) & (c) & (d) & (e)
			\\
		\end{tabular}
	\caption{(a): A partially directed acyclic graph ($G$). (b): Induced candidate sub-graph of $G$ for Meek rule 1. (c): Induced candidate sub-graph of $G$ for Meek rule 2. (d): Induced candidate sub-graph of $G$ for Meek rule 3. (e): Induced candidate sub-graph of $G$ for Meek rule 4.}
	\label{Fig:candidate}
    \end{center}
\end{figure}
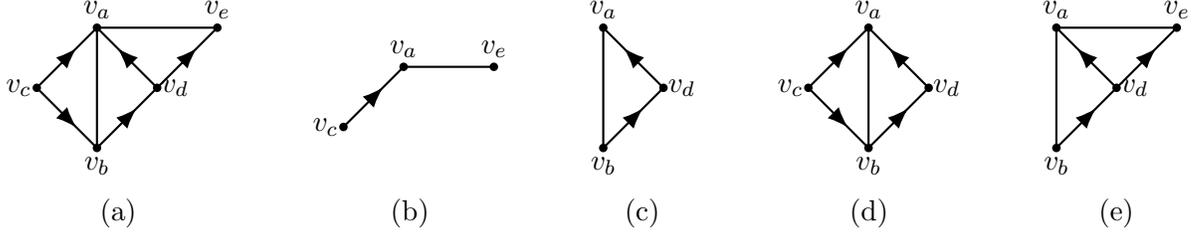

In practice, we need to apply multiple Meek functions to discover more edges' orientation. Any combination of Meek rules can be considered. For example, we define $M_{123}(\E)=\E_{M_{123}}$ as a function that takes set of edges $\E$ as input and returns the set of edges $\E_{M_{123}}$ as the result, by applying Meek rules 1, 2, and 3 in any order, repeatedly, till no new edge can be discovered by applying any of these Meek rules. Again, the function $M_{123}$ has the property that $M_{123}(M_{123}(\E)) = M_{123}(\E)$.


Now, we define minimal PDAG (for short MPDAG), which is the minimal graph, in terms of number of oriented edges, for representing equivalence class of a DAG.

\begin{definition}[Minimal PDAG] A minimal PDAG is a PDAG which is obtained by replacing some directed edges in a DAG with unidrected one, those edges that do not participate in a v-structure.
We denote a minimal PDAG by $G_o = (\V,\E)$.
\end{definition}

The following two theorems guarantee that applying Meek functions on the corresponding candidate sub-graphs are complete. One of these theorems is on the completeness of  applying Meek functions on MPDAG, and the other one is on the completeness of applying Meek functions on MPDAG which has the some further oriented edge, which are selected from one of the DAGs in the corresponding MEC.

\begin{definition}[Maximally oriented graph] A PDAG $G=(\V,\E)$ in the MEC of a MPDAG $G_o$ is a maximally oriented graph such that for each undirected edge $\v_i \line \v_j$ in $G$, if there exist DAGs $D1=(\V,\E_1)$ and $D_2=(\V,\E_2)$ in the MEC of $G_o$ such that we have $\v_i \rightarrow \v_j \in \E_1$ and $\v_j \rightarrow \v_i \in \E_2$, where $\overrightarrow{\E} \subseteq \E_1$ and $\overrightarrow{\E} \subseteq \E_2$.
\end{definition}

\begin{theorem}[Orientation completeness on MPDAG]
	\label{thm:MPDAGcompleteness}
	Let $G_o=(\V,\E)$ be a MPDAG. The result of applying $M_{123}$ on $\E$ is a maximally oriented graph.
\end{theorem}

\begin{theorem}[Orientation completeness]
	\label{thm:GeneralMeekcompleteness}
	Let $G=(\V,\E)$ be a MPDAG, including some further directed edges, which are selected from one of the DAGs in MEC of its MPDAG. The result of applying $M_{1234}$ on $E$ is a maximally oriented graph with respect to $\overrightarrow{\E}$.
\end{theorem}

In the next part, we present some key properties of these functions.

\subsection{Meek functions Properties}

In the previous part, we defined four Meek functions inspired by Meek rules \citep{Meek95}. In this part, we describe key  properties of these functions. By exploiting these properties, we can accelerate executing Meek functions and utilize it as tools to design efficient algorithms for the problem of experiment design or checking the consistency of prior knowledge with the observational or interventional data (see Section \ref{sec:applications}). Before presenting the properties, we define the partially directed and connected chordal graph (for short PCCG). In fact, PCCG is a representation of one of the chain components of an essential graph, where some of its edges are oriented according to one of the DAGs in the corresponding MEC.

\begin{definition}[Partially directed and connected chordal graph] A PDAG $G=(\V,\E)$ is a PCCG  if:
\begin{itemize}
    \item Skeleton of $G$ is a UCCG.
    \item $\overrightarrow{\E}$ in not empty.
    \item 
    There is a DAG in the MEC of $G$  with essential graph equal to skeleton of $G$, which is also consistent with $\overrightarrow{\E}$\footnote{ As the PCCG is a representation for partially oriented chain component, this constraint enforces to consider only DAGs with consistent skeleton and no v-structure.}.
    
\end{itemize}
\end{definition}

 In the following lemma, we will provide some properties of Meek functions on PCCGs.

\begin{lemma}[Meek functions properties on PCCGs]
	\label{lem:MeekPropertiesPCCG} 
     Let $G=(\V,\E)$ be a PCCG. The following properties hold for Meek functions:
	\begin{enumerate}
	
		\item For any $\SET{T} \subseteq \V$, we have: $M_{1234}(\E{[\SET{T}]}) = M_{124}(\E{[\SET{T}]})
 		$
		\item For $M_{i} \in \{M_{1}, M_{4}, M_{14}\}$, we have: $M_{i}(\E) =  \underset{\v_i \rightarrow \v_j \in \overrightarrow{\E}}{\bigsqcup} M_{i}( \{\v_i \rightarrow \v_j\} \Rcup \overline{\E} )$.
	\item $M_{124} \left ( \E \right )= M_{14}(M_{2}(\E))$.
		\item For any subset $\SET{S} = \{ \v_i \rightarrow \v_k,\v_k \rightarrow \v_j,\v_i \line \v_j\} \subseteq \SET{\E}$, we have:
		\begin{align*}
		M_{14}(M_2(\SET{S}) \sqcup \E\backslash \SET{S})
		= M_{14}( \E ) \Rcup M_{2}(\SET{S}).
		\end{align*}
	\end{enumerate}
\end{lemma}

In the following, we provide some intuitions of the above properties. Based on property 1, function $M_3$ does not change the result when it is applied on a PCCG. This is due to the fact that there will be no v-structure in these type of graphs, and as a result there will be no candidate sub-graph for Meek rule 3. Thus, there is no need to apply this function on the mentioned graphs.

The property 2 states that all the information in order to reveal additional edges' orientations as a result of applying one of the functions $M_1$,$M_4$ or $M_{14}$ are in the $\overrightarrow{\E}$. In other words, if an edge is oriented after applying Meek functions $M_1$,$M_4$ or $M_{14}$ on set $\E$, consequently, this edge will be oriented by applying Meek functions $M_1$,$M_4$ or $M_{14}$ on $e \sqcup \overline{\E}$, for some $e \in \overrightarrow{\E}$.
This property has two appealing consequences. First, the result of applying mentioned Meek functions can be computed once and used multiple times such as in the problem of experiment design (see Section 4).
 Second, we can devise a dynamic programming method to apply Meek functions which  reduces the running times (see Section 4).

By applying Meek function $M_{124}$ on set of edges $\E$, some undirected edges will be oriented in $\E$. 
Property 3 asserts that we can recover all these edges, by consecutively applying function $M_{14}$ after function $M_2$ on $\E$. 

A candidate sub-graph for Meek rule 2 is in the form of $\{\v_i \rightarrow \v_k, \v_k \rightarrow \v_j, \v_i \line \v_j\}$. After applying function $M_2$ on this candidate sub-graph, $\v_i \line \v_j$ edge will be oriented as $\v_i \rightarrow \v_j$. Property 4 states that discovering this edge's orientation cannot help in orienting further edges, even we apply function $M_{14}$ on the set of edges including this directed edge.

In the next lemma, we express two properties for MPDAGs.

\begin{lemma}[Meek functions properties on MPDAGs]
	\label{lem:MeekPropertiesObserved} Let $G_o=(\V_o,\E_o)$ be a MPDAG. The following properties hold for Meek functions:
	\begin{enumerate}
	
		\item $M_{1234}(\E_o) =  M_{12}(M_{3}(\E_o))$.
		\item 
		$M_{1}(\E_o) =  \underset{\v_i \rightarrow \v_j \in \overrightarrow{\E_o}}{\bigsqcup} M_{1}( \{\v_i \rightarrow \v_j\} \Rcup \overline{\E_o}, \overline{\E_o} \backslash \{\v_i \line \v_j\} )$.
	\end{enumerate}
\end{lemma}

Based on the Theorem \ref{thm:MPDAGcompleteness}, we know that Meek function $M_4$ is not necessary to be applied in order to obtain the essential graph . Thus, we can omit this function from the functions that must be applied. Moreover, candidate sub-graph for Meek rule 3 cannot be occurred as a result of applying other Meek functions. This is because there is a v-structure in the candidate sub-graph for Meek rule 3 and applying any of other Meek functions cannot generate a new v-structure. Hence, from property 1, we can imply that it is just needed to first apply function $M_3$ and then apply the other two functions.

Similar to Property 2 in Lemma \ref{lem:MeekPropertiesPCCG}, in Property 2 in Lemma \ref{lem:MeekPropertiesObserved}, we use only directed edges in $\E_o$ for orienting further edges in graph $G_o$. In this property, we take v-structures as the forbidden edges argument of Meek function $M_1$ to avoid orient them in the reverse direction. 

After performing an intervention on an arbitrary node $\v$ in $G$, we can identify the direction of all the edges incident with node $v$ \citep{He08,Eberhardt05}. We know that these edges are enough to recover the $\mathcal{I}$-essential graph \citep{Hauser14}. Moreover, according to Theorem \ref{thm:GeneralMeekcompleteness}, all the recoverable orientations can be identified by applying Meek functions. Furthermore, various possible orientations of incident edges yield different possible $\mathcal{I}$-essential graphs. For instance, \cite{Hauser14} used this idea to compute a score function to select a node for performing an intervention.

In the following, we define intervened and connected chordal graph (for short ICCG). Suppose a node $\v$ in a UCCG $G = (\V,\E)$ has been intervened on. After intervention, all the edges connected to this node,  will be oriented. We denote the sets of nodes that have in-going edges toward and outgoing edges from node $\v$ by $\SET{I}$ and $\SET{O}$, respectively. Furthermore, we have $\SET{O} = neigh(\v) \backslash \SET{I}$.

\begin{definition}[Intervened connected chordal graph] We say PCCG $G=(\V,\E)$ is an ICCG if there exists only one node $v \in \V$ such that $\overrightarrow{\E} = \SET{S}$, where $\SET{S}= \{\v_i \rightarrow \v |  \v_i \in \SET{I} \} \sqcup \{ \v \rightarrow \v_o |  \v_o \in \SET{O} \}$ and $\SET{O} = neigh(\v) \backslash \SET{I}$.
\end{definition}
 An ICCG is a representation for $\mathcal{I}$-essential graph which is obtained after performing intervention on a node in a UCCG. This is because, based on the \cite{Hauser14}, knowing all the in-goring edges toward intervened node, is a complete representation of $\mathcal{I}$-essential graph. According to Proposition 6 in \cite{Hauser14}, set $\SET{I}$ is a subset of a maximal clique in the neighborhood of node $v$. An example of ICCG is shown in Example \ref{Ex:example}.

\begin{remark}
\label{InterveneGraphIsSpecialOfPCCF}
An ICCG is a special case of PCCG. Thus, the properties for PCCGs hold for ICCGs.
\end{remark}

In the Lemma \ref{lem:MeekProperties}, we provide one properties for ICCGs.
\begin{lemma}[Meek function property on ICCGs]
	\label{lem:MeekProperties} 
Let $G=(\V,\E)$ be an ICCG. Denoting all the induced candidate sub-graphs for Meek rule 2 in  $\E [ Neigh(\v)]$ by $\SET{C}_{M_2}(\v)$, we have:
		\begin{align*}
		M_{2}(\E) = \left ( \underset{\SET{C}_i \in \SET{C}_{M_2}(\v)}{\bigsqcup} M_{2}(\SET{C}_i) \right) \sqcup \E .
		\end{align*}
\end{lemma}

This property  asserts that, after performing an intervention, the candidate sub-graphs that consists of edges in $\overrightarrow{\E}$ as their directed edges are sufficient to compute result of applying function $M_2$.
Hence, we can obtain the result of this part in two steps: (1) extracting candidate sub-graphs for function $M_2$, (2) apply this function to only extracted candidate sub-graphs instead of applying to set of all edges. This procedure accelerates the execution of function $M_2$.

In some applications of COL problems, such as those mentioned in Section 3, we need to apply function $M_{124}$ on a given ICCG. This is because we do not have any v-structure in those types of graphs and as a result we do not need to apply Meek function $M_3$. The conventional method is to repeatedly apply Meek rules 1, 2 and 4 on sub-graphs corresponding to these rules, till no more edge can be oriented.
The following theorem states one of the key properties of Meek functions where applying function $M_{124}$ can be decomposed to applying function $M_{14}$ and $M_{2}$. This theorem ensures that the result of decomposition method is the same as the conventional one of applying Meek rules.

\begin{theorem}[Decomposition of applying $M_{124}$ on ICCGs]
	\label{thm:decomposition}
Let $G = (\V,\E)$ be an ICCG. Applying function $M_{124}$ on set $\E$ can be decomposed as follows:
		\begin{align*}
		M_{124}(\E) = \left ( \underset{e \in \overrightarrow{\E}}{\bigsqcup} M_{14}( \{e\} \Rcup \overline{\E} ) \right ) \bigsqcup \left ( \underset{\SET{C}_i \in \SET{C}_{M_2}(\v)}{\bigsqcup} M_{2}(\SET{C}_i) \right ),
		\end{align*}
		where $\SET{C}_{M_2}(\v)$ is collection of all induced candidate sub-graphs for Meek rule 2 in the sub-graph $\E [ Neigh(\v)]$.
\end{theorem}

In the next section, we exploit these properties of Meek functions in solving COL problems in different applications.

\section{Applications of COL Problems}
\label{sec:applications}

In this section, we utilize Meek function properties in different problems such as experiment design or checking the edge orientation consistency with a prior knowledge. In particular, in Section \ref{Chapter:Fast discovering}, we introduce a new method that accelerates computing Meek functions. In Section \ref{Chapter:LowerBound}, we propose a lower bound on number of edges that can be oriented as a result of performing an intervention on a variable. This lower bound can be used as a measure for designing experiments. In the last part, given a graph with set of combined directed and undirected edges, we check whether directed edges in the given graph belong to a DAG or not.



\subsection{Fast Computation of $\mathcal{I}$-essential Graphs}
\label{Chapter:Fast discovering}
In this part, we propose a novel method for efficiently discovering the orientations of additional edges that can be identified after performing an intervention\footnote{Similar to the method proposed in this part, we can accelerate the procedure of obtaining the essential graph. We will describe this method in  Section \ref{sec:Experiments}.}. This method uses mentioned properties of Meek functions in the previous section. An $\mathcal{I}$-essential graph is characterized by oriented edges in the neighboring of intervened node \citep{Hauser14}. In order to recover orientations of further edges, Meek functions must be applied. As we mentioned in the previous section, it suffices to apply Meek function $M_{124}$. Based on Theorem \ref{thm:decomposition}, we can decompose the output of Meek function $M_{124}$ computation into two parts: (a) The output of applying Meek function $M_{14}$, (b) The output of applying Meek function $M_2$. Hence, accelerating the computation of Meek functions $M_{14}$ and $M_{2}$ will reduce the running time of executing Meek function $M_{124}$. In the following, we describe the procedure in which we accelerate these functions.

 We first describe our method for applying Meek function $M_{14}$ in an efficient manner. Let $G = (\V,\E)$ be a PCCG. Based on property 2 of Lemma \ref{lem:MeekPropertiesPCCG}, we can imply that Meek function $M_{14}$ can be computed by performing mixed edge union on the results of applying this function on graphs that each of them has only one directed edge.
 We take advantage of this property to accelerate this function. Let $\v_s \line \v_d$ be an undirected edge in the set of edges $\overline{\E}$. For any edge $\v_s \line \v_d \in \overline{\E}$, we define $DP[\v_s \rightarrow \v_d]$ as follows:

\begin{equation}
\label{equ:DP}
DP\left [ \v_s \rightarrow \v_d \right ] = M_{14} ( \{\v_s \rightarrow \v_d\} \sqcup  \overline{\E} ) .
\end{equation}

In the following, we propose a dynamic programming method for computing the above function.
\begin{proposition}
	\label{Prop:DP} The following equation holds for DP function in (\ref{equ:DP}):
\begin{align*}
DP [\v_{s} \rightarrow \v_{d}]
= \left ( \underset{\v_l \rightarrow \v_k \in \overrightarrow{\E^\prime} }{\bigsqcup} DP [ \v_l \rightarrow \v_k ] 
 \right )
\bigsqcup
\left \{ 
\v_{s} \rightarrow \v_{d}
\right \},
\end{align*}
where $\E^\prime = 
 M_{14}( \{ \v_{s} \rightarrow \v_{d} \} \sqcup \overline{\E}[Neigh(\v_d)] ) 
\backslash \left \{ \v_s \rightarrow \v_d  \right \}$.
\end{proposition}

Proposition \ref{Prop:DP} suggests a method to compute the result of applying Meek function $M_{14}$ on a graph with one directed edge based on a dynamic programming method. According to Proposition \ref{Prop:DP}, we need to apply Meek function $M_{14}$ results in the neighborhood of $\v_d$. Hence, it suffices to check  two candidate sub-graphs for Meek rules 1 and 4 in the neighborhood of $\v_d$ for further edges orientation.

Figure \ref{Fig:Alg} illustrates these two sub-graphs. Note that there is no edge between nodes $\v_s$ and $\v_j$. So, in order to find candidate sub-graphs that are shown in Figure \ref{Fig:Alg}, we attempt to find some node $\v_j \notin neigh(\v_s)$. 
In the following, we design an algorithm that computes and stores the DP results using these sub-graphs. 

In Algorithm \ref{Alg:DPONE}, we suggest a recursive implementation that computes and stores DP results. This algorithm shows how one null entry of DP table can be filled.  We first initialize $DP[\v_s \rightarrow \v_d]$  with $\v_s \rightarrow \v_d$  in Line 4. According to the sub-graphs in Figure \ref{Fig:Alg}, there is no edge between nodes $\v_s$ and $\v_j$. Thus, we search for node $\v_j \in neigh(\v_d) \backslash neigh(\v_s)$ in Line 5. For any such node $\v_j$, using Meek function $M_1$, we must orient the edge $\v_d \line \v_j$ as $\v_d \rightarrow \v_j$. For the next step, we must compute the union of $DP[\v_s \rightarrow \v_d]$ and $DP[\v_d \rightarrow \v_j]$. In Lines 6-8, we check whether the entry $\v_d \rightarrow \v_j$ in DP is null or not. If DP is already calculated, we can compute the union. Otherwise, we go through another function call to calculate this entry. Lines 9-12 correspond to the case in Figure \ref{Fig:Alg}(a). In this case, we search for a node $\v_i$ in the neighborhood of all three nodes $\v_s$, $\v_d$ and $\v_j$. For any such node $\v_i$, using Meek function $M_4$, we must orient the edge $\v_i \line \v_j$ as $\v_i \rightarrow \v_j$ and then compute the union of $DP[\v_s \rightarrow \v_d]$ and $DP[\v_i \rightarrow \v_j]$. After all these recursive function calls are completed, the DP result for  $\v_s \rightarrow \v_d$ will be available. Note that in the procedure of obtaining one entry of DP table, multiple indices of this table might be filled.
\begin{theorem}
	Computational complexity of filling all entries of DP table using Algorithm \ref{Alg:DPONE} is in the order of $\mathcal{O}(|\E|\Delta^2)$, where $\Delta$ is the maximum degree in the graph.
	\label{the:TimeComplexityDP}
\end{theorem}
Since we compute DP for each edge in both directions, DP table size is two times of the number of edges. In the rest of this paper, we will assume that DP entries are available. In the next section, we will see that having DP table is helpful in applications that have high-demands of applying Meek functions $M_{14}$.

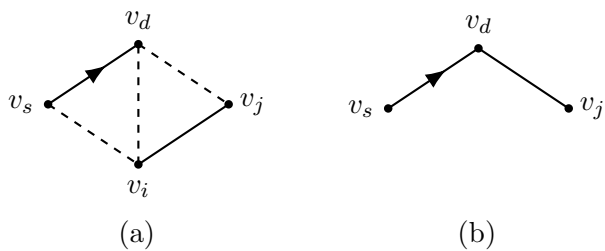
\begin{figure}
	\begin{center}
		\begin{tabular}{cc}
			\begin{tabular}{c}
				\begin{tikzpicture}[thick, scale=.8]
				\node (a) at (-.45,0) {$\v_s$};
				\node (b) at (1.5,1.4) {$\v_d$};
				\node (c) at (1.5,-1.4) {$v_i$};
				\node (d) at (3.4,0) {$v_j$};
				\draw[fill=black] (0,0) circle (1.5pt);
				\draw[fill=black] (1.5,-1) circle (1.5pt);
				\draw[fill=black] (1.5,1) circle (1.5pt);
				\draw[fill=black] (3,0) circle (1.5pt);
				\draw[midarrow]   (0,0) -- (1.5,1);
				\draw[dashed]	  (0,0) -- (1.5,-1);
				\draw[dashed]			  (1.5,1) -- (3,0);
				\draw[]			  (1.5,-1) -- (3,0);
				\draw[dashed]			  (1.5,-1) -- (1.5,1);
				\end{tikzpicture}
			\end{tabular}	
			&	
			\begin{tabular}{c}
				\begin{tikzpicture}[thick, scale=.8]
				\node (a) at (-.45,0) {$\v_s$};
				\node (b) at (1.5,1.4) {$\v_d$};
				\node (d) at (3.4,0) {$v_j$};
				\node (d) at (1.5,-1.4) {};
				\draw[fill=black] (0,0) circle (1.5pt);
				\draw[fill=black] (1.5,1) circle (1.5pt);
				\draw[fill=black] (3,0) circle (1.5pt);
				\draw[midarrow]   (0,0) -- (1.5,1);
				\draw[]			  (1.5,1) -- (3,0);
				\end{tikzpicture}
			\end{tabular}	
			\\
			(a) & (b)
			\\
		\end{tabular}
		\caption{(a): Candidate sub-graph for Meek rule 4. (b): Candidate sub-graph for Meek rule 1.}
		\label{Fig:Alg}
	\end{center}
\end{figure}

\begin{algorithm}[ht]
	\caption{Computation of $M_{14}( \{ \v_s \rightarrow \v_d \} \sqcup  \overline{\E} )$ }
	\label{Alg:DPONE}
	\begin{algorithmic}[1]

		\State \textbf{Input:} 			$G=(\V,\E)$, a PCCG
		\State \textbf{Output:} 		Oriented edges as a result of applying function $M_{14}$ on $\{\v_s \rightarrow \v_d\} \sqcup \overline{\E}$ which is equal to $DP[\v_s \rightarrow \v_d]$
		\Function{$OrientOneEdge$}{DP, $\v_s$, $\v_d$} 

		\State $DP[\v_s \rightarrow \v_d] \leftarrow \{\v_s \rightarrow \v_d\}$
		\For{$\v_j \in neigh(\v_d) \backslash neigh(\v_s)$} 
            \If{$DP[\v_d \rightarrow \v_j] = NULL$}
				\State $DP \leftarrow OrientOneEdge(DP,\v_d,\v_j)$
			\EndIf
			\State $DP[\v_s \rightarrow \v_d] \leftarrow DP[\v_s \rightarrow \v_d] \sqcup DP[\v_d \rightarrow \v_j]$
		\For{$\v_i \in neigh(\v_s) \cap neigh(\v_d) \cap neigh(\v_j)$} 
		\If{$DP[\v_i \rightarrow \v_j] = NULL$}
		\State $DP \leftarrow OrientOneEdge(DP,\v_i,\v_j)$
		\EndIf
		\State $DP[\v_s \rightarrow \v_d] \leftarrow DP[\v_s \rightarrow \v_d] \sqcup DP[ \v_i \rightarrow \v_j]$
		\EndFor  
		\EndFor 
		\EndFunction
		\State \textbf{Return} $DP$
	\end{algorithmic}
\end{algorithm}

Having the DP table, now we show how it can be utilized to apply $M_{14}$ function on a graph whose some edges are oriented as a result of intervention on a node.
In the following lemma, we will show that the result of applying $M_{14}$ can be accelerated by both using already calculated DP table and also taking advantage of property 2 in Lemma \ref{lem:MeekPropertiesPCCG}.
\begin{lemma}[Applying $M_{14}$ on PCCGs]
	\label{Lem:InterventionM14}
	 Let $G=(\V,\E)$ be a ICCG, which is obtained after intervention on node $v$. The result of applying $M_{14}$ on $\E$ can be obtained as follows:
	\begin{align*}
M_{14} \left ( \E \right ) = \underset{\v_l \rightarrow \v_k \in \SET{S^\prime} }{\bigsqcup} DP [\v_l \rightarrow \v_k],
	\end{align*}
where $\SET{S^\prime} =\{ \v_i \rightarrow \v |  \v_i \in \SET{I} \} \sqcup \{ \v \rightarrow \v_o |  \v_o \in \SET{O}, \SET{I} \subseteq neigh(\v_o) \}$.
\end{lemma}
Based on the previous lemma, we can obtain the result of Meek function $M_{14}$ by merely combining some entries of DP table. Thus, for further acceleration of function $M_{124}$, it suffices to accelerate applying Meek function $M_{2}$. Using properties 3 and 4 in Lemma \ref{lem:MeekPropertiesPCCG} allows us to obtain the result of Meek function $M_{2}$ on a graph without explicitly applying it. The next theorem suggests a fast method to discover new edges' orientations after applying Meek function $M_{124}$.

\begin{theorem}[Applying $M_{124}$ to obtain $\mathcal{I}$-essential graph]
	\label{Thm:InterventionM124}
	Let graph $G=(\V,\E)$ be an ICCG, which is obtained after intervention on node $v$.
	 The result of applying $M_{124}$ on set $\E$ can be obtained as follows:
		\begin{align*}
		&M_{124} ( \E ) = \left \{ \underset{\v_l \rightarrow \v_k \in \SET{S^\prime} }{\bigsqcup} DP [\v_l \rightarrow \v_k ] \right \} \bigsqcup  \left \{\v_i \rightarrow \v_o |  \v_i \in \SET{I},  \v_o \in \SET{V}_c \right \}
		\end{align*}
where $\SET{S^\prime} =\{ \v_i \rightarrow \v |  \v_i \in \SET{I} \} \sqcup \{ \v \rightarrow \v_o |  \v_o \in \SET{V}_c \}$, and $\SET{V}_c= \{\v_o|\v_o \in \SET{O}, \SET{I} \subseteq neigh(\v_o)\}$.
\end{theorem}


\begin{example}
	\label{Ex:example} Figure \ref{Ex:intervention}(a) illustrates a UCCG $G = (\V,\E)$. Assume that an intervention was performed on node $\v_3$ and the ICCG in Figure \ref{Ex:intervention}(b) was obtained. Entries of DP table for all possible edges' orientations are given in Figure \ref{Ex:intervention}(c). From the graph, we have $\SET{I} = \{\v_4\}$, $\SET{O} = \{\v_2\}$ and $\SET{S} = \{\v_4 \rightarrow \v_3,\v_3 \rightarrow \v_2\}$. Furthermore, we have $\SET{S}^\prime = \SET{S}$. Therefore, according to Theorem \ref{Thm:InterventionM124}, the result of applying Meek function $M_{124}$ can be obtained as follows:
		\begin{align*}
		M_{124} \left ( \E \sqcup \SET{S} \right ) &= \left ( \underset{\v_l \rightarrow \v_k \in \SET{S} }{\bigsqcup} DP[ \v_l \rightarrow \v_k] \right ) \bigsqcup \{\v_4 \rightarrow \v_2\} \\ 
		& = \{\v_3 \rightarrow \v_2,\v_4 \rightarrow \v_1,\v_2 \rightarrow \v_1,\v_1 \rightarrow \v_5,\v_4 \rightarrow \v_3,\v_4 \rightarrow \v_2\}.
		\end{align*}
		
\begin{figure}
		\begin{tabular}{*{2}{c}}
 				\multicolumn{1}{c}{\text{\space\space\space\space\space\space\space\space\space\space\space\space\space\space\space}
				\begin{tikzpicture}[thick, scale=1]
				\node (v1) at (0,1.25)  {$\v_1$};
				\node (v2) at (1.5,1.25) {$\v_2$};
				\node (v3) at (3,1.25) {$\v_3$};
				\node (v4) at (1.5,-.25) {$\v_4$};
				\node (v5) at (0,-.25) {$\v_5$};
				\draw[fill=black] (0,1) circle (1.5pt);
				\draw[fill=black] (1.5,1) circle (1.5pt);
				\draw[fill=black] (3,1) circle (1.5pt);
				\draw[fill=black] (1.5,0) circle (1.5pt);
				\draw[fill=black] (0,0) circle (1.5pt);
				\draw[]				(0,0) -- (0,1);
				\draw[]				(0,1) -- (1.5,1);
				\draw[]				(0,1) -- (1.5,0);
				\draw[]				(1.5,0) -- (1.5,1);
				\draw[]				(1.5,1) -- (3,1);
				\draw[]				(1.5,0) -- (3,1);
				
				\end{tikzpicture}	
			}
			&
 				\multicolumn{1}{c}{\begin{tikzpicture}[thick, scale=1]
				\node (v1) at (0,1.25)  {$\v_1$};
				\node (v2) at (1.5,1.25) {$\v_2$};
				\node (v3) at (3,1.25) {$\v_3$};
				\node (v4) at (1.5,-.25) {$\v_4$};
				\node (v5) at (0,-.25) {$\v_5$};
				\draw[fill=black] (0,1) circle (1.5pt);
				\draw[fill=black] (1.5,1) circle (1.5pt);
				\draw[fill=black] (3,1) circle (1.5pt);
				\draw[fill=black] (1.5,0) circle (1.5pt);
				\draw[fill=black] (0,0) circle (1.5pt);
				\draw[]				(0,0) -- (0,1);
				\draw[]				(0,1) -- (1.5,1);
				\draw[]				(0,1) -- (1.5,0);
				\draw[]				(1.5,0) -- (1.5,1);
				\draw[midarrow]				(3,1) -- (1.5,1);
				\draw[midarrow]				(1.5,0) -- (3,1);
				
				\end{tikzpicture}}
			\\
				\multicolumn{1}{c}{\text{\space\space\space\space\space\space\space\space\space\space\space\space\space\space\space}(a)}& \multicolumn{1}{c}{(b)}\\
			\\
				\multicolumn{2}{c}{\begin{tabular}{|c|c|c|c|}
				\hline	
				Index	& DP Value  & Index	& DP Value \\
				\hline
				\multicolumn{1}{|c|}{$\v_2 \rightarrow \v_1$} &\multicolumn{1}{|c|}{$\{\v_2 \rightarrow \v_1$,$\v_1 \rightarrow \v_5$,$\v_4 \rightarrow \v_1$,$\v_1 \rightarrow \v_5\}$} & \multicolumn{1}{|c|}{$\v_1 \rightarrow \v_2$} & \multicolumn{1}{|c|}{$\{\v_1 \rightarrow \v_2$,$\v_2 \rightarrow \v_3$,$\v_4 \rightarrow \v_3\}$} \\ 
				\hline
				\multicolumn{1}{|c|}{$\v_4 \rightarrow \v_1$}
				&\multicolumn{1}{|c|}{\{$\v_4 \rightarrow \v_1$,$\v_1 \rightarrow \v_5$\}} & \multicolumn{1}{|c|}{$\v_1 \rightarrow \v_4$} & \multicolumn{1}{|c|}{$\{\v_1 \rightarrow \v_4$,$\v_2 \rightarrow \v_3$,$\v_4 \rightarrow \v_3\}$} \\ 
				\hline
				\multicolumn{1}{|c|}{$\v_3 \rightarrow \v_2$} &\multicolumn{1}{|c|}{$\{\v_3 \rightarrow \v_2$,$\v_4 \rightarrow \v_1$,$\v_2 \rightarrow \v_1$,$\v_1 \rightarrow \v_5\}$} & \multicolumn{1}{|c|}{$\v_2 \rightarrow \v_3$} & \multicolumn{1}{|c|}{$\{\v_2 \rightarrow \v_3\}$}  \\ 
				\hline
				\multicolumn{1}{|c|}{$\v_4 \rightarrow \v_2$} &\multicolumn{1}{|c|}{$\{\v_4 \rightarrow \v_2\}$} & \multicolumn{1}{|c|}{$\v_2 \rightarrow \v_4$} & \multicolumn{1}{|c|}{$\{\v_2 \rightarrow \v_4\}$} \\ 
				\hline
				\multirow{2}{*}{$\v_4 \rightarrow \v_3$} &\multirow{2}{*}{{$\{\v_4 \rightarrow \v_3\}$}} & \multirow{2}{*}{$\v_3 \rightarrow \v_4$} & \multirow{2}{*}{\shortstack[l]{$\{\v_3 \rightarrow \v_4,\v_4 \rightarrow \v_1,$ \\ \;\;$\v_1 \rightarrow \v_5,\v_2 \rightarrow \v_1\}$}} \\ 
				& & &\\
				\hline
				\multirow{2}{*}{$\v_5 \rightarrow \v_1$} &\multirow{2}{*}{\shortstack[l]{$\{\v_5 \rightarrow \v_1$,$\v_1 \rightarrow \v_2$,$\v_2 \rightarrow \v_3,$ \\ \;\;\;\;\;\;\;\;\;\;\;\;\;\;\;$\v_1 \rightarrow \v_4$,$\v_4 \rightarrow \v_3\}$}} & \multirow{2}{*}{$\v_1 \rightarrow \v_5$} & \multirow{2}{*}{$\{\v_1 \rightarrow \v_5\}$} \\ 
				& & &\\
				\hline
			\end{tabular}}\\
			\\
			\multicolumn{2}{c}{(c)}\\
		\end{tabular}
		\caption{(a): UCCG $G=(\V,\E)$. (b): Obtained ICCG after intervention on node $\v_3$. (c): DP values with respect to graph $G$.}
		\label{Ex:intervention}
\end{figure}
\end{example}

\subsection{Lower Bound on the Number of Discovered Edges' Orientations}
\label{Chapter:LowerBound}
In the literature of experiment design, several objective functions have been proposed to determine which node should be intervened on in order to recover as many edges' orientations as possible. One of the main objective function is in the form of minimax  where we try to find that for each node how many edges' orientations will be discovered in the worst case scenario if we decide to intervene on that node. A naive solution is to consider all the possible orientations for the edges incident with a node and then compute the number of oriented edges can be discovered in the worst case for each of these cases. 
 However, this solution may become very time consuming and it cannot be applied in medium-size or large graphs. Here, we propose a method to calculate a lower bound on number of edges that can be oriented as a result of intervening on each node. 
 
We consider a UCCG $G=(\V,\E)$ with set of nodes $\V$ and set of edges $\E$. We denote the set of all maximal cliques in the neighborhood of a node $\v \in \V$ by $\SET{C}(\v)$. For computing the lower bound on a node such as $v$, we partition all the possible edges orientations that are connected to this node, into two cases:  (1) There exists a maximal clique $\SET{C}_k \in \SET{C}(\v)$, where at least one edge from a node $\v_i \in \SET{C}_k$ is in-going toward node $\v$ and all the edges between node $\v$ and $neigh(v)\backslash \SET{C}_k$ are outgoing from node $\v$, (2) All edges that are connected to node $v$ are outgoing from this node. For the first case, we compute a lower bound for each maximal clique $\SET{C}_k \in \SET{C}(\v)$ which means how many edges, at least, will be oriented after performing intervention on node $v$, considering edges' orientations in this case.  For the second case, we obtain the true value of number of edges that can oriented.  Finally, we compute the lower bound based on these two cases.  


Here, first we describe the lower bound on a maximal clique. We will obtain the lower bound on a maximal clique by partitioning the problem into two settings: (a) computing number of edges that will be oriented when all edges are from maximal clique to node $v$. We denote this number by $L_I$. (b) computing a lower bound on the number of edges when at least one edge is from a node in the maximal clique to node $v$ and at least one edge is oriented in the reverse direction, i.e., from node $v$ to a node in the maximal clique. We denote this lower bound by $L_C$. 

For the first setting, we can exactly find the number of oriented edges as a result of intervention. In the second setting, there exists an edge $\v_i \rightarrow \v$ that is from node $\v_i \in \SET{C}_k$ toward node $\v$ and an edge $\v \rightarrow \v_o$ that is from node $\v$ to a node $\v_o \in \SET{C}_k$. Consider the set of all in-going edges for the nodes $\v_i \in \V$ and all out-going edges for the nodes $\v_o \in \V $. Every edge in this set will be separately considered for orienting further edges and other edges in this set do not contribute to this process. Next, by post-processing these results, we obtain a lower for this maximal clique $\SET{C}_k$.

In order to compute $L_C$, we first obtain the set of edges that are oriented as a result of applying Meek functions on edges from node $v$ to nodes in $neigh(\v) \backslash \SET{C}_k$, and we denote this set by  $\SET{R}$. Next, we obtain a lower bound on the number of edges that can be oriented, when we know there exist exactly $j$ number of edges from nodes in clique $\SET{C}_k$ to node $\v$, excluding those that are in set $R$. We denote this lower bound by $P_j$. For computing $P_j$, we consider all possible cases that there are exactly $j$ number of edges from clique $\SET{C}_k$ to node $\v$. Let $\SET{I}$ be the set of these edges in one of these cases where $|\SET{I}|=j$.  
Now, for any $v_i\in \SET{I}$, we consider the orientations of edges from $v$ to $neigh(v)\backslash\SET{C}_k$ and $\v_i \rightarrow \v$ as the input and apply Meek functions in order to recover the orientations of further edges. By removing the set $\SET{R}$ from the resulted oriented edges, we can deduce the set of edges that can be oriented based on the edge $v_i\rightarrow v$. We repeat this procedure for any node $v_i\in \SET{I}$, and by considering the best case (i.e., picking the case resulting in maximum number of directed edges), we can obtain a lower bound on the number of edges that can be oriented by the whole edges $\{v_i\rightarrow v|v_i\in \SET{I}\}$. This is due to the fact that we just use the orientation of one of the edges in $\{v_i\rightarrow v|v_i\in \SET{I}\}$ in each of the cases for deriving the lower bound. Moreover, there might be multiple options for the set $\SET{I}$ with size $j$. Hence, we need to consider all these options and pick the worst one as the lower bound $P_j$. 
Similarly, we can obtain a lower bound when we know there exist exactly $j$ number of edges from $v$ to nodes in clique $\SET{C}_k$. We denote this lower bound by $Q_j$. 

Next, we try to orient edges in induced sub-graph $\E[\SET{C}_k]$. We can exactly compute the number of oriented edges in this set, if we assume that there exist exactly $j$ edges from nodes in clique $\SET{C}_k$ to node $\v$. We will show in Lemma \ref{lem:LBMC} that this value is equal to $j|\SET{C}_k-j|$. Note that having $j$ ingoing edges to $\v$ means that we will have $|\SET{C}_k|-j$ outgoing edges from node $\v$ to clique $\SET{C}_k$.

Finally, we consider the edges that we did not take into account in the previous steps but they will be oriented as a result of intervention. In particular, after performing intervention on $v$, all the edges between nodes in clique $\SET{C}_k$ and node $\v$ will be oriented. We consider one ingoing edge toward $v$ in computing $P_j$ and one outgoing edge from $v$ in computing $Q_j$. Thus, we need to consider $|\SET{C}_k-2|$ remaining oriented edges in our lower bound.

In the following, based on $P_j$'s and $Q_j$'s, we present a lower bound on number of edges that can be oriented by intervening on node $v$, constraint to have at least one ingoing edge toward node $\v$.

\begin{lemma}[Lower bound for the case of a maximal clique]
\label{lem:LBMC} 
Consider the set of ICCGs that can be obtained after intervention on node $v$, such that we have $\SET{I} \subseteq \SET{C}_k \in \SET{C}(v)$, $\SET{I} \neq \emptyset$ and $\SET{O} = neigh(v)\backslash \SET{C}_k$. The lower bound on number of oriented edges for all possible sets $\SET{I}$,  $L(\SET{C}_k,v)$, can be obtained as follows:
\begin{align*}
&L(\SET{C}_k,v) = min(L_{I} ,L_{C}),
\end{align*}
 where,
 \begin{align*}
 &L_I = \left | \underset{\v_i \in \SET{C}_k}{\bigsqcup} DP [ \v_i \rightarrow \v] \right |\\
  &L_{C} = |\SET{R}| +  \underset{l \in \{1,...,|\SET{C}_k|-1\} } {min    } P_l + Q_{|\SET{C}_k|-l} + |l|(|\SET{C}_k|-|l|) + (|\SET{C}_k|-2)\\
 &\SET{R}= \bigsqcup_{\v_o \in neigh(\v)\backslash\SET{C}_k}DP[\v \rightarrow \v_o ] \\
 &P_j = \underset{ |\SET{I}|=j,\SET{I}\subset \SET{C}_k\;\;\; } {min}\underset{\v_i \in \SET{I}} {max \;}\left | \left (DP \left [\v_i \rightarrow \v \right ] {\sqcup}  \{ \v_i \rightarrow \v_o | \v_o \in neigh(\v_i)\cap neigh(\v)\backslash\SET{C}_k \}\right) \backslash\SET{R} \right| \\
 &Q_j = \underset{|\SET{O}|=j,\SET{O}\subset \SET{C}_k\;\;\; } {min}\underset{\v_o \in \SET{O}} {max \;} | DP [\v \rightarrow \v_o] \backslash\SET{R}|.
 \end{align*}
\end{lemma}
In Lemma \ref{lem:LBMC}, the minimum value of settings (a) and (b) is considered as a lower bound for the setting where the maximal clique $\SET{C}_k$ has at least one outgoing edge toward the target node. In this bound, the computation of $L_I$ is straight forward. In order to compute $L_C$, we need to consider all the cases of having $l=1,\cdots, |\SET{C}_k|-1$ number of edges from the maximal clique $\SET{C}_k$ to node $v$ and compute $P_l$ and $Q_{|\SET{C}_k|-l}$ (note that we have $|\SET{C}_k|-l$ number of edges from node $v$ to the maximal clique). In each case, we also need to take into account the size of the set $\SET{R}$, the number of oriented edges in the induced sub-graph $E[\SET{C}_k]$, i.e., $|l|(|\SET{C}_k|-|l|)$, and the other edges that are oriented as the result of intervention on $v$, i.e., $|\SET{C}_k|-2$. By picking the worst case scenario for different values of $l$, we can obtain the lower bound $L_C$.

In order to compute $P_l$ and $Q_{|\SET{C}_k|-l}$, we can enumerate all the possible sets $\SET{I}$ of size $l$ which might be time consuming. Instead,
we present Algorithm \ref{Alg:LowerBoundMaximalClique} which can compute $L(\SET{C}_k,v)$ in an efficient manner. In Line 5, we compute $L_I$ and we compute set $\SET{R}$ in Line 6. In Lemma \ref{lem:LBMC}, we need to compute the number of oriented edges as a result of applying Meek function $M_{14}$ in the worst case for two scenarios; In the first scenario, only one directed edge exists from a node in maximal clique $\SET{C}_k$ to node $v$ while in the second scenario, only one directed edge exists from node $v$ to maximal clique $\SET{C}_k$. In Line 9, we obtain the number of edges that can be oriented as a result of applying Meek function $M_{14}$ where there exists a directed edge from maximal clique $\SET{C}_k$ to node $\v$, i.e., $DP[v_r\rightarrow v]$. In Line 10, we obtain the number of edges that can be oriented as a result of applying Meek function $M_{14}$ where there exists a directed edge from node $\v$ toward maximal clique $\SET{C}_k$, i.e., $DP[v\rightarrow v_r]$. In Line 11, we sort the obtained values of $P$ and $Q$ in the ascending order. In Lines 12-13, we search for minimum number of oriented edges where the number of ingoing edges toward node $v$ can be varied from $1$ to $|\SET{C}_k|-1$. Selecting $l$-th item from sorted variables $P$ and $Q$ is the same as considering values of $P_l$ and $Q_l$ in Lemma \ref{lem:LBMC}, respectively. To see why this is true, note that for achieving the lower bound $P_l$, we just need to consider $l$ ingoing edges toward node $v$ corresponding to the $l$ lowest values in array $P$ as the set $\SET{I}$. In the same vein, it can be seen that $l$-th entry of array $Q$ is equal to $Q_l$.

\begin{algorithm}
	\caption{Computation of maximal clique lower bound}
	\label{Alg:LowerBoundMaximalClique}
	\begin{algorithmic}[1]
		
		\State \textbf{Input:} A UCCG  $G=(\V,\E)$
		\State \textbf{Output:} lower bound on number of oriented edges after applying Meek functions 
		\Function{$L$}{$G$, $\SET{C}_k$, $v$, $DP$} 
		\State 
		$r \leftarrow 0$
		\State
		$\mathcal{L} \leftarrow \left | \bigsqcup_{\v_i \in \SET{C}_k} DP [ \v_i \rightarrow \v] \right |$
		\State $\SET{R}= \bigsqcup_{\v_o \in neigh(\v)\backslash\SET{C}_k}DP [\v \rightarrow \v_o ]$
		\For{$\v_r \in \SET{C}_k$}
		\State	 $r \leftarrow r + 1$
		\State	 $P[r] = |\left (DP [\v_r \rightarrow \v] \bigsqcup_{\v_o \in neigh(\v_i)\cap neigh(\v)\backslash\SET{C}_k}  \{ \v_r \rightarrow \v_o  \}\right) \backslash\SET{R}|$
		\State	 $Q[r] = |DP [\v \rightarrow \v_r  ]\backslash\SET{R}|$
		\EndFor 
		\State $P \leftarrow sort(P)$,$Q \leftarrow sort(Q)$
		
		\For{$l \leftarrow 1:|\SET{C}_k|-1$}
		\State	$\mathcal{L} = min(\mathcal{L},|\SET{R}| + P[l] + Q[|\SET{C}_k|-l] + l(|\SET{C}_k|-l))  +(|\SET{C}_k|-2)$	
		\EndFor
		\EndFunction
		\State \textbf{Return} $\mathcal{L}$
	\end{algorithmic}
\end{algorithm}

\begin{remark}
Having access to entries of $DP$ table, the
	computational complexity of Algorithm \ref{Alg:LowerBoundMaximalClique} for a node $\v$ and maximal clique $\SET{C}_k$ is $\mathcal{O} (\omega \log \omega)$, where $\omega$ is the clique number.
\end{remark}

In the following, we propose a lower bound on number of oriented edges as a result of performing an intervention on a node. We do not consider any assumption on edges' orientations. Possible  orientations of edges incident with the intervened node can be partitioned into two groups: In the first group, all the edges in the neighborhood of the target node are outgoing from it. In the second group, there is an edge (or possibly edges) from a maximal clique to the target node. In Lemma \ref{lem:LBMC}, we investigated the lower bound for a maximal clique with at least one ingoing edge toward the target node. Hence, the lower bound in second group can be achieved by calculating the minimum value of lower bounds for all maximal cliques in the neighborhood of the target node. Finally, the lower bound on a target node can be obtained as given in the following lemma.

\begin{theorem}[Lower bound for the case of a node]
\label{thm:LowerNode} The lower bound on the number of edges that will be oriented after performing an intervention on node $\v$, $L(v)$, in an arbitrary UCCG $G$ can be obtained as follows:

\begin{align*}
&L(v) = min \left (   \left | \underset{\v_o \in neigh(\v)}{\bigsqcup} DP [\v \rightarrow \v_o]  \right|, 
 \underset{\SET{C}_k \in \SET{C}(\v)}{min} L(\SET{C}_k,\v)
  \right).
\end{align*}

\begin{remark}
The computational complexity of lower bound in Theorem  \ref{thm:LowerNode} for a node $\v$ is $\mathcal{O}(M \omega \log \omega)$, where $\omega$ is the clique number, and $M$ is the maximum number of maximal cliques in the neighborhood of a node.
\end{remark}

\end{theorem}

\subsection{Orientation Consistency}
In this part, we assume that there is an expert who reveals the orientations of some undirected edges in a PDAG. This prior knowledge can also be obtained from a partial ordering on nodes \citep{Hauser12,Wang17} or restricting the causal relationships to a certain model  \citep{Eigenmann17,Hoyer12,Buhlmann18}.
We will take advantage of Meek functions' properties in the previous section to validate the correctness of these orientations in an efficient manner.  

Suppose that we are trying to recover the underlying ground truth graph, and our identification algorithm discovered a PCCG. In this stage, an expert suggests to orient some further undirected edges in the obtained PCCG based on a domain knowledge. We are interested in checking whether  the combination of already directed edges and recently added orientations are consistent or not. As underlying ground truth graph is a DAG, we just need to find a DAG being compatible  with all the given directed edges. To do so, in the following, we first provide some definitions.

\begin{definition}[Consistent edges]
Consider a graph $G=(\V,\E)$ with set of nodes $\V$ and set of edges $\E$. We say that in a graph with skeleton $\overline{\E}$, two directed edges $\v_i \rightarrow \v_j$, $\v_r \rightarrow \v_t$ in $\overrightarrow{\E}$ are consistent edges if for any $\v_k \rightarrow \v_l \in    DP[ \v_i \rightarrow \v_j ]$, we have: $\v_l \rightarrow \v_k \notin DP[\v_r \rightarrow \v_t ]$.
\end{definition}

\begin{definition}[Consistent set]
Consider a PCCG $G = (\V,\E)$ with  set of nodes $\V$ and set of edges $\E$. We say that the set of directed edges $\overrightarrow{\E}$ in structure $\overline{\E}$ is a consistent set if each pair of directed edges in this set are consistent.
\end{definition}

For a DAG $G$ with the set of edges $\E$, any set of $\E^\prime \subseteq \overrightarrow{\E}$ is a consistent set on a graph with structure $\overline{\E}$, where there is no v-structure in $G$. This is because applying Meek function $M_{14}$ on $\E^\prime \sqcup \overline{\E}$ does not make any conflict in edges orientation.

Now, we will use the definition of edge consistency to validate the domain knowledge of the expert. Consider a graph that consists of a combination of directed and undirected edges. The following theorem presents two conditions to check whether these directed edges are subset or equal to edges in a directed acyclic graph or not. 

\begin{theorem}
\label{thm:Consistency}
Consider a PCCG $G = (\V,\E)$ with  set of nodes $\V$ and set of edges $\E$. There exists a DAG $G'=(\V',\E')$ which is in the MEC of graph $G$, such that $\overrightarrow{\E} \subseteq  \overrightarrow{\E'}$ and $\overline{\E} = \overline{\E'}$ if and only if:
\begin{itemize}
	\item There is no directed cycle in $\E$.
	\item $\E$ is a consistent set.
\end{itemize}
\end{theorem}

We can utilize Theorem \ref{thm:Consistency} to validate the domain knowledge of the expert. To do so, we check whether any pair of directed edges in $\SET{E}$ are consistent from the DP table. For detecting a directed cycle, we can execute DFS algorithm from all vertices on a graph that is extracted from $G$ by removing its undirected edges.


\section{Experiments}
\label{sec:Experiments}
In this section, we provide some experiments for various application of COL problems based on our results in previous sections. 
We propose different algorithms in each subsection in the following. We utilize DP table in designing almost all of these algorithms. Hence, we call the class of the proposed algorithms ``Causal Orientation Learning with Dynamic programming (COLD)". In our experiments, we define the edge density as the ratio of edges to the maximum number of possible edges in a graph of size $n$, i.e., $n(n-1)/2$. Moreover, we define the average degree as the ratio of edges to the number of  the nodes.

In all experiments, we generated chordal graphs by the method which has been presented in \cite{Lilian2008}. For each point in all figures, we averaged over 100 instances of the chordal graphs, otherwise we explicitly mention it. Our generated graphs are based on the similar networks given in bnlearn\footnote{https://www.bnlearn.com/}. The codes for proposed methods in this part are available in \hyperref[Codes]{https://github.com/raminsafaeian/COLD}.

\subsection{Recovering Essential Graphs}

One of the well-known methods for obtaining the essential graph from the observational data is PC algorithm \citep{Spirtes00}. This algorithm has three main steps: (1) performing conditional independence tests between random variables in order to construct skeleton (2) identifying v-structures from separation sets (3) discovering more edges' orientations by applying Meek rules on graph including v-structures. According to Theorem \ref{thm:MPDAGcompleteness}, for further recovering edges' orientations, it suffices to apply Meek rules 1, 2 and 3 in third step of PC algorithm. In the following we propose a method to accelerate obtaining essential graph.

Similar to our previous DP table construction in Algorithm \ref{Alg:DPONE}, and considering Property 2 in Lemma \ref{lem:MeekPropertiesObserved}, we propose DPO method for the case that there exist v-structures in the graph. We show how one entry of DPO table is filled in Algorithm \ref{Alg:DPO}. The main difference between this algorithm and what we proposed in Algorithm \ref{Alg:DPONE} is that the Meek function $M_4$ is removed and v-structures orientations are considered for further recovering other edges' orientations.
Using DPO table enables us to accelerate Meek function $M_1$ which in turn will accelerate execution time of getting the essential graph from the MPDAG.

Consider an MPDAG $G = (\V,\E)$ with  set of nodes $\V$ and set of edges $\E$. We propose Algorithm \ref{Alg:EssentialGraph} for discovering further edges' orientations in order to obtain the essential graph. Based on Property 1 in Lemma \ref{lem:MeekPropertiesObserved}, we apply Meek function $M_3$ once, in the beginning of algorithm in Line 7. In Lines 8-14, we repetitively orient further edges by applying Meek functions $M_1$ and $M_2$. Applying Meek function $M_1$ is based on our DPO table in Lines 9-12. We call this algorithm ``COLD (Essential)".

We compared the execution time of our proposed COLD algorithm with the conventional
method for the third step of PC algorithm. Conventional method is the one that we apply
Meek rules in no particular order to orient further edges, until no more edge can be oriented.
In order to have a fair comparison, we implemented a function for the applying Meek rules similar to the one in \cite{PCAlg} library in python and considered it as the conventional method. For generating MPDAGs, we first constructed random chordal graphs. Then, we use lexicographic BFS or LexBFS \citep{Donald1970} to consider an ordering on the vertices and obtain a DAG. Then, we remove one of the edges such that at least one v-structure be created. Then, we will get the MPDAG from this generated DAG.

\begin{algorithm}[ht]
	\caption{Computation of $M_{1}( \{ \v_s \rightarrow \v_d \} \sqcup  \overline{\E} )$ using dynamic programming method, considering v-structures.}
	\label{Alg:DPO}
	\begin{algorithmic}[1]

		\State \textbf{Input:} 			$G=(\V,\E)$, a MPDAG
		\State \textbf{Output:} 		Oriented edges as a result of applying function $M_{1}$ on $\{\v_s \rightarrow \v_d\} \sqcup {\E}$ which is in $DPO[\v_s \rightarrow \v_d]$
		\State $\E_v\leftarrow$ set of v-structure sub-graphs in ${\E}$
		\Function{$OrientOneEdge$}{$DPO$, $\v_s$, $\v_d$, $\E_v$} 

		\State $DPO[\v_s \rightarrow \v_d] \leftarrow \{\v_s \rightarrow \v_d\}$
		\For{$\v_j \in neigh(\v_d) \backslash neigh(\v_s)$} 
			\If{$\v_j \rightarrow \v_d \notin \E_v$}
			\If{$DPO[\v_d \rightarrow \v_j] = NULL$}
				\State $DPO \leftarrow OrientOneEdge(DPO,\v_d,\v_j,\E_v)$
			\EndIf
			\State $DPO[\v_s \rightarrow \v_d] \leftarrow DPO[\v_s \rightarrow \v_d] \sqcup DPO[\v_d \rightarrow \v_j]$
    	\EndIf
		\EndFor 
		\EndFunction
		\State \textbf{Return} $DPO$
	\end{algorithmic}
\end{algorithm}

\begin{algorithm}[ht]
	\caption{COLD (Essential)}
	\label{Alg:EssentialGraph}
	\begin{algorithmic}[1]

		\State \textbf{Input:} 		MPDAG graph $G=(\V,\E)$	\State \textbf{Output:} 		Essential graph of $G$
		
		\State Fill all $DPO$ entries with $NULL$ value

		\Function{Essential}{$G$} 
		\State $\E_v \leftarrow$ set of v-structure sub-graphs in ${\E}$
		\State $\E_o \leftarrow \E$
		\State $\E_t = M_3(\E)$
		\While{$|\overrightarrow{\E}_t| > 0$} 
		    \For{$\v_s \rightarrow \v_d \in \overrightarrow{\E}_t$}
			\If{$DPO[\v_s \rightarrow \v_d] = NULL$}
				\State $DPO \leftarrow OrientOneEdge(DPO,$ $\v_s,$ $\v_d,$ $\E_v)$
			\EndIf
			\State $\E_o = \E_o \sqcup DPO[\v_s \rightarrow \v_d]$
			    		\EndFor
    		\State $\E_t = M_2(\E_o)\backslash \overrightarrow{\E}_o$

		\EndWhile 
		\State \textbf{end}
		\EndFunction
		\State \textbf{Return} $G=(\V,\E_o)$
	\end{algorithmic}
\end{algorithm}

\begin{figure}[ht]
		\centering

		\begin{tabular}{*{2}{c}}
			\includegraphics[]{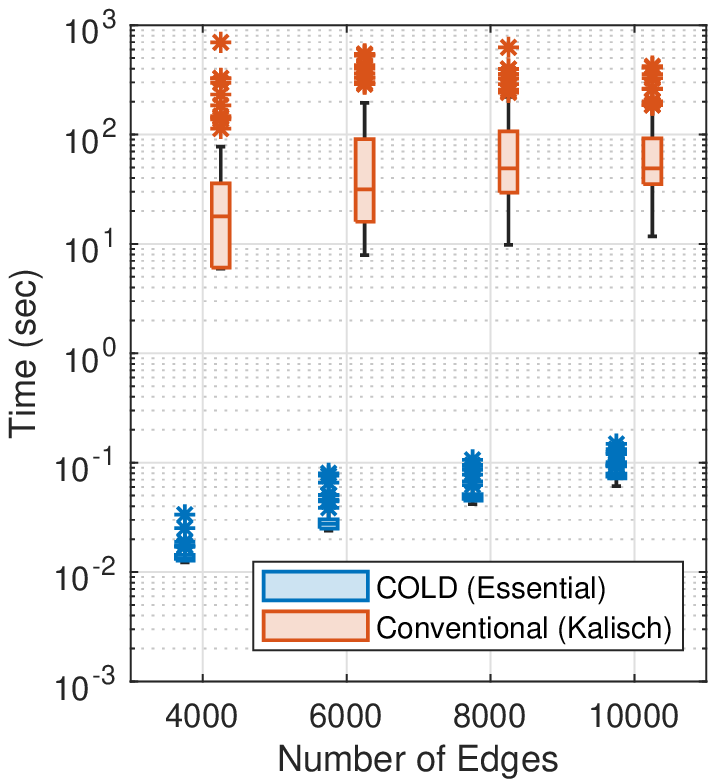}
			&
			\includegraphics[]{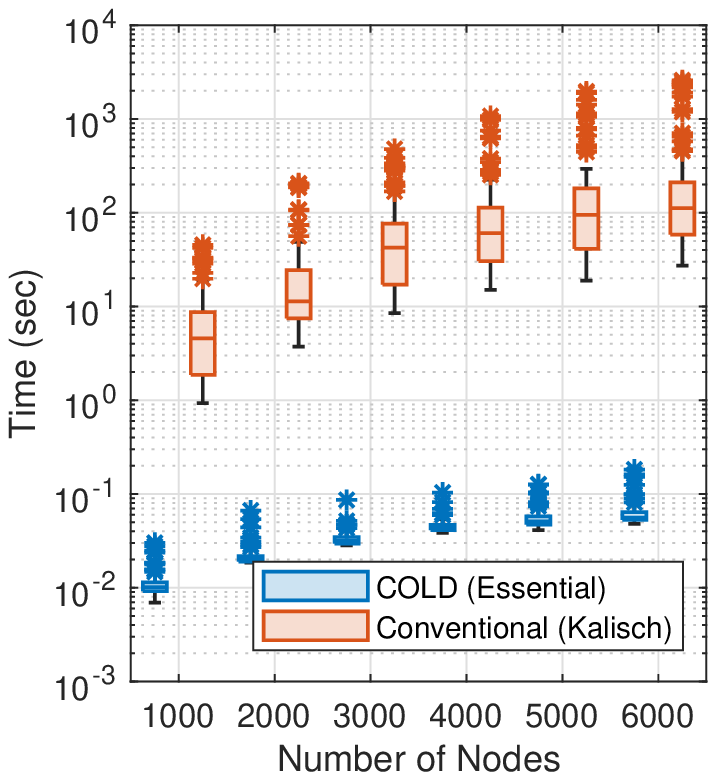}
			\\
			$\;\;\;\;\;\;\;\;\;\;\;\;$(a) & $\;\;\;\;\;\;\;\;\;\;\;\;$(b)
		\end{tabular}
	\caption{(a): 
	Comparison between execution times of COLD (Essential) and conventional method, given a MPDAG, versus the number of edge for graphs with 3000 number of nodes (b): Comparison between execution times of COLD (Essential) and conventional method versus the number of nodes for graphs with average degree of 2.2.}
	    \label{Fig:COLDEssential}

\end{figure}
Figure \ref{Fig:COLDEssential}(a) shows the execution time of COLD (Essential) and conventional method versus number of edges for graphs with $3000$ nodes. As can be seen, the COLD (Essential) is outperform the conventional method by a factor up to $20788$. 
For instance, for a graph with $3000$ nodes and $4000$ edges, the execution time of COLD (Essential) is about $0.0335$ seconds, while the execution time of  conventional method is about $696.4$ seconds. Comparison of these two methods for graphs with fixed average degree, $2.2$, versus number of graph nodes is depicted in \ref{Fig:COLDEssential}(b). Again, the proposed algorithm has much lower execution time with respect to the conventional method. For instance, for graphs of size $6000$, the speed up ratio is about $24910$, where the execution times of COLD (Essential) and conventional method are $0.0994$ and $2476$ seconds, respectively.

\subsection{Accelerating MEC Size Computation}
The size of a MEC is the number of DAGs it contains. Main approach in previous work is to partition existing DAGs in an MEC in order to count them. Total number of DAGs can be computed by summation over the number of DAGs in each partition. There are different methods for such a partitioning. For instance, \cite{He15} performed partitioning of DAGs in an MEC based on their roots\footnote{A node is defined as a root when all the edges are outgoing from this node toward its neighbors. }. Additionally, they proposed some closed-form formula to compute number of DAGs in an MEC when there is a specific relation between number of edges and number of nodes in the graph. They extended their work by considering some special structure ``core graphs'' in \cite{He16} to efficiently count the number of DAGs in an MEC. \cite{Saleh19} proposed to partition the MEC based on its clique tree representation. Additionally, they have proposed a dynamic programming method to store the computed size of each sub-graph to avoid multiple calculation of its size. Similar to \cite{Saleh19}, \cite{Talviti19} proposed a dynamic programming method for MEC size calculation. Recently, \cite{Teshnizi20} proposed to compute the MEC size by partitioning DAGs in MEC by all possible orientations for edges that are connected to a node. 

We can take advantage of Meek functions properties to accelerate any previous method that uses  Meek rules as a subroutine. Herein, as an example, we will accelerate MEC size computation method proposed by \cite{He15}. In this method, given a UCCG $G$, its size can be computed as follows:
\begin{align*}
&Size(G) = \underset{v \in \V\;}{\sum\;} \underset{{C}_i \in \mathcal{C}(G^\v)}{\prod}  Size({C}_i),
\end{align*}
where $Size(G)$ is the size of MEC corresponding to $G$ and $G^\v$ is a graph as a result of setting node $\v$ as a root in graph $G$. We utilize our proposed Meek function properties, such as Proposition \ref{Prop:DP} and Theorem \ref{Thm:InterventionM124}, and also dynamic programming method proposed by \cite{Saleh19} to accelerate UCCG size calculation. We present our proposed algorithm for counting number of DAGs in an MEC as ``COLD (MECSize)", in Algorithm \ref{Alg:MECSize}.

In Algorithm \ref{Alg:MECSize}, Lines 7-11 is based on what have been proposed in \cite{He15}. In these lines, we check whether the UCCG satisfies a specific relation between number of nodes and edges. We can determine the MEC size of these graphs merely from number of nodes. If we cannot calculate MEC size based on these lines, we execute Lines 12-13 to know whether we have already computed the MEC size for this UCCG or not. If that is the case, we will use the stored value in the memory. Otherwise, we execute Lines 14-17 for MEC size computation. In Line 15, we obtain the edges that can be oriented as a result of considering a node $v$ to be a root. In Line 16, we remove these directed edges from the graph and extract the chain components of the remaining graph. For each chain component, the size function will be called in Line 17. Finally, in Line 18, we store the computed size in memory for using in the next calls.

\begin{algorithm}
	\caption{COLD (MECSize) Algorithm}
	\label{Alg:MECSize}
	\begin{algorithmic}[1]
		\State \textbf{Input:} 			UCCG $G=(\V,\E)$
		\State \textbf{Output:} 		$|MEC(G)|$
		\State \text{SizeDP} \text{is a storage indexed with set of nodes} $\SET{T}$ and is initialized by NULL for each $\SET{T} \subseteq \V$

		\Function{$Count$}{$G=(\V,\E)$}
		\State $p \leftarrow |\V|$
		\State \textbf{switch}\text{ $|\E|$ \textbf{do}}
		\State \textbf{\space\space\space\space case} $p - 1$ \textbf{return} $p$;
		\State \textbf{\space\space\space\space case} $p$ \textbf{return} $2p$;
		\State \textbf{\space\space\space\space case} $p(p - 1)/2 - 2$ \textbf{return} $(p^2 - p - 4)(p - 3)!$;
		\State \textbf{\space\space\space\space case} $p(p - 1)/2 - 1$ \textbf{return} $2(p - 1)! - (p - 2)!$;
		\State \textbf{\space\space\space\space case} $p(p - 1)/2$ \textbf{return} $p!$;
		\If{$\V$ is in SizeDP}
		\State \textbf{return} $ SizeDP[\V]$	
		\EndIf
		\For{each $\v$ in $\V$} 
		\State $\E_r = \bigsqcup_{ \v_o \in neigh(\v) }DP [\v \rightarrow \v_o ]$
		\State $G' \leftarrow G=(\V, \overline{\E} \backslash \overline{\E}_r)$
		\State $S_v =  \prod_{\mathcal{C}_i \in \mathcal{C}(G')}Count(\mathcal{C}_i)$
		\EndFor  
		\State $SizeDP[\V] =  \sum_{\v \in \V }S_v$
		\State \Return $SizeDP[\V]$
		\EndFunction
	\end{algorithmic}
\end{algorithm}

We compared COLD (MECSize) algorithm with the previous state of the art algorithm, LazyCount \citep{Teshnizi20} and a former work MemoMAMO \citep{Talviti19}. The size of graphs in this experiment is equal to $30$.  
As can be seen in Figure \ref{Experiment:COLDMEC}, our proposed algorithm performs better than the others for a wide range of number of edges.
Moreover, performance of LazyCount degrades by decreasing the number of edges. While performance of MemoMAMO degrades by increasing the number of edges. For instance, for graphs with $350$ edges, the speed up ratios with respect to MemoMAMO and LazyCount are $2.88$ and $1.83$, respectively.

\begin{figure}[ht]
	\centering
	\includegraphics[]{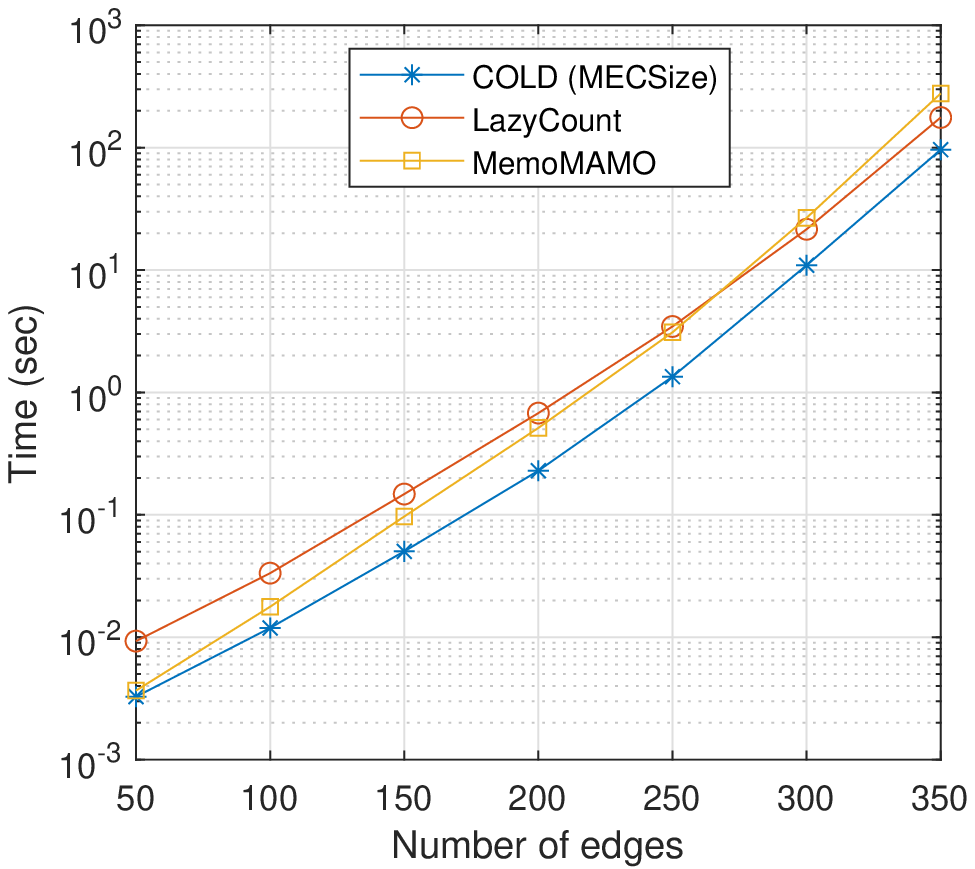}
	\caption{Running times of COLD (MECSize), LazyCount, and MemoMAMO for graphs with 30 number of nodes.}
	\label{Experiment:COLDMEC}
\end{figure}

\subsection{Acceleration of Experiment Design in Active Learning Setting}
\label{Sec:Exp_MinMax}
In active learning setting, we determine the nodes for next interventions after observing the results of previous interventions. Some previous work tried to use gathered data from previous interventions to design next intervention set based on a Bayesian method \citep{Ness17,Kristjan19,Uhler19}. For instance,  \cite{Kristjan19} proposed a method to search for a node that has the maximum posterior probability to be a root variable. This method is mainly suitable for graphs with tree structures. Some other work assumed that infinite samples are available after performing intervention \citep{Hauser14,He08,Kocaoglu17,Shanmugam15,Saleh18}. Therefore, the edges' orientations incident to the intervened variable can be exactly determined. 

In the case of infinite samples, previous works have considered different objective functions for the problem of experiment design. \cite{He08} proposed to select candidate intervention nodes based on mutual information criterion. 
Another approach that has been proposed by this work is to select a node with maximum number of neighbors. \cite{Hauser14} considered an objective function to select a node for intervention that results in orienting maximum number of edges in the worst case scenario after  the intervention. More specifically, for a given UCCG $G=(\V,\E)$, the following objective function is considered:
\begin{align*}
&v^\star \in \underset{v \in \V} {argmax\;\;} \underset{\SET{I} \subseteq \SET{C}_k, \SET{C}_k \in \SET{C}(\v)} {min \;\;} \left | M_{124} \left ( \SET{S} \sqcup \overline{\E} \right )\right |,
\end{align*}
where $\SET{S} = \{\v_i \rightarrow \v|\v_i \in \SET{I} \} \sqcup \{\v \rightarrow \v_o| \v_o \in neigh(\v)\backslash\SET{I} \}$. We call this objective function, ``MinMax" function. In MinMax function, for every node in the graph, the minimum number of oriented edges, in the worst case, will be computed when we intend to perform an intervention on that node. The possible edges' orientation after intervention will be partitioned by considering different ingoing edges toward the target node \citep{Hauser14}. We know that the ingoing edges must be a subset of a maximal clique. Thus, we search over all possible ingoing edges sets like $\SET{I}$ in every maximal clique $\SET{C}_k$.
For each of these orientations, we apply Meek function $M_{124}$ to discover further edges' orientations. As mentioned earlier, using Meek function properties, we can accelerate any function that needs Meek rules to be applied. In the following, we show that using our proposed method for computing Meek function results, we can accelerate the process of optimizing MinMax objective function.

We compared our proposed algorithm, which we call ``COLD (MinMax)", with LazyIter \citep{Teshnizi20} method, which is the state of art  in solving MinMax problem. COLD (MinMax) uses DP tables and utilizes Theorem \ref{Thm:InterventionM124} to apply Meek rules, rather than applying conventional Meek rules.  The execution times for these two methods are given in Figure \ref{Experiment:MinMaxVsHauser}.  The curves are the average of execution times over 100 generated chordal graphs.
 Figure \ref{Experiment:MinMaxVsHauser}(a) shows that the execution time of finding the best node versus number of edges for graphs with 1000 number of nodes. Figure \ref{Experiment:MinMaxVsHauser}(b) shows the execution time of finding the best node versus number of node for graphs graphs with 1.1 average degree.
 
As can be seen in Figure \ref{Experiment:MinMaxVsHauser}, the execution time of our proposed algorithm  is considerably less than the one for LazyIter Algorithm. As expected, we gain by avoiding from applying the enormous number of Meek rules using DP table.

\begin{figure}[ht]
		\centering

		\begin{tabular}{*{2}{c}}
			\includegraphics[]{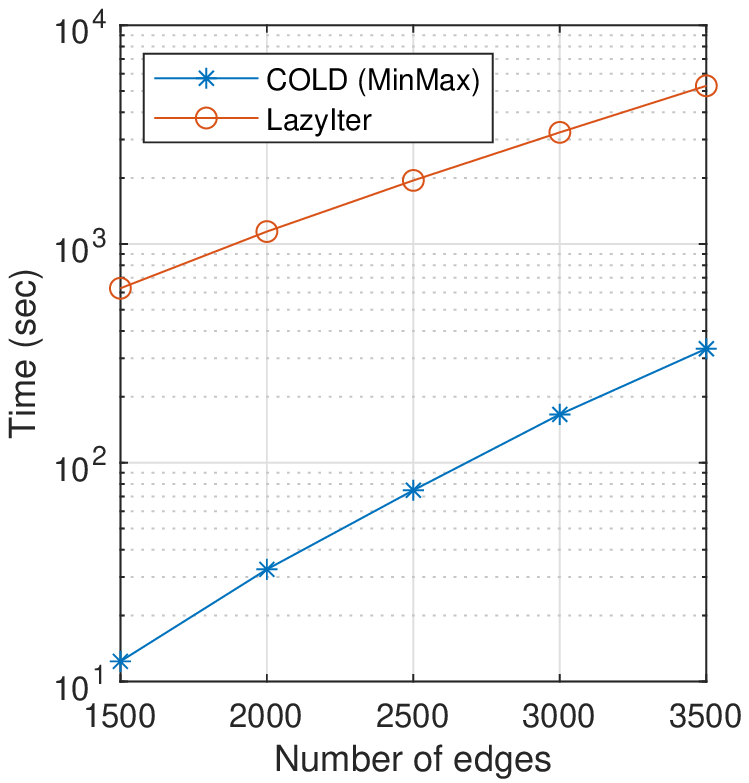}
			&
			\includegraphics[]{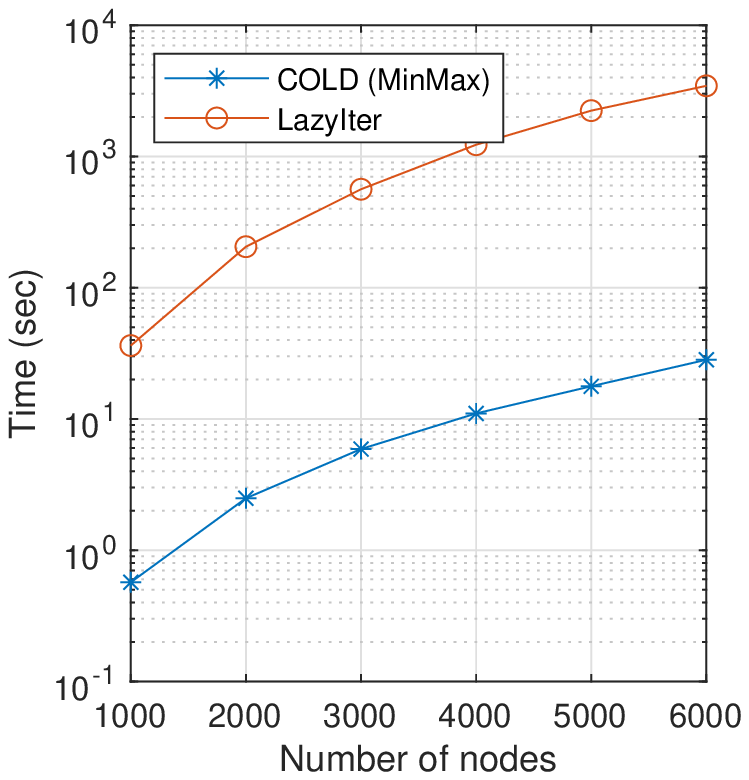}
			\\
			$\;\;\;\;\;\;\;\;\;\;\;\;$(a) & $\;\;\;\;\;\;\;\;\;\;\;\;$(b)
		\end{tabular}
	\caption{(a): Comparison between execution times of COLD (MinMax) and LazyIter \citep{Teshnizi20} versus number of edges for graphs with 1000 number of nodes (b): Comparison between execution times of COLD (MinMax) and LazyIter \citep{Teshnizi20} versus number of nodes for graphs with 1.1 average degree.}
	    \label{Experiment:MinMaxVsHauser}

\end{figure}

\subsection{Quality of the Proposed Lower Bound}
In this section, we investigate the quality of our proposed lower bound with respect to number of edges that can be oriented in the worst case scenario for each node.  We evaluated the lower bound in two cases: 1- The edge density is fixed and we increase the number of nodes, 2- The number of nodes is fixed and we increase the number of edges. We provide the summary of these simulation results for these two cases in Table \ref{Table:LBFixedDensity} and \ref{Table:LBFixedNode}, respectively. Note that for finding the minimum number of oriented edges in the worst case, we enumerated all consistent DAGs. 

\begin{table}[ht]
	\begin{center}
		\begin{tabular}{|c|c|c|c|c|c|c|}
			\hline
			$|\V|$ & Edge density & Exact \% & Norm. avg. worst case & Mean gap & Avg. Time (Sec)
			\\ \hline
			20 & 0.4 & 0.9915 &     0.0029 &    0.0245 &    0.0080 \\
			30 & 0.4 & 0.9827 &     0.0088 &    0.1190 &    0.0444 \\
			40 & 0.4 & 0.9722 &     0.0159 &    0.2843 &    0.1596 \\
			50 & 0.4 & 0.9708 &     0.0213 &    0.5328 &    0.4981 \\			
			60 & 0.4 & 0.9627 &     0.0228 &    0.8143 &    1.2824 \\		
			\hline
		\end{tabular}
		\caption{Lower bound comparison with the true value for different number of oriented edges with fixed edge density}
		\label{Table:LBFixedDensity}
	\end{center}
\end{table}

\begin{table}[ht]
	\begin{center}
		\begin{tabular}{|c|c|c|c|c|c|c|}
			\hline
			$|\V|$ & Edge density & Exact \% & Norm. avg. worst case & Mean gap & Avg. Time (Sec)
			\\ \hline
			40 &  0.1 &  0.9996 &  0.0002  &   0.0004 &    0.0064 \\
			40 &  0.2 &  0.9854 &  0.0071  &   0.0563 &    0.0319 \\
			40 &  0.3 &  0.9779 &  0.0120  &   0.1549 &    0.0843 \\
			40 &  0.4 &  0.9722 &  0.0159  &   0.2843 &    0.1596 \\
			40 &  0.5 &  0.9762 &  0.0157  &   0.3555 &    0.2605 \\
			40 &  0.6 &  0.9781 &  0.0129  &   0.4136 &    0.3654 \\
			40 &  0.7 &  0.9825 &  0.0104  &   0.4338 &    0.4671 \\
			\hline
		\end{tabular}
		\caption{Lower bound comparison with the true value for different edge density with fixed number of nodes}
		\label{Table:LBFixedNode}
	\end{center}
\end{table}

In these tables, the ``exact $\%$" column shows in how many cases, our proposed lower bound is equal to the true value. In ``norm avg worst case" column, we considered the maximum difference between the true value and the proposed lower bound over all nodes for each graph, and then averaged over all graphs and divided it by the number of existed edges for that setting. 
In ``mean gap" column, we computed the average difference between the true value and the lower bound for each node in all graphs. The last column shows the average execution time for computing the lower bounds for all nodes in each graph.

The results in Table \ref{Table:LBFixedDensity} and \ref{Table:LBFixedNode} show that in almost all cases, the lower bound is exactly equal to the true value. Moreover, in other few cases, the gap between the lower bound and the true value is negligible.

\subsection{Intervention Design Based on the Lower Bound Criterion}
In this part, we propose a new objective function to find the best node for intervention. This function reduces the computational complexity while preserving the number of interventions in order to fully identify the whole causal structure.
Our objective function is defined as the following:
\begin{align*}
&v^\star \in \underset{v \in \V} {argmax\;\;} L(v),
\end{align*}
where the $L(v)$ is the lower bound on number of oriented edges after intervention on node $v$. Based on the above objective function we select a node for intervention that has maximum lower bound with respect to the other nodes. 

In the previous part, we evaluated the quality of our proposed lower bound. The reason for using lower bound instead of the true value is that lower bound is a good estimator of it while saving the computation time. We compared our heuristic algorithm based on our lower bound which we call it ``COLD (LB)" with the one that we proposed in Section \ref{Sec:Exp_MinMax}, i.e., COLD (MinMax).

\begin{figure}[ht]
		\begin{tabular}{*{2}{c}}
			\includegraphics[]{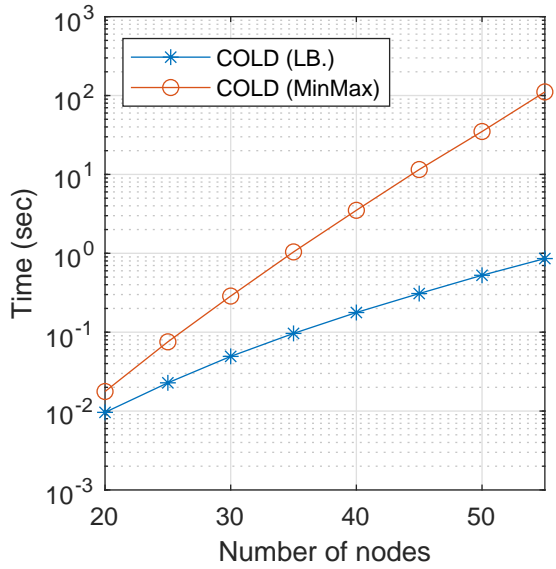}
			&
			\includegraphics[]{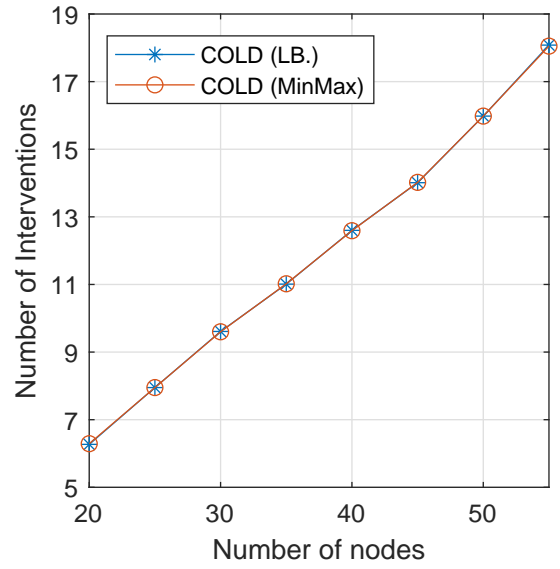}
			\\
			$\;\;\;\;\;\;\;\;\;\;\;\;$(a) & $\;\;\;\;\;\;\;\;\;\;\;\;$(b)
			\\
			\includegraphics[]{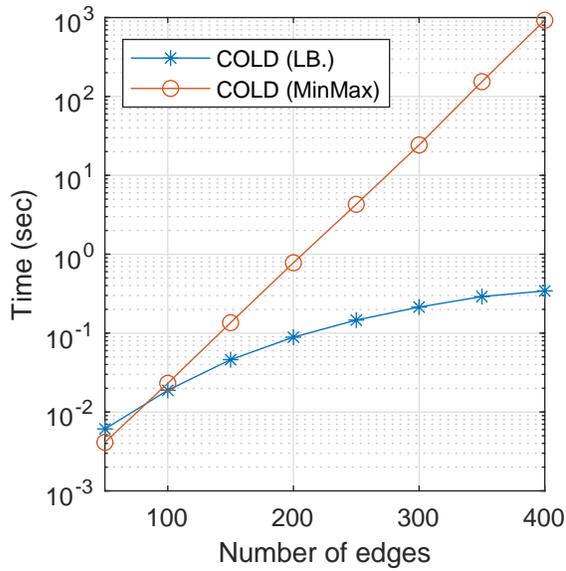}
			&
			\includegraphics[]{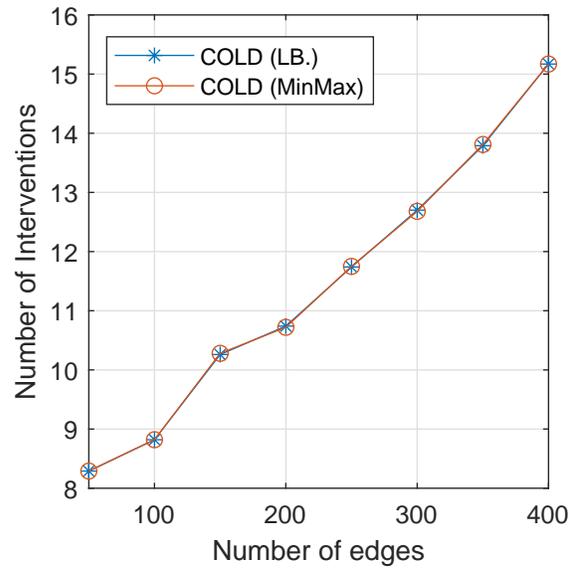}
			\\
			$\;\;\;\;\;\;\;\;\;\;\;\;$(c) & $\;\;\;\;\;\;\;\;\;\;\;\;$(d)
		\end{tabular}
		\caption{Comparison between COLD (LB) and COLD (MinMax) (a): execution time versus number of nodes for graphs with edge density of 0.35. (b): Average number of interventions for full identification versus number of nodes for graphs with edge density of 0.35. (c): execution times versus number of edges for graphs with 40 nodes. (d): Average number of interventions that is needed for full identification versus number of edges for graphs with 40 nodes.}
		 
		\label{Experiment:LBvsMinMAx}
\end{figure}

The result of comparison between two algorithms COLD (MinMax) and COLD (LB) is shown in Figure \ref{Experiment:LBvsMinMAx}. We see that COLD (LB) is considerably faster than COLD (MinMax) in both cases of increasing number of nodes (Figure \ref{Experiment:LBvsMinMAx}(a)) or increasing number of edges (Figure \ref{Experiment:LBvsMinMAx}(c)). We plot the average number of interventions that is needed for full identification in Figure \ref{Experiment:LBvsMinMAx}(b) and \ref{Experiment:LBvsMinMAx}(d). We can see that in both of these plots, the average number of interventions for full identification is the same for both algorithms.

In this part, we evaluate our algorithm  for the finite-sample case. Here, our main goal is to demonstrate how one can use the infinite sample based algorithms in the finite sample case.  
Thus, we just compared our algorithm with two baselines: 1- ``Random Naive'' algorithm, which always selects a random node for the next intervention from the set of all nodes, and 2- ``Random Chordal'' algorithm, which selects a random node from set of nodes that those nodes have at least one connected undirected edge in the obtained $\mathcal{I}$-essential graph after each intervention.

\begin{figure}[ht]
	\centering
	\includegraphics[]{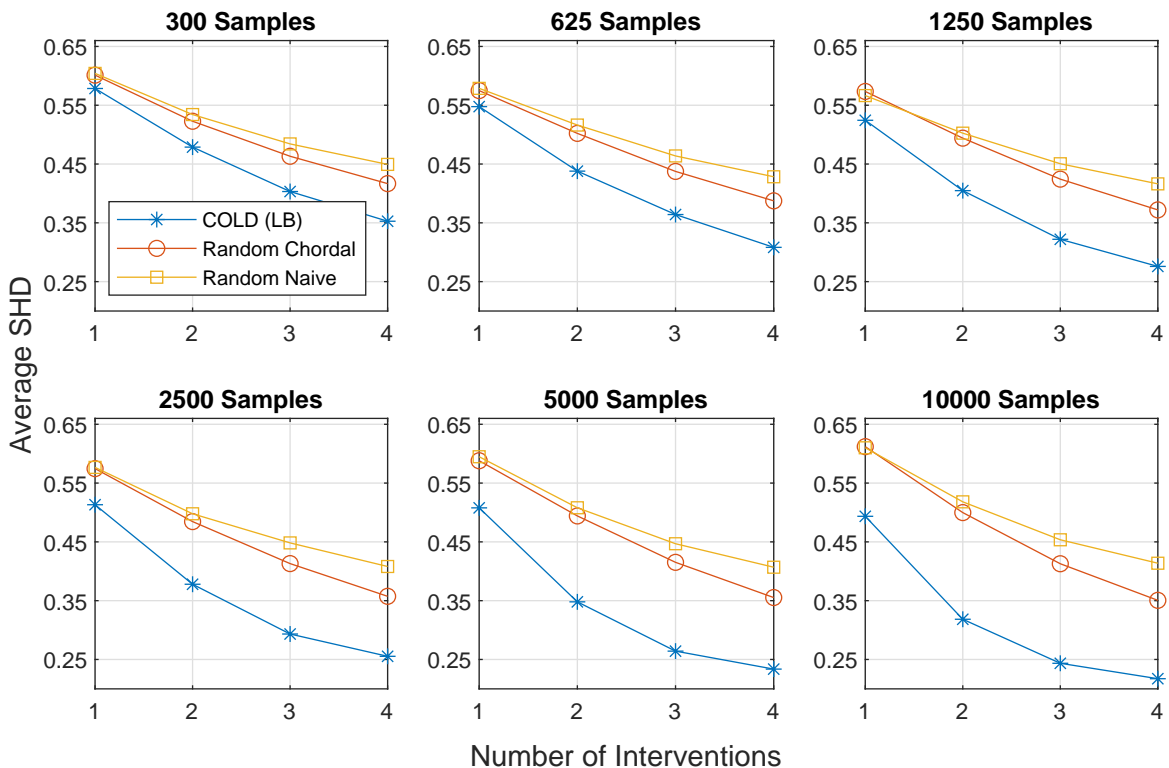}
	\caption{Average SHD values of COLD (LB), Random chordal, and Random Naive versus number of interventions for graphs with 25 nodes and 35 edges.}
	\label{Experiment:COLDSampleBased}
\end{figure}

We considered 3000 randomly generated linear Gaussian causal model for graph with 25 nodes and 35 edges. All edges coefficients were sampled uniformly from $[-1.5,-0.5]\cup[0.5,1.5]$ and error variances were uniformly sampled from $[0.01,.2]$. We consider a budget of at most four interventions. We compared our proposed COLD (LB) algorithm, with Random Naive and Random Chordal algorithms for the number of samples $300,625,1250,2500,$ $5000,10000$. Figure \ref{Experiment:COLDSampleBased} shows the result of these simulations. We use the structural hamming distance (SHD), which is defined in \cite{Brown}, for comparing the identification rate for each of these algorithms. We divided the SHD value of each graph in each step of intervention by number of edges and averaged over all instances. Note that limited number of samples may cause to converge to a wrong DAG. In this stage, Random Chordal and COLD (LB) algorithms stop intervening on nodes. Thus, their SHD will remain unchanged after a number of interventions. In order to compare these algorithms with Random Naive algorithm, which has not any stopping rule, we consider the average SHD value in the last intervention before getting a DAG as the SHD value for the following steps of interventions. Although, the number of interventions in Random Naive algorithm is more than the ones in the other two algorithms, simulation results show that our COLD (LB) Algorithm outperforms baselines even in the small number of samples.

\subsection{Practical Trick}
In this part, we add a practical trick to enhance the  acceleration of some of the applications that have been discussed in this section. In applications such as MinMax problem or obtaining the lower bound in previous part, we can stop the computation as early as we know that the selected node cannot be the desired solution. We call this ``early stopping trick". For example, suppose we have executed the lower bound algorithm on a part of nodes in the graph and obtain that the maximum lower bound among these nodes. We select another node such $v$ to compute its corresponding lower bound. If in the process of computing some $L_C$, we find that the lower bound of the node $v$ is less than what we have already obtained, we immediately terminate the process of this node and we do not consider it as our desired solution.

\begin{figure}[ht]
		\centering

		\begin{tabular}{*{2}{c}}
			\includegraphics[]{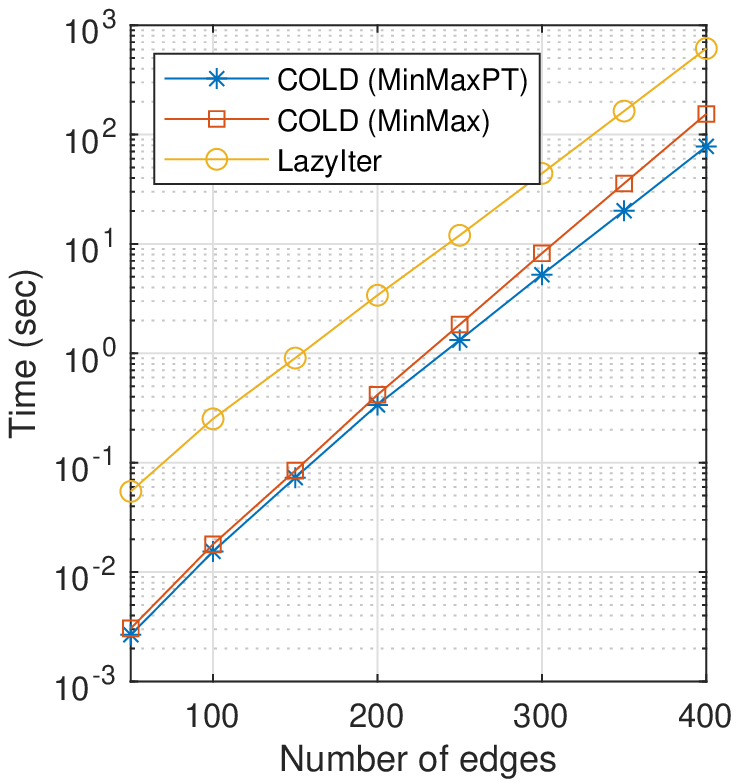}
			&
			\includegraphics[]{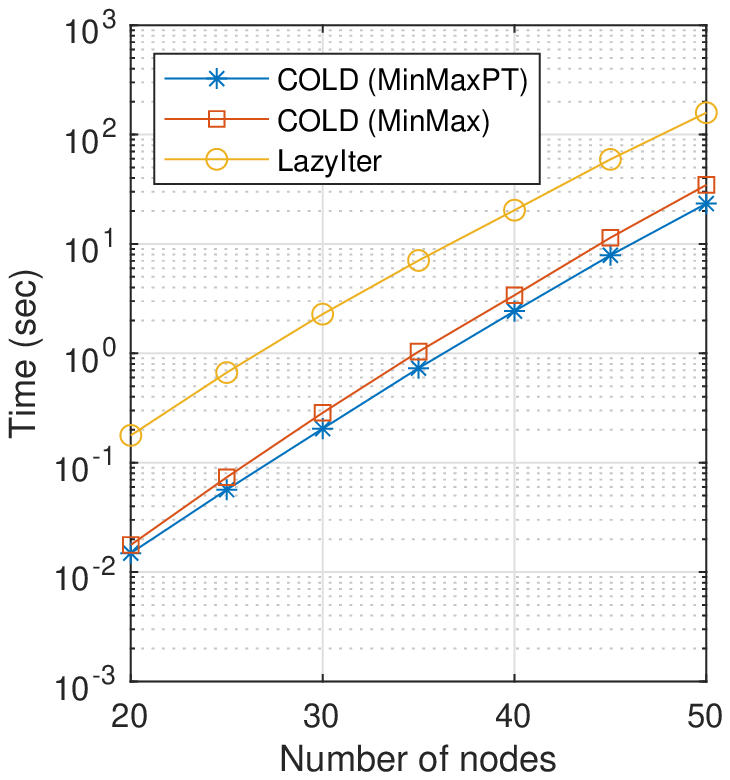}
			\\
			$\;\;\;\;\;\;\;\;\;\;\;\;$(a) & $\;\;\;\;\;\;\;\;\;\;\;\;$(b)
		\end{tabular}
	\caption{(a): Comparison between execution times of COLD (MinMaxPT), COLD (MinMax) and LazyIter \citep{Teshnizi20} versus number of edges for graphs with 40 nodes (b): Comparison between execution times of COLD (MinMaxPT), COLD (MinMax) and LazyIter \citep{Teshnizi20} versus number of nodes for graphs with edges density of 0.35.}
	\label{Experiment:PracticalTrick}
\end{figure}

In order to evaluate the gain we can get from this trick, we utilize it in selecting a target intervention by considering MinMax objective function in the active learning setting. We call the proposed algorithm using this trick ``COLD (MinMaxPT)".
Here, we compare this method with COLD (MinMax) that has been proposed previously in this section. The execution times of these two algorithms are shown in Figure 	\ref{Experiment:PracticalTrick}. In both cases, i.e., increasing number of edges  with fixed number of nodes (Figure \ref{Experiment:PracticalTrick}(a)) or increasing number of nodes with fixed edge density (Figure \ref{Experiment:PracticalTrick}(b)), using practical trick improves the computation time for solving MinMax problem.

\section{Conclusion}
\label{sec:Conclusion}

In this paper, we focused on the COL problem which may appear in many causal discovery tasks. Traditionally, Meek rules can be applied in order to solve a COL problem.
Here, we presented a mathematical representation of these rules with Meek functions. Based on this representation, we provided some desirable properties for these functions. These properties enable us to simplify or accelerate solving COL problems. For example, we utilized these properties to accelerate different causal discovery problems such as solving experiment design problems where we need to apply Meek rules multiple times. Additionally, we took advantage of these properties to obtain a tight lower bound on number of oriented edges for a target intervention in a causal graph. To the best of our knowledge, this is the first lower bound on number of oriented edges for a target intervention. Furthermore, we used these rules to check whether the prior knowledge acquired from an expert in orienting some undirected edges in causal graph is compatible with already oriented edges or not. We proposed a heuristic algorithm that solves experiment design problem in active setting based on lower bound which reduce the execution time significantly.

As an interesting future research direction, one can use Meek properties to obtain an upper bound for number of oriented edges after performing intervention. Furthermore,  designing algorithms based on these Meek properties is another possible research direction for experiment design in passive setting. 



\newpage

\appendix

\section*{Appendices}

\subsection*{Proof Sketch.} 
Lemma 12 provides Meek function properties on PCCGs. We prove this lemma by showing each property of that in Sections E, F, G, and H. Lemma 20 is related to applying Meek functions $M_{14}$ on PCCGs and it is proved by using Lemma 12. According to Lemma 12 and Lemma 16, which is about Meek function properties on ICCGs, Theorem 17 is proved. This theorem asserts that Meek functions $M_{124}$ can be decomposed to union of applying Meek functions $M_{14}$ and $M_2$. Using Theorem 17 and Lemma 20, we prove the Theorem 21 which introduces a method to obtain $\mathcal{I}$-essential graph. Based on the below sketch, we derive the proof of Lemma 22, lower bound for a maximal clique, by using a few lemmas and theorems depicted in the following figure. Finally, the proof of Theorem 24, lower bound for a case of node, is using Theorem 21 and Lemma 22.
\usetikzlibrary{decorations.markings,arrows.meta}
\tikzset
{midarrow/.style={decoration={markings,mark=at position 0.6 with
			{\arrow[xshift=2pt]{Latex[length=8pt,#1]}}},postaction={decorate}}
}
\tikzset
{midarrowplus/.style={decoration={markings,mark=at position 0.73 with
			{\arrow[xshift=2pt]{Latex[length=8pt,#1]}}},postaction={decorate}}
}
\tikzset
{midarrowminus/.style={decoration={markings,mark=at position 0.47 with
			{\arrow[xshift=2pt]{Latex[length=8pt,#1]}}},postaction={decorate}}
}

\begin{figure}[ht]
    \centering
    
\begin{tikzpicture}[
  My Style/.style = {shape=rectangle, rounded corners,
    draw=none, align=center,thick,
    top color=white, bottom color=gray!0},
  My Style1/.style = {shape=rectangle, rounded corners,
    draw=none, align=center,thick,
    top color=white, bottom color=gray!0},
 every edge/.style={draw=black,thick}]


\node [My Style1] (L7) at (-1,-4)  {Lemma 12(E, F, G, H)};
\node [My Style1] (P15) at (-7,-8)  {Proposition 18(M)};
\node [My Style1] (L13) at (-2.3,-5.5)  {Lemma 16(J)};
\node [My Style1] (T14) at (-.75,-8)  {Theorem 17(L)};
\node [My Style1] (L19) at (-3.5,-14)  {Lemma 22(Q)};
\node [My Style1] (T18) at (2,-11)  {Theorem 21(P)};
\node [My Style1] (L17) at (+2.2,-8)  {Lemma 20(O)};
\node [My Style1] (L28) at (-7.5,-11)  {Lemma 32(Q)};


\node [My Style1] (T21) at (-2,-17)  {Theorem 24(R)};

snake=expanding waves,segment angle=7

\path[snake=expanding waves,->,draw=none]
(L7) edge[out=270, in=90]     node[right]                      {} (T14)
(L13) edge[out=270, in=90]     node[right]                      {} (T14)
(L7) edge[out=270, in=90]     node[right]                      {} (L17)
(L17) edge[out=270, in=90]     node[right]                      {} (T18)
(T14) edge[out=270, in=90]     node[right]                      {} (T18)
(T18) edge[out=270, in=90]     node[right]                      {} (L19)
(P15) edge[out=270, in=90]     node[right]                      {} (L19)
(T14) edge[out=270, in=90]     node[right]                      {} (L19)
(L13) edge[out=270, in=90]     node[right]                      {} (L19)
(L7) edge[out=180, in=90]     node[right]                      {} (L19)
(L28) edge[out=270, in=90]     node[right]                      {} (L19)
(L19) edge[out=270, in=90]     node[right]                      {} (T21)
(T18) edge[out=270, in=90]     node[right]                      {} (T21);


\end{tikzpicture}
\nonumber
\end{figure}

\subsection*{A. Proof of Theorem \ref{thm:Orientationsoundness}}
Corresponding sub-graphs for Meek rule 1, 2 and 3 are as the same of their corresponding candidate sub-graphs. Thus, based on the \cite{Meek95}, applying these rules on their corresponding candidate sub-graphs are sound. We use the the following lemma to prove the soundness of applying Meek rule 4 on its corresponding candidate sub-graph.
\begin{lemma}
    \label{lem:correspondingcandidateMeekrule4}
    The result of applying Meek rule 1 and then Meek rule 4 on corresponding candidate sub-graph for Meek rule 4 is the same as of applying Meek rule 4 on its corresponding sub-graph.
\end{lemma}
\begin{proof}
As dashed lines in Table \ref{table:Meek} cannot be part of a v-structure, the edge $\v_k \line \v_l$ in corresponding sub-graph of Meek rule 4 can be either directed as $\v_k \rightarrow \v_l$ or undirected as $\v_k \line \v_l$ in its corresponding candidate sub-graph in Table \ref{table:Meek}. In the case of having  $\v_k \line \v_l$, this edge will be oriented as $\v_k \rightarrow \v_l$ by applying Meek rule 1 on sub-graph $\{\v_j \rightarrow \v_k,\v_k \line \v_l\}$. Thus, after applying Meek rule 1 on the corresponding candidate sub-graph for Meek rule 4, we get the corresponding sub-graph for Meek rule 4. Therefore, the proof is complete.
\end{proof}

Based on the Lemma \ref{lem:correspondingcandidateMeekrule4}, applying Meek rule 1 and 4 on corresponding candidate sub-graph for Meek rule 4 is as the same as the applying Meek rule 4 on corresponding sub-graph for Meek rule 4. As applying Meek rules are sound \citep{Meek95}, applying Meek rule 4 on corresponding candidate sub-graph will be sound.

\subsection*{B. Proof of Theorem \ref{thm:MPDAGcompleteness}}
Applying Meek rules 1, 2 and 3 on a MPDAG is complete \citep{Meek95}. Additionally, corresponding sub-graphs for Meek rule 1, 2 and 3 are as the same of their corresponding candidate sub-graphs. Therefore, applying Meek functions 1, 2 and 3 on a MPDAG is complete.

\subsection*{C. Proof of Theorem \ref{thm:GeneralMeekcompleteness}}
Applying Meek rules on a MPDAG, with some further oriented edges is complete \citep{Meek95}. Additionally, corresponding sub-graphs for Meek rule 1, 2 and 3 are as the same of their corresponding candidate sub-graphs. Moreover, based on the Lemma  \ref{lem:correspondingcandidateMeekrule4}, the result of applying Meek function $M_4$ on set of edges is equal or subset of the results of repeatedly applying Meek rule 1 and Meek rule 4 on that set. Therefore, applying Meek functions 1, 2, 3 and 4 on a graph is complete.

\subsection*{D. Proof of Property 1 in Lemma \ref{lem:MeekPropertiesObserved}}
According to \cite{Spirtes00}, having skeleton and the discovered v-structures, applying Meek function $M_{123}$ is enough for discovering further edges' orientations to obtain the essential graph. Moreover, with the infinite samples from observational data, there will be no undiscovered v-structure sub-graph in the identified graph. As there is no undiscovered v-structure sub-graph, no new v-structure will be discovered by identifying more edges' orientations rather than those that exist in essential graph. Hence, it suffices to apply Meek function $M_3$ once before applying other Meek functions.

\subsection*{E. Proof of Property 1 in Lemma \ref{lem:MeekPropertiesPCCG}}
As $G$ is a PCCG, no v-structure sub-graph exists in $G$. This is because all v-structures in causal DAG are discovered in the procedure of obtaining essential graph and removed. Considering the fact that there is a v-structure in candidate sub-graph for Meek rule 3, we cannot find any candidate sub-graph for Meek rule 3 in $\E$. Thus, applying Meek function $M_3$ on graph $G$ can not discover any further edges' orientations.

\subsection*{F. Proof of property 2 in Lemma \ref{lem:MeekPropertiesPCCG}}
Given a PCCG $G = (\V,\E)$ with set of nodes $\V$ and set of edges $\E$, we will show the following holds:
	\begin{align*}
	M_{i}(\E) =  \underset{\v_i \rightarrow \v_j \in \overrightarrow{\E}}{\bigsqcup} M_{i}( \{\v_i \rightarrow \v_j\} \Rcup \overline{\E} ),
	\end{align*}
where $M_i$ is one of the Meek functions $M_1$, $M_4$ or $M_{14}$. We can write $M_i$ as follows:
\begin{align*}
M_{i}(\E) = M_{i}\left ( \underset{\v_i \rightarrow \v_j \in \overrightarrow{\E}}{\bigsqcup} M_{i}( \{\v_i \rightarrow \v_j\} \Rcup \overline{\E} ) \right)=  M_{i}\left ( \underset{\v_i \rightarrow \v_j \in \overrightarrow{\E}}{\bigsqcup} \E_{ij} \right),
\end{align*}
where $\E_{ij} = M_i(\{\v_i \rightarrow \v_j\} \sqcup  \overline{\E})$. We will show that the right hand side of above equation is equal to $\E'= \bigsqcup_{\v_i \rightarrow \v_j \in \overrightarrow{\E}} \E_{ij}$. By contradiction, suppose that there are some edges which are not in $\E'$ but they will be oriented by applying Meek function $M_i$ on $\E'$. Therefore, there would be at least some edge like $\v_t \rightarrow \v_r$ which is in a candidate sub-graph, having directed edges in $\E'$. Moreover, it has been oriented by applying Meek function $M_i$ on this candidate sub-graph. Note that there should be such an edge like $\v_t \rightarrow \v_r$, otherwise $\E' = M_i(\E')$. Now, assume that the edge $\v_t \rightarrow \v_r$ has been oriented by applying Meek function $M_i$ on a candidate sub-graph for Meek rule 1, like $\{\v_m \rightarrow \v_t, \v_t \line \v_r\}$, such that we have $\v_m \rightarrow \v_t \in \E_{qw}$ for some $\v_q \rightarrow \v_w \in \left \{\v_i \rightarrow \v_j |\v_i \rightarrow \v_j \in \overrightarrow{\E} \right \}$. However, the edge $\v_t \line \v_r$ has been oriented by applying Meek function $M_1$ on $\v_q \rightarrow \v_w$, and we have $\v_t \rightarrow \v_r \in \E'$, which is a contradiction. Similar to proof of candidate sub-graph for Meek rule 1, we can prove this for candidate sub-graph for Meek rule 4, and the proof is complete.

\subsection*{G. Candidate sub-graphs for Meek rule 2 can not be genetated as a reult of applying Meek function $M_{14}$}
\begin{lemma} \label{lem:cand3M13}
	Let $G = (\V,\E)$ be a PCCG, $\V$ be the set of nodes and $\E$ be the set of edges. There exist some candidate sub-graphs for Meek rule 2 in $G$ and no new candidate sub-graph for this rule is generated after applying Meek functions $M_{14}$ on $\E$. 
\end{lemma}


\usetikzlibrary{decorations.markings,arrows.meta}
\tikzset
{midarrow/.style={decoration={markings,mark=at position 0.6 with
			{\arrow[xshift=2pt]{Latex[length=8pt,#1]}}},postaction={decorate}}
}
\tikzset
{midarrowplus/.style={decoration={markings,mark=at position 0.73 with
			{\arrow[xshift=2pt]{Latex[length=8pt,#1]}}},postaction={decorate}}
}
\tikzset
{midarrowminus/.style={decoration={markings,mark=at position 0.47 with
			{\arrow[xshift=2pt]{Latex[length=8pt,#1]}}},postaction={decorate}}
}	
\begin{figure}[ht]
	\begin{center}
		\begin{tikzpicture}[thick, scale=1]
		\node (b) at (1,-.3) {$\v_j$};
		\node (c) at (0,1.25) {$\v_k$};
		\node (a) at (2,1.25) {$\v_i$};
		\draw[fill=black] (1,0) circle (1.5pt);
		\draw[fill=black] (0,1) circle (1.5pt);
		\draw[fill=black] (2,1) circle (1.5pt);
		\draw[]				(2,1) -- (0,1);
		\draw[midarrow]	  	(2,1) -- (1,0);
		\draw[midarrow]   	(1,0) -- (0,1);
		\end{tikzpicture}
		\caption{Meek candidate sub-graph for Meek rule 2}	
		\label{Fig:MeekCandidateM2}
	\end{center}
\end{figure}
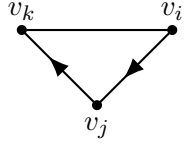
\begin{table}
	\caption{Different cases in Lemma \ref{lem:cand3M13} for orienting one edge in order to make a candidate sub-graph for Meek rule 2}
	\label{Nocyclelemma}	
	\begin{tabular}{|*{30}{c|}}  
		\hline
		\multirow{10}{*}{Case 1}
		&
		\multirow{10}{*}{\shortstack[l]{$\E' = \{\v_q \rightarrow \v_i\}\sqcup \overline{\E}$ \\ $\v_i \rightarrow \v_j \in M_1(\E')$}}
		& \multirow{5}{*}{$\v_q \in neigh(\v_k)$}  & \multirow{5}{*}{
			\begin{tikzpicture}[thick, scale=.7, baseline=0]
			\node (b) at (1,-.4) {$\v_j$};
			\node (c) at (-.4,1) {$\v_k$};
			\node (a) at (2.4,1) {$\v_i$};
			\node (a) at (1,2.4) {$\v_q$};
			\draw[fill=black] (1,0) circle (1.5pt);
			\draw[fill=black] (0,1) circle (1.5pt);
			\draw[fill=black] (2,1) circle (1.5pt);
			\draw[fill=black] (1,2) circle (1.5pt);
			\draw[midarrow=red]	(2,1) -- (0,1);
			\draw[midarrowplus]	  	(2,1) -- (1,0);
			\draw[midarrowminus=blue]	  	(2,1) -- (1,0);
			\draw[midarrowminus]   	(1,0) -- (0,1);
			\draw[midarrowminus=blue]   	(0,1) -- (1,0);
			\draw[midarrow=green!400]   	(1,2) -- (2,1);
			\draw[]			   	(1,2) -- (0,1);
			\end{tikzpicture}
		} & \multirow{5}{*}{ contradiction on $\v_j \rightarrow \v_k$}\\
		& & & & \multirow{5}{*}{ $\v_k \rightarrow \v_j \in M_4(\E')$} \\
		& & & &    \\
		& & & &  \\
		& & & &  \\\cline{3-5}
		& & \multirow{5}{*}{$\v_q \notin neigh(\v_k)$}  &
		\multirow{5}{*}{\begin{tikzpicture}[thick, scale=.7, baseline=0]
			\node (b) at (1,-.4) {$\v_j$};
			\node (c) at (-.4,1) {$\v_k$};
			\node (a) at (2.4,1) {$\v_i$};
			\node (a) at (1,2.4) {$\v_q$};
			\draw[fill=black] (1,0) circle (1.5pt);
			\draw[fill=black] (0,1) circle (1.5pt);
			\draw[fill=black] (2,1) circle (1.5pt);
			\draw[fill=black] (1,2) circle (1.5pt);
			\draw[midarrowminus=blue]	(2,1) -- (0,1);
			\draw[midarrowplus=red]	(2,1) -- (0,1);
			\draw[midarrowplus]	  	(2,1) -- (1,0);
			\draw[midarrowminus=blue]	  	(2,1) -- (1,0);
			\draw[midarrowminus]   	(1,0) -- (0,1);
			\draw[midarrow=green!400]   	(1,2) -- (2,1);
			\end{tikzpicture}}
		&
		\multirow{5}{*}{$\v_i \rightarrow \v_k \in M_1(\E')$}\\
		& & & &  \\
		& & & &  \\
		& & & &  \\
		& & & &  \\ \hline
		\multirow{8}{*}{Case 2}
		&
		\multirow{8}{*}{\shortstack[l]{$\E' = \{\v_q \rightarrow \v_j\}\sqcup \overline{\E}$ \\ $\v_j \rightarrow \v_k \in M_1(\E')$}}
		& \multirow{4}{*}{$\v_q \in neigh(\v_i)$}  & 
		\multirow{4}{*}{\begin{tikzpicture}[thick, scale=.7, baseline=0]
			\node (b) at (1,-.4) {$\v_j$};
			\node (c) at (-.4,1) {$\v_k$};
			\node (a) at (2.4,1) {$\v_i$};
			\node (a) at (3,-.4) {$\v_q$};
			\draw[fill=black] (1,0) circle (1.5pt);
			\draw[fill=black] (0,1) circle (1.5pt);
			\draw[fill=black] (2,1) circle (1.5pt);
			\draw[fill=black] (3,0) circle (1.5pt);
			\draw[midarrowplus=red]	(2,1) -- (0,1);
			\draw[midarrowminus=blue]	(2,1) -- (0,1);
			\draw[midarrowplus]	  	(2,1) -- (1,0);
			\draw[midarrowminus=green!400]	  	(3,0) -- (1,0);
			\draw[]	 		 	(3,0) -- (2,1);
			\draw[midarrowplus]   	(1,0) -- (0,1);
			\draw[midarrowminus=blue]   	(1,0) -- (0,1);
			\end{tikzpicture}}
		& \multirow{4}{*}{$\v_i \rightarrow \v_k \in M_4(\E')$}\\
		& & & &  \\
		& & & &  \\
		& & & &  \\\cline{3-5}
		& & \multirow{4}{*}{$\v_q \notin neigh(\v_i)$}  &
		\multirow{4}{*}{\begin{tikzpicture}[thick, scale=.7, baseline=0]
			\node (b) at (1,-.4) {$\v_j$};
			\node (c) at (-.4,1) {$\v_k$};
			\node (a) at (2.4,1) {$\v_i$};
			\node (a) at (3,-.4) {$\v_q$};
			\draw[fill=black] (1,0) circle (1.5pt);
			\draw[fill=black] (0,1) circle (1.5pt);
			\draw[fill=black] (2,1) circle (1.5pt);
			\draw[fill=black] (3,0) circle (1.5pt);
			\draw[midarrow=red]	(2,1) -- (0,1);
			\draw[midarrowminus]	  	(2,1) -- (1,0);
			\draw[midarrowminus=blue]	  	(1,0) -- (2,1);
			\draw[midarrow=green!400]	  	(3,0) -- (1,0);
			\draw[midarrowplus]   	(1,0) -- (0,1);
			\draw[midarrowminus=blue]   	(1,0) -- (0,1);
			\end{tikzpicture}}
		&\multirow{4}{*}{ contradiction on $\v_i \rightarrow \v_j$}\\
		& & & & \multirow{4}{*}{ $\v_j \rightarrow \v_i \in M_1(\E')$} \\
		& & & &  \\
		& & & &  \\ \hline
		\multirow{12}{*}{Case 3}
		&
		\multirow{12}{*}{\shortstack[l]{$\E' = \{\v_w \rightarrow \v_q\}\sqcup \overline{\E}$ \\ $\v_i \rightarrow \v_j \in M_4(\E')$ \\ $\v_i \rightarrow \v_j \notin M_1(\E')$}}
		& \multirow{4}{*}{\shortstack[l]{$\v_q \notin neigh(\v_k)$}}  & 
		\multirow{4}{*}{\begin{tikzpicture}[thick, scale=.7, baseline=0]
			\node (b) at (1,-.4) {$\v_j$};
			\node (c) at (0,1.4) {$\v_k$};
			\node (a) at (2,1.4) {$\v_i$};
			\node (a) at (2,-.4) {$\v_q$};
			\node (a) at (3,-.4) {$\v_w$};
			\draw[fill=black] (1,0) circle (1.5pt);
			\draw[fill=black] (0,1) circle (1.5pt);
			\draw[fill=black] (2,1) circle (1.5pt);
			\draw[fill=black] (2,0) circle (1.5pt);
			\draw[fill=black] (3,0) circle (1.5pt);
			\draw[midarrowplus=red]	(2,1) -- (0,1);
			\draw[midarrowminus=blue]	(2,1) -- (0,1);
			\draw[midarrowminus]	  	(2,1) -- (1,0);
			\draw[midarrowplus=blue]	  	(2,1) -- (1,0);
			\draw[midarrow=green!400]	  	(3,0) -- (2,0);
			\draw[]	 		 	(3,0) -- (2,1);
			\draw[midarrow]   	(1,0) -- (0,1);
			\draw[]   	(1,0) -- (0,1);
			\draw[]   	(2,0) -- (2,1);
			\draw[midarrow=blue]   	(2,0) -- (1,0);
			\end{tikzpicture}}
		& \multirow{4}{*}{$\v_i \rightarrow \v_k \in M_4(\E')$}\\
		& & & &  \\
		& & & &  \\
		& & & &  \\\cline{3-5}
		& & \multirow{4}{*}{\shortstack[l]{$\v_q \in neigh(\v_k)$ \\ $\v_w \notin neigh(\v_k)$}}  &
		\multirow{4}{*}{\begin{tikzpicture}[thick, scale=.7, baseline=0]
			\node (b) at (1,-.4) {$\v_j$};
			\node (c) at (0,1.4) {$\v_k$};
			\node (a) at (2,1.4) {$\v_i$};
			\node (a) at (2,-.4) {$\v_q$};
			\node (a) at (3,-.4) {$\v_w$};
			\draw[fill=black] (1,0) circle (1.5pt);
			\draw[fill=black] (0,1) circle (1.5pt);
			\draw[fill=black] (2,1) circle (1.5pt);
			\draw[fill=black] (2,0) circle (1.5pt);
			\draw[fill=black] (3,0) circle (1.5pt);
			\draw[midarrow=blue]   				(2,0) -- (0,1);
			\draw[midarrowplus=red]	(2,1) -- (0,1);
			\draw[midarrowminus=blue]	(2,1) -- (0,1);
			\draw[midarrowminus]	  	(2,1) -- (1,0);
			\draw[midarrowplus=blue]	  	(2,1) -- (1,0);
			\draw[midarrow=green!400]	  	(3,0) -- (2,0);
			\draw[]	 		 	(3,0) -- (2,1);
			\draw[midarrowplus]   	(1,0) -- (0,1);
			\draw[]   	(1,0) -- (0,1);
			\draw[]   	(2,0) -- (2,1);
			\draw[midarrow=blue]   	(2,0) -- (1,0);
			\end{tikzpicture}}
		&
		\multirow{4}{*}{$\v_i \rightarrow \v_k \in M_4(\E')$}\\
		& & & &  \\
		& & & &  \\
		& & & &  \\\cline{3-5}
		& & \multirow{4}{*}{\shortstack[l]{$\v_q \in neigh(\v_k)$ \\ $\v_w \in neigh(\v_k)$}}  &
		\multirow{4}{*}{\begin{tikzpicture}[thick, scale=.7, baseline=0]
			\node (b) at (1,-.4) {$\v_j$};
			\node (c) at (0,1.4) {$\v_k$};
			\node (a) at (2,1.4) {$\v_i$};
			\node (a) at (2,-.4) {$\v_q$};
			\node (a) at (3,-.4) {$\v_w$};
			\draw[fill=black] (1,0) circle (1.5pt);
			\draw[fill=black] (0,1) circle (1.5pt);
			\draw[fill=black] (2,1) circle (1.5pt);
			\draw[fill=black] (2,0) circle (1.5pt);
			\draw[fill=black] (3,0) circle (1.5pt);
			\draw[]   				(2,0) -- (0,1);
			\draw[]   				(3,0) -- (0,1);
			\draw[midarrowplus=red]	(2,1) -- (0,1);
			\draw[midarrowminus]	  	(2,1) -- (1,0);
			\draw[midarrowplus=blue]	  	(2,1) -- (1,0);
			\draw[midarrowplus=green!400]	  	(3,0) -- (2,0);
			\draw[]	 		 	(3,0) -- (2,1);
			\draw[midarrowminus]   	(1,0) -- (0,1);
			\draw[midarrowminus=blue]   	(0,1) -- (1,0);
			\draw[]   	(2,0) -- (2,1);
			\draw[midarrowplus=blue]   	(2,0) -- (1,0);
			\end{tikzpicture}}
		&\multirow{4}{*}{ contradiction on $\v_j \rightarrow \v_k$}\\
		& & & & \multirow{4}{*}{ $\v_k \rightarrow \v_j \in M_4(\E')$} \\
		& & & &  \\
		& & & &  \\ \hline
		\multirow{12}{*}{Case 4}
		&
		\multirow{12}{*}{\shortstack[l]{$\E' = \{\v_w \rightarrow \v_q\}\sqcup \overline{\E}$ \\ $\v_j \rightarrow \v_k \in M_4(\E')$ \\ $\v_j \rightarrow \v_k \notin M_1(\E')$}}
		& \multirow{4}{*}{\shortstack[l]{$\v_q \notin neigh(\v_i)$}}  & 
		\multirow{4}{*}{\begin{tikzpicture}[thick, scale=.7, baseline=0]
			\node (b) at (1,-.4) {$\v_j$};
			\node (c) at (0,1.4) {$\v_k$};
			\node (a) at (2,1.4) {$\v_i$};
			\node (a) at (-1.4,.5) {$\v_q$};
			\node (a) at (-1.4,-.5) {$\v_w$};
			\draw[fill=black] (1,0) circle (1.5pt);
			\draw[fill=black] (0,1) circle (1.5pt);
			\draw[fill=black] (2,1) circle (1.5pt);
			\draw[fill=black] (-1,-.5) circle (1.5pt);
			\draw[fill=black] (-1,.5) circle (1.5pt);
			\draw[midarrowplus=red]	(2,1) -- (0,1);
			\draw[midarrowplus=blue]	(0,1) -- (2,1);
			\draw[midarrowplus]	  	(2,1) -- (1,0);
			\draw[midarrowplus=blue]	(1,0) -- (2,1);
			\draw[]	  	(-1,.5) -- (1,0);
			\draw[midarrow=blue]	 		 	(-1,.5) -- (0,1);
			\draw[midarrowplus]   	(1,0) -- (0,1);
			\draw[midarrowminus=blue]   	(1,0) -- (0,1);
			\draw[midarrow=green!400]   	(-1,-.5) -- (-1,.5);
			\draw[]   	(-1,-.5) -- (1,0);
			\end{tikzpicture}}
		&\multirow{4}{*}{ contradiction on $\v_i \rightarrow \v_j$}\\
		& & & & \multirow{4}{*}{ $\v_j \rightarrow \v_i \in M_4(\E')$} \\
		& & & &  \\
		& & & &  \\\cline{3-5}
		& & \multirow{4}{*}{\shortstack[l]{$\v_q \in neigh(\v_i)$ \\ $\v_w \notin neigh(\v_i)$}}  & 
		\multirow{4}{*}{\begin{tikzpicture}[thick, scale=.7, baseline=0]
			\node (b) at (1,-.4) {$\v_j$};
			\node (c) at (0,1.4) {$\v_k$};
			\node (a) at (2,1.4) {$\v_i$};
			\node (a) at (-1.4,.5) {$\v_q$};
			\node (a) at (-1.4,-.5) {$\v_w$};
			\draw[fill=black] (1,0) circle (1.5pt);
			\draw[fill=black] (0,1) circle (1.5pt);
			\draw[fill=black] (2,1) circle (1.5pt);
			\draw[fill=black] (-1,-.5) circle (1.5pt);
			\draw[fill=black] (-1,.5) circle (1.5pt);
			\draw[midarrowplus=red]	(2,1) -- (0,1);
			\draw[midarrowplus]	  	(2,1) -- (1,0);
			\draw[midarrowplus=blue]	(1,0) -- (2,1);
			\draw[]	  	(-1,.5) -- (1,0);
			\draw[midarrowplus=blue]	  	(-1,.5) -- (2,1);
			\draw[midarrow=blue]	 		 	(-1,.5) -- (0,1);
			\draw[midarrowplus]   	(1,0) -- (0,1);
			\draw[midarrowminus=blue]   	(1,0) -- (0,1);
			\draw[midarrow=green!400]   	(-1,-.5) -- (-1,.5);
			\draw[]   	(-1,-.5) -- (1,0);
			\end{tikzpicture}}
		&\multirow{4}{*}{ contradiction on $\v_i \rightarrow \v_j$}\\
		& & & & \multirow{4}{*}{ $\v_j \rightarrow \v_i \in M_4(\E')$} \\
		& & & &  \\
		& & & &  \\\cline{3-5}
		& & \multirow{4}{*}{\shortstack[l]{$\v_q \in neigh(\v_i)$ \\ $\v_w \in neigh(\v_i)$}}  & 
		\multirow{4}{*}{\begin{tikzpicture}[thick, scale=.7, baseline=0]
			\node (b) at (1,-.4) {$\v_j$};
			\node (c) at (0,1.4) {$\v_k$};
			\node (a) at (2,1.4) {$\v_i$};
			\node (a) at (-1.4,.5) {$\v_q$};
			\node (a) at (-1.4,-.5) {$\v_w$};
			\draw[fill=black] (1,0) circle (1.5pt);
			\draw[fill=black] (0,1) circle (1.5pt);
			\draw[fill=black] (2,1) circle (1.5pt);
			\draw[fill=black] (-1,-.5) circle (1.5pt);
			\draw[fill=black] (-1,.5) circle (1.5pt);
			\draw[]	  	(-1,.5) -- (1,0);
			\draw[]	  	(-1,.5) -- (2,1);
			\draw[]   	(-1,-.5) -- (2,1);
			\draw[midarrowplus=red]	(2,1) -- (0,1);
			\draw[midarrowminus=blue]	(2,1) -- (0,1);
			\draw[midarrowplus]	  	(2,1) -- (1,0);
			\draw[]						(1,0) -- (2,1);
			\draw[midarrow=blue]	 		 	(-1,.5) -- (0,1);
			\draw[midarrowplus]   	(1,0) -- (0,1);
			\draw[midarrowminus=blue]   	(1,0) -- (0,1);
			\draw[midarrow=green!400]   	(-1,-.5) -- (-1,.5);
			\draw[]   	(-1,-.5) -- (1,0);
			\end{tikzpicture}}
		&
		\multirow{4}{*}{$\v_i \rightarrow \v_k \in M_4(\E')$}\\
		& & & &  \\
		& & & &  \\
		& & & &  \\ \hline
	\end{tabular}
\end{table}

We prove this lemma by contradiction. A candidate sub-graph for Meek rule 2 is depicted in Figure \ref{Fig:MeekCandidateM2}. We assume that one of the directed edges in the candidate sub-graph for Meek rule 2 is already oriented and the other one will be oriented by one of the Meek functions $M_1$ or $M_4$. All possible scenarios have been depicted in Table \ref{Nocyclelemma}. Note that in PCCGs, sub-graphs are chordal and there exist no v-structure before or after applying any Meek functions. We will see that two types of contradictions will be encountered in all of these cases. In some cases, there will be a conflict in direction of recently oriented edges with the already directed edges in the candidate sub-graph. In some other cases, we see that an undirected edge in candidate sub-graph will be oriented and there is no need to apply Meek function $M_2$. Hence, for both of these cases, there will be no candidate sub-graph for Meek rule 2. In the following, we will study each of these cases in Table \ref{Nocyclelemma}.

In depicted graphs in Table \ref{Nocyclelemma}, we use different colors to distinguish different types of oriented edges. In all of these graphs, black oriented edges are the edges in Meek function $M_2$ candidate sub-graph. The red oriented edge is the one that will be oriented as a result of applying Meek function $M_2$. The green oriented edge is supposed to be the edge that applying Meek functions $M_1$ or $M_4$ on that, results in orienting one further edge in order to construct Meek candidate sub-graph for function $M_2$. The blue oriented edges are the result of applying Meek function $M_1$, $M_4$ or $M_{14}$ on green oriented edge in each graph. Note that the blue edges are a part of edges that can be oriented based on green edges and knowing their orientations is enough to prove the properties in the lemma.

In the first case, we investigate the scenario that the edge $\v_i \rightarrow \v_j$ has been oriented as a result of applying Meek function $M_1$ on an existing edge $\v_q \rightarrow \v_i$. As $\v_i \rightarrow \v_j \in M_1(\{\v_q \rightarrow \v_i\} \sqcup \overline{\E})$, we know $\v_q \notin neigh(\v_j)$. Thus, two different cases $\v_q \notin neigh(\v_k)$ and $\v_q \in neigh(\v_k)$ can be considered. In the case of $\v_q \in neigh(\v_k)$, the edge $\v_k \line \v_j$ will be oriented in opposite direction of $\v_j \rightarrow \v_k$, which is a contradiction. In the second case, $\v_q \notin neigh(\v_k)$, the edge $\v_i \rightarrow \v_k$ will be oriented as a result of applying Meek function $M_{1}$. Hence, in both cases, no candidate sub-graph for Meek rule 2 will be created.

In the second case, we consider the scenario that the edge $\v_j \rightarrow \v_k$ has been oriented as a result of applying Meek function $M_1$ on an existing edge $\v_q \rightarrow \v_j$. As $\v_j \rightarrow \v_k \in M_1(\{\v_q \rightarrow \v_j\} \sqcup \overline{\E})$, we know $\v_q \notin neigh(\v_k)$. Two different cases can be considered: $\v_q \notin neigh(\v_i)$ and $\v_q \in neigh(\v_i)$. In the case of $\v_q \in neigh(\v_i)$, the edge $\v_i \rightarrow \v_k$ will be oriented as a result of applying Meek function $M_{4}(\{\v_q \rightarrow \v_j\} \sqcup \overline{\E})$. In the other case, $\v_q \notin neigh(\v_i)$, the edge $\v_i \line \v_j$ will be oriented in opposite direction of $\v_i \rightarrow \v_j$, which is a contradiction. Therefore, no candidate sub-graph for Meek rule 2 will be created.

In the third case, we assume that the edge $\v_i \rightarrow \v_j$ has been oriented as a result of  applying Meek function $M_4$ on a set $\{\v_w \rightarrow \v_q\} \sqcup \overline{\E}$, and we also know that $\v_i \rightarrow \v_j \notin M_1(\{\v_w \rightarrow \v_q\} \sqcup \overline{\E})$ (see Table \ref{Nocyclelemma}). As $\v_i \rightarrow \v_j \in M_4(\{\v_w \rightarrow \v_q\} \sqcup \overline{\E})$, we know $\v_w \notin neigh(\v_j)$. We partition the possible skeletons in three cases: (A) $\v_q \notin neigh(\v_k)$, (B) $\v_q \in neigh(\v_k)$, and $\v_w \notin neigh(\v_k)$ (C) $\v_q \in neigh(\v_k)$ and $\v_w \in neigh(\v_k)$. In case A, we know $\v_q \rightarrow \v_j \in M_4(\{\v_w \rightarrow \v_q\} \sqcup \overline{\E})$, and therefore, we have $\v_i \rightarrow \v_k \in M_4(\{\v_q \rightarrow \v_j\} \sqcup \overline{\E})$. In case B, we have $\v_i \rightarrow \v_k \in M_4(\{\v_w \rightarrow \v_q\} \sqcup \overline{\E})$. In case C, the edge $\v_j \line \v_k$ will be oriented in opposite direction of $\v_j \rightarrow \v_k$, which is a contradiction.

In the fourth case, we assume that the edge $\v_j \rightarrow \v_k$ has been oriented as a result of  applying $M_4$ on a set $\{\v_w \rightarrow \v_q\} \sqcup \overline{\E}$, and also we know $\v_j \rightarrow \v_k \notin M_1(\{\v_w \rightarrow \v_q\} \sqcup \overline{\E})$. As $\v_j \rightarrow \v_k \in M_4(\{\v_w \rightarrow \v_q\} \sqcup \overline{\E})$, we know $\v_w \notin neigh(\v_k)$. We partition the possible skeletons for the mentioned scenario in three cases: (A) $\v_q \notin neigh(\v_i)$, (B) $\v_q \in neigh(\v_i)$, and $\v_w \notin neigh(\v_i)$ (C) $\v_q \in neigh(\v_i)$ and $\v_w \in neigh(\v_i)$. In case A, we know $\v_q \rightarrow \v_k \in M_4(\{\v_w \rightarrow \v_q\} \sqcup \overline{\E})$, and therefore, we have $\v_j \rightarrow \v_i \in M_4(\{\v_q \rightarrow \v_k\} \sqcup \overline{\E})$, which is a contradiction.  In case B, we have $\v_j \rightarrow \v_i \in M_4(\{\v_w \rightarrow \v_q\} \sqcup \overline{\E})$, which is a contradiction. In the case C, we have $\v_i \rightarrow \v_k \in M_4(\{\v_w \rightarrow \v_q\} \sqcup \overline{\E})$. Therefore, there is no candidate sub-graph for Meek rule 2 in all of possible scenarios. Thus, the proof is complete.

\subsection*{H. Proof of Property 3 in Lemma \ref{lem:MeekPropertiesPCCG}}
We write Meek function $M_{124}(\E)$ as follows:
\begin{align*}
M_{124}(\E) = M_{124}( M_{14}( M_{2}(\E))) =  M_{124}(\E^\prime) = \E^\prime \sqcup \E_N,
\end{align*}
where $\E^\prime = M_{14}( M_{2}(\E))$. Note that $\overrightarrow{\E}\subseteq M_{14}( M_{2}(\E))$. It suffices to show $\E_N$ is an empty set. We prove this by contradiction. Suppose $\SET{S} \subseteq \E^\prime$ is a candidate sub-graph for Meek rule 1, 2 or 4 and $M_i$ is one of the Meek functions $M_1$, $M_2$ or $M_4$. Assume that there exists an edge $e_n \in \E_N$ such that $e_n \notin \E^\prime$. We have:
	\begin{align*}
	M_{124}(\SET{S} \sqcup \E^\prime\backslash\SET{S}) = M_{124}(\SET{S} \sqcup (\E^\prime\backslash\SET{S}) \sqcup \{e_n\}),
	\end{align*}
where edge $e_n$ has been oriented in the result of applying one of the Meek functions $M_1$, $M_2$ or $M_4$ on candidate sub-graph $\SET{S}$. In the case of $e_n \in M_1(\E^\prime)$ and $e_n \in M_4(\E^\prime)$, we will have $e_n \in \E^\prime$ which is a contradiction. In the case of $e_n \in M_2(\E^\prime)$, we know that $\SET{S} \nsubseteq M_2(\E)$. This is because it is not possible to have a candidate sub-graph of Meek rule 2 on the result of Meek function $M_2$. Furthermore, we know from the Lemma \ref{lem:cand3M13} that candidate sub-graph for Meek rule 2 cannot be generated as a result of applying Meek function $M_{14}$. Hence, we will have $\SET{S} \nsubseteq \E^\prime$ which is a contradiction.

\subsection*{I. Proof of Property 4 in Lemma \ref{lem:MeekPropertiesPCCG}}
We want to show that for any subset $\SET{S} = \{ \v_i \rightarrow \v_k,\v_k \rightarrow \v_j,\v_i \line \v_j\} \subseteq \SET{\E}$, we have:
\begin{align*}
M_{14}(M_2(\SET{S}) \sqcup \E\backslash \SET{S})
= M_{14}( \E ) \Rcup M_{2}(\SET{S}) .
\end{align*}
Equivalently, we want to show that:
	\begin{align}
	M_{14} \left ( \left \{ 
	\begin{array}{c} 
	\v_i \rightarrow \v_k \\ \v_k \rightarrow \v_j \\ \v_i \rightarrow \v_j 
	\end{array} 
	\right \} \sqcup \E  \right)
	&= M_{14} \left ( \left \{ \begin{array}{c} \v_i \rightarrow \v_k \\ \v_k \rightarrow \v_j \\ \v_i \line \v_j  \end{array}  \right \}  \sqcup \E   \right  ) \sqcup  \left \{ \begin{array}{c} \v_i \rightarrow \v_k \\ \v_k \rightarrow \v_j \\ \v_i \rightarrow \v_j  \end{array}  \right \}  \\
	&= M_{14} \left ( \left \{ \begin{array}{c} \v_i \rightarrow \v_k \\ \v_k \rightarrow \v_j \\ \v_i \line \v_j  \end{array}  \right \}  \sqcup \E   \right  ) \sqcup  \left \{ \v_i \rightarrow \v_j \right \}.
	\end{align}
The last equation holds because the edges in $\overrightarrow{\E}$ exist in the set of edges as a result of applying any Meek function. Furthermore, according to Property 2 in Lemma \ref{lem:MeekPropertiesPCCG}, for any $e_1 \in \overrightarrow{\E}$, we have:
\begin{align*}
M_{14}(\E) &=  \underset{e \in \overrightarrow{\E}}{\bigsqcup} M_{14}(\{e\} \Rcup \overline{\E} ) \\
&=  \left (  \underset{e\in \overrightarrow{\E}\backslash\{e_1\}}{\bigsqcup} M_{14}( \{e\} \Rcup \overline{\E} ) \right ) \sqcup M_{14}( \{e_1\} \Rcup \overline{\E} ) \\
&=M_{14}(\overline{\E} \sqcup \overrightarrow{\E}\backslash \{e_1\})\sqcup M_{14}( \{e_1\} \Rcup \overline{\E} ),
\end{align*}
where the third equality is due to Property 2 in Lemma \ref{lem:MeekPropertiesPCCG}.
Based on above equation, we can decompose the left side of (2) as follows:
	\begin{align*}
	M_{14} \left ( \left \{ 
	\begin{array}{c} 
	\v_i \rightarrow \v_k \\ \v_k \rightarrow \v_j \\ \v_i \rightarrow \v_j 
	\end{array} 
	\right \} \sqcup \E  \right)
	= 
	M_{14} \left ( \left \{ 
	\begin{array}{c} 
	\v_i \rightarrow \v_k \\ \v_k \rightarrow \v_j \\ \v_i \line \v_j 
	\end{array} 
	\right \} \sqcup \E  \right)
	\sqcup
	M_{14} \left ( \left \{ 
	\begin{array}{c} 
	\v_i \line \v_k \\ \v_k \line \v_j \\ \v_i \rightarrow \v_j 
	\end{array} 
	\right \} \sqcup \overline{\E}  \right).
	\end{align*}
It suffices to show that:
\begin{center}
	\begin{align*}
	M_{14} \left ( \left \{ 
	\begin{array}{c} 
	\v_i \line \v_k \\ \v_k \line \v_j \\ \v_i \rightarrow \v_j 
	\end{array} 
	\right \} \sqcup  \overline{\E}  \right) \backslash \{\v_i \rightarrow \v_j \}
	\subseteq
	M_{14} \left ( \left \{ 
	\begin{array}{c} 
	\v_i \rightarrow \v_k \\ \v_k \rightarrow \v_j \\ \v_i \line \v_j 
	\end{array} 
	\right \} \sqcup \E  \right)
	\end{align*}
\end{center}
In order to show this relation, we depict all possible skeletons, containing sub-graph $\{\v_i \rightarrow \v_k,\v_k \rightarrow \v_j,\v_i \rightarrow \v_j\}$, for further edges' orientations in Table \ref{Table:Property5}. Note that in PCCGs, sub-graphs are chordal and there exist no v-structure before or after applying any Meek functions. We intend to determine all additional edges that can be oriented as a result of applying Meek function $M_{14}$ on $\overline{\E} \sqcup \{\v_i \rightarrow \v_j\}$. Hence, if this function can orient any further edges, then there should be a node $\v_q \in neigh(\v_j)$, and also, an edge $\v_j \line \v_q$ in all possible skeletons. This edge will be oriented as a result of applying Meek function $M_{14}$ on $\overline{\E} \sqcup \{\v_i \rightarrow \v_j\}$. In the following, we will investigate each case in Table \ref{Table:Property5}, and we will show that the above equation holds.
\usetikzlibrary{decorations.markings,arrows.meta}
\tikzset
{midarrow/.style={decoration={markings,mark=at position 0.6 with
			{\arrow[xshift=2pt]{Latex[length=8pt,#1]}}},postaction={decorate}}
}
\tikzset
{midarrowplus/.style={decoration={markings,mark=at position 0.73 with
			{\arrow[xshift=2pt]{Latex[length=8pt,#1]}}},postaction={decorate}}
}
\tikzset
{midarrowminus/.style={decoration={markings,mark=at position 0.47 with
			{\arrow[xshift=2pt]{Latex[length=8pt,#1]}}},postaction={decorate}}
}	

\begin{table}[ht]
	\begin{center}
		\caption{Applying Meek function $M_{14}$ on candidate sub-graph for Meek rule 2}
		\label{Table:Property5}
		\begin{tabular}{|c|c|c|c|}
			\hline	
			\multicolumn{1}{|c|}{Case 1} & \multicolumn{1}{|c|}{Case 2} & \multicolumn{1}{|c|}{Case 3} &\multicolumn{1}{|c|}{Case 4} \\ 
			\hline
			\multicolumn{2}{|c|}{$\v_q \notin neigh(\v_i)$} & \multicolumn{2}{c|}{$\v_q \notin neigh(\v_k)$}\\ 
			\hline
			\multicolumn{1}{|c|}{$\v_q \notin neigh(\v_k)$} & \multicolumn{1}{c|}{$\v_q \in neigh(\v_k)$} & 
			\multicolumn{1}{|c|}{$\v_w \notin neigh(\v_k)$} & \multicolumn{1}{c|}{$\v_w \in neigh(\v_k)$}\\ 
			\hline
			
			\begin{tikzpicture}[thick, scale=.6]
			\node (a) at (1,2.35) {$\v_q$};
			\node (b) at (1,-.45) {$\v_k$};
			\node (c) at (-.35,.75) {$\v_j$};
			\node (d) at (2.35,.75) {$\v_i$};
			\draw[fill=black] (1,2) circle (1.5pt);
			\draw[fill=black] (1,0) circle (1.5pt);
			\draw[fill=black] (0,1) circle (1.5pt);
			\draw[fill=black] (2,1) circle (1.5pt);
			\draw[midarrowminus=red,midarrowplus]           (0,1) -- (1,2);
			\draw[midarrow=red]	(2,1) -- (0,1);
			\draw[midarrow]	  	(2,1) -- (1,0);
			\draw[midarrow]   	(1,0) -- (0,1);
			\end{tikzpicture}
			
			&
			\begin{tikzpicture}[thick, scale=.6]
			\node (a) at (-1,-.45) {$\v_q$};
			\node (b) at (1,-.45) {$\v_k$};
			\node (c) at (0,1.35) {$\v_j$};
			\node (d) at (2,1.35) {$\v_i$};
			\draw[fill=black] (-1,0) circle (1.5pt);
			\draw[fill=black] (1,0) circle (1.5pt);
			\draw[fill=black] (0,1) circle (1.5pt);
			\draw[fill=black] (2,1) circle (1.5pt);
			\draw[midarrowminus=red,midarrowplus]           (1,0) -- (-1,0);
			\draw[midarrowminus=red,midarrowplus]		      (0,1) -- (-1,0);
			\draw[midarrow=red]	(2,1) -- (0,1);
			\draw[midarrow]	  	(2,1) -- (1,0);
			\draw[midarrow]   	(1,0) -- (0,1);
			\end{tikzpicture}
			
			&
			\begin{tikzpicture}[thick, scale=.6]
			\node (a) at (1,2.35) {$\v_w$};
			\node (b) at (1,-.45) {$\v_k$};
			\node (c) at (-.35,.75) {$\v_j$};
			\node (d) at (2.35,.75) {$\v_i$};
			\node (d) at (-1,2.35) {$\v_q$};	
			\draw[fill=black] (1,2) circle (1.5pt);
			\draw[fill=black] (1,0) circle (1.5pt);
			\draw[fill=black] (0,1) circle (1.5pt);
			\draw[fill=black] (2,1) circle (1.5pt);
			\draw[fill=black] (-1,2) circle (1.5pt);
			\draw[midarrowminus=red,midarrowplus]           (1,2) -- (-1,2);
			\draw[midarrowminus=red,midarrowplus]           (0,1) -- (-1,2);
			\draw[midarrow]           (2,1) -- (1,2);
			\draw[midarrow]		      (0,1) -- (1,2);
			\draw[midarrow=red]	(2,1) -- (0,1);
			\draw[midarrow]	  	(2,1) -- (1,0);
			\draw[midarrow]   	(1,0) -- (0,1);
			\end{tikzpicture}
			
			&
			\begin{tikzpicture}[thick, scale=.6]
			\node (a) at (1,2.35) {$\v_w$};
			\node (b) at (1,-.45) {$\v_k$};
			\node (c) at (-.35,.75) {$\v_j$};
			\node (d) at (2.35,.75) {$\v_i$};
			\node (d) at (-1,2.35) {$\v_q$};		
			\draw[fill=black] (1,2) circle (1.5pt);
			\draw[fill=black] (1,0) circle (1.5pt);
			\draw[fill=black] (0,1) circle (1.5pt);
			\draw[fill=black] (2,1) circle (1.5pt);
			\draw[fill=black] (-1,2) circle (1.5pt);
			\draw[midarrowminus=red,midarrowplus]           (1,2) -- (-1,2);
			\draw[midarrowminus=red,midarrowplus]           (0,1) -- (-1,2);
			\draw[]           (2,1) -- (1,2);
			\draw[](0,1) -- (1,2);
			\draw[midarrowplus=red]  (2,1) -- (0,1);
			\draw[midarrow]	  	(2,1) -- (1,0);
			\draw[midarrow]   	(1,0) -- (0,1);
			\draw[] 		  	(1,2) -- (1,0);
			\end{tikzpicture}
			\\
			\hline	
			\multicolumn{4}{c}{} 
			\\
			\hline	
			\multicolumn{2}{|c|}{Case 5} & \multicolumn{2}{|c|}{Case 6} \\ 
			\hline
			\multicolumn{4}{|c|}{$\v_q \in neigh(\v_k)$}\\ 
			\hline
			\multicolumn{2}{|c|}{$\v_w \notin neigh(\v_k)$} & \multicolumn{2}{c|}{$\v_w \in neigh(\v_k)$} \\ 
			\hline
			\multicolumn{2}{|c|}{
			\begin{tikzpicture}[thick, scale=.6]
			\node (a) at (1,2.35) {$\v_w$};
			\node (b) at (0,-.45) {$\v_k$};
			\node (c) at (-.5,.65) {$\v_j$};
			\node (d) at (2.35,.75) {$\v_i$};
			\node (d) at (-3.1,1) {$\v_q$};	
			\draw[fill=black] (1,2) circle (1.5pt);
			\draw[fill=black] (0,0) circle (1.5pt);
			\draw[fill=black] (0,1) circle (1.5pt);
			\draw[fill=black] (2,1) circle (1.5pt);
			\draw[fill=black] (-2.5,1) circle (1.5pt);
			\draw[midarrowminus]           (-2.5,1) -- (1,2);
			\draw[midarrowminus=red]           (1,2) -- (-2.5,1);
			\draw[midarrowminus=red,midarrowplus]           (0,1) -- (-2.5,1);
			\draw[]           (2,1) -- (1,2);
			\draw[]		      (0,1) -- (1,2);
			\draw[midarrow=red]	(2,1) -- (0,1);
			\draw[midarrow]	  	(2,1) -- (0,0);
			\draw[midarrowplus]   	(0,0) -- (0,1);
			\draw[midarrowminus=red,midarrowplus]   	(0,0) -- (-2.5,1);

			\end{tikzpicture}
			}
			
			&
			\multicolumn{2}{|c|}{
			\begin{tikzpicture}[thick, scale=.6]
			\node (a) at (1,2.35) {$\v_w$};
			\node (b) at (0,-.45) {$\v_k$};
			\node (c) at (-.5,.65) {$\v_j$};
			\node (d) at (2.35,.75) {$\v_i$};
			\node (d) at (-3.1,1) {$\v_q$};	
			\draw[fill=black] (1,2) circle (1.5pt);
			\draw[fill=black] (0,0) circle (1.5pt);
			\draw[fill=black] (0,1) circle (1.5pt);
			\draw[fill=black] (2,1) circle (1.5pt);
			\draw[fill=black] (-2.5,1) circle (1.5pt);
			\draw[midarrowminus=red,midarrow]           (1,2) -- (-2.5,1);
			\draw[midarrowminus=red,midarrowplus]           (0,1) -- (-2.5,1);
			\draw[]           (2,1) -- (1,2);
			\draw[]		      (0,1) -- (1,2);
			\draw[midarrow=red]	(2,1) -- (0,1);
			\draw[midarrow]	  	(2,1) -- (0,0);
			\draw[midarrowplus]   	(0,0) -- (0,1);
			\draw[midarrowminus=red,midarrowplus]   	(0,0) -- (-2.5,1);
			\draw[]   	(1,2) -- (0,0);

			\end{tikzpicture}
			}
			\\
			
			\hline	
		\end{tabular}
	\end{center}
\end{table}
In case 1, the edge $\v_j \rightarrow \v_q$ has been oriented as a result of applying Meek function $M_1$ on $\overline{\E} \sqcup \{\v_i \rightarrow \v_j\}$. This edge will also be oriented as a result of applying Meek function $M_1$ on set $\overline{\E} \sqcup \{\v_k \rightarrow \v_j\}$. Hence, there is no new information in applying Meek function $M_{14}$ on set $\overline{\E} \sqcup \{\v_i \rightarrow \v_j\}$.

In the case 2, the edges $\v_j \rightarrow \v_q$ and $\v_k \rightarrow \v_q$ have been oriented as a result of applying Meek function $M_4$ on $\overline{\E} \sqcup \{\v_i \rightarrow \v_j\}$. We have $\{\v_j \rightarrow \v_q,\v_k \rightarrow \v_q\} \subseteq M_4(\overline{\E} \sqcup \{\v_i \rightarrow \v_k\})$.

In the case 3, the edges $\v_j \rightarrow \v_q$ and $\v_w \rightarrow \v_q$ have been oriented as a result of applying Meek function $M_4$ on $\overline{\E} \sqcup \{\v_i \rightarrow \v_j\}$. We have $\v_i \rightarrow \v_w \in M_4(\overline{\E} \sqcup \{\v_k \rightarrow \v_j\})$ and $\{\v_j \rightarrow \v_q,\v_w \rightarrow \v_q\} \subseteq M_4(\overline{\E} \sqcup \{\v_i \rightarrow \v_w\})$.

In the case 4, the edges $\v_j \rightarrow \v_q$ and $\v_w \rightarrow \v_q$ have been oriented as a result of applying Meek function $M_4$ on $\overline{\E} \sqcup \{\v_i \rightarrow \v_j\}$. We have $\{\v_j \rightarrow \v_q,\v_w \rightarrow \v_q\} \subseteq M_4(\E \sqcup \{\v_k \rightarrow \v_j\})$.

In the case 5, the edge $\v_w \rightarrow \v_q$ has been oriented  as a result of applying Meek function $M_4$ on $\overline{\E} \sqcup \{\v_i \rightarrow \v_j\}$. In the other hand, this edge has been oriented in the opposite direction, i.e., $\v_q \rightarrow \v_w$  as a result of applying Meek function $M_4$ on $\overline{\E} \sqcup \{\v_k \rightarrow \v_j\}$, which is a contradiction.

In the case 6, the edges $\v_w \rightarrow \v_q$, $\v_j \rightarrow \v_q$ and $\v_k \rightarrow \v_q$ have been oriented as a result of applying Meek function $M_4$ on $\overline{\E} \sqcup \{\v_i \rightarrow \v_j\}$. Herein, for the set of nodes $\SET{T} =\{\v_i,\v_k,\v_j,\v_q\}$, we have $\{\v_k \rightarrow \v_q, \v_j \rightarrow \v_q \} \subseteq M_4(\E[\SET{T}] \sqcup \{\v_i \rightarrow \v_k\})$. Additionally, for the set of nodes $\SET{T} =\{\v_i,\v_k,\v_w,\v_q\}$, we have $\{\v_w \rightarrow \v_q, \v_k \rightarrow \v_q \} \subseteq M_4(\E[\SET{T}] \sqcup \{\v_i \rightarrow \v_k\})$.

\subsection*{J. Proof of Lemma \ref{lem:MeekProperties}}
We know $\SET{S}= \{\v_i \rightarrow \v |  \v_i \in \SET{I} \} \sqcup \{ \v \rightarrow \v_o |  \v_o \in \SET{O} \}$ where $\SET{S} \subseteq \E$, $\overrightarrow{\SET{S}} = \overrightarrow{\E}$  and $neigh(\v) = \SET{I} \cup \SET{O}$. We want to show that:
\begin{align*}
M_{2}(\E) = \left ( \underset{\SET{C}_i \in \SET{C}_{M_2}(\v)}{\bigsqcup} M_{2}(\SET{C}_i) \right) \sqcup \E .	
\end{align*}
It suffices to show that orienting new edges as a result of applying Meek function $M_2$ on set of edges $\E$ does not make any new candidate sub-graph for Meek rule 2, more than those that exist in set $\E[Neigh(\v)]$ before applying Meek function $M_2$. We will show that no new candidate sub-graph for Meek rule 2 will be generated neither in $\E[Neigh(\v)]$ nor in an candidate sub-graph, having at least one edge outside $\E[Neigh(\v)]$. 

For the first case, Figure \ref{Fig:ProofOfLemma6} shows the only possible skeleton for making new candidate sub-graph for Meek rule 2 in $\E[Neigh(\v)]$. Note that all the edges incident with $\v$ have already been oriented. As a result, in order to have a new candidate sub-graph in $\E[Neigh(\v)]$ it should be a triangle with three nodes like $\{\v_s,\v_d,\v_j\}$ which are in neighbor of $\v$. In Figure \ref{Fig:ProofOfLemma6}(a), assume the red arrow edge, i.e., $\{\v_d \rightarrow \v_s\}$ is oriented as a result of applying Meek function $M_2$ on candidate sub-graph $\{\v_d \rightarrow \v,\v \rightarrow \v_s, \v_d \line \v_s\}$. Now, we show that the triangle $\{\v_s,\v_d,\v_j\}$ cannot form a new candidate sub-graph for Meek rule 2. There exist two options for the orientation of edge $\v \line \v_j$. In Figure \ref{Fig:ProofOfLemma6}(b), we orient this edge as $\v \rightarrow \v_j$. As a result of applying Meek function $M_2$  on candidate sub-graph $\{\v_d \rightarrow \v,\v \rightarrow \v_j, \v_d \line \v_j\}$, the edge $\v_d \line \v_j$ will be oriented as  $\v_d \rightarrow \v_j$. Thus, no new candidate sub-graph is made. In the case of orienting the edge $\v \line \v_j$ as  $\v_j \rightarrow \v$, shown in Figure \ref{Fig:ProofOfLemma6}(c), we cannot obtain new candidate sub-graph for Meek rule 2. Hence, we can imply that no new candidate sub-graph will be made within sub-graph $\E[Neigh(v)]$. 

For the second case, we show that it is not possible to make a new candidate sub-graph, containing at least one node not being in $Neigh(v)$. First we show that it is not possible to orient any edge which exactly one of its end-points is in $neigh(\v)$. By showing this, as a result, no edges will be oriented outside the sub-graph $\E[Neigh(v)]$. We prove this by contradiction. 
For having a candidate sub-graph of this kind, the only possible scenario is to have a structure like $\{\v_t \rightarrow \v_m, \v_m \rightarrow \v_n, \v_t \line \v_n\}$, such that we have $\{\v_m,\v_n\} \subseteq neigh(v)$ and $\v_t \in \V \backslash Neigh(v)$.  Thus, we have $\{\v_t \line \v_m, \v_t \line \v_n\} \subseteq \overline{\E} \backslash  \overline{\E}[Neigh(v)]$ and we intend to orient the only undirected edge in this candidate sub-graph by applying Meek function $M_2$. Now, we study whether such directed edge $\v_t \rightarrow \v_m$ exists in our scenario. To orient $\v_t \rightarrow \v_m$, we need another candidate sub-graph of Meek rule 2, having two edges outside $\E[Neigh(\v)]$ (note that one of them is $\v_t \line \v_m$). The other edge should have already oriented by another candidate sub-graph for Meek rule 2 with the same property that it has two edges outside  $\E[Neigh(\v)]$. However, this continues relying on some other edges outside $\E[Neigh(\v)]$, while the number of edges of this type is finite, and it gives a contradiction.


\begin{figure}[ht]
	\begin{center}
		\begin{tabular}{ccc}
				\begin{tikzpicture}[thick, scale=.8]
				\node (a) at (-.45,0) {$\v_s$};
				\node (b) at (1.5,1.4) {$\v_d$};
				\node (c) at (1.5,-1.4) {$v$};
				\node (d) at (3.4,0) {$v_j$};
				\draw[fill=black] (0,0) circle (1.5pt);
				\draw[fill=black] (1.5,-1) circle (1.5pt);
				\draw[fill=black] (1.5,1) circle (1.5pt);
				\draw[fill=black] (3,0) circle (1.5pt);
				\draw[midarrowminus]	  (1.5,1) -- (1.5,-1);

				\draw[midarrow=red]   (1.5,1) -- (0,0);
				\draw[]   (0,0) -- (3,0);
				\draw[]   (3,0) -- (1.5,1);
				\draw[midarrow]	  (1.5,-1) -- (0,0);
				\draw[]   (1.5,-1) -- (3,0);
				\end{tikzpicture}
				&
				\begin{tikzpicture}[thick, scale=.8]
				\node (a) at (-.45,0) {$\v_s$};
				\node (b) at (1.5,1.4) {$\v_d$};
				\node (c) at (1.5,-1.4) {$v$};
				\node (d) at (3.4,0) {$v_j$};
				\draw[fill=black] (0,0) circle (1.5pt);
				\draw[fill=black] (1.5,-1) circle (1.5pt);
				\draw[fill=black] (1.5,1) circle (1.5pt);
				\draw[fill=black] (3,0) circle (1.5pt);
				\draw[midarrowminus]	  (1.5,1) -- (1.5,-1);

				\draw[midarrow=red]   (1.5,1) -- (0,0);
				\draw[]   (0,0) -- (3,0);
				\draw[midarrow=blue]   (1.5,1) --(3,0);
				\draw[midarrow]	  (1.5,-1) -- (0,0);
				\draw[midarrow]   (1.5,-1) -- (3,0);
				\end{tikzpicture}
				&
				\begin{tikzpicture}[thick, scale=.8]
				\node (a) at (-.45,0) {$\v_s$};
				\node (b) at (1.5,1.4) {$\v_d$};
				\node (c) at (1.5,-1.4) {$v$};
				\node (d) at (3.4,0) {$v_j$};
				\draw[fill=black] (0,0) circle (1.5pt);
				\draw[fill=black] (1.5,-1) circle (1.5pt);
				\draw[fill=black] (1.5,1) circle (1.5pt);
				\draw[fill=black] (3,0) circle (1.5pt);
				\draw[midarrowminus]	  (1.5,1) -- (1.5,-1);

				\draw[midarrow=red]   (1.5,1) -- (0,0);
				\draw[midarrowplus=blue] (3,0)  -- (0,0) ;
				\draw[]   (3,0) -- (1.5,1);
				\draw[midarrow]	  (1.5,-1) -- (0,0);
				\draw[midarrow]   (3,0) -- (1.5,-1);
				\end{tikzpicture}
				\\
				(a) & (b) & (c)
		\end{tabular}
		\caption{Possible orientations for making new candidate sub-graph for Meek rule 2}
		\label{Fig:ProofOfLemma6}
	\end{center}
\end{figure}
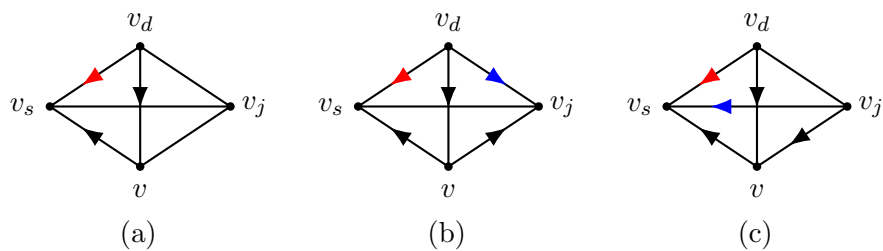

\subsection*{K. Proof of Property 2 in Lemma \ref{lem:MeekPropertiesObserved}}
 The proof of Property 2 is exactly the same as the proof of Property 2 in Lemma \ref{lem:MeekPropertiesPCCG} as adding some forbidden edges does not affect the proof. Note that we do not assume that the graph is chordal or there is no v-structures for proving Property 2 in Lemma \ref{lem:MeekPropertiesPCCG}.

\subsection*{L. Proof of Theorem \ref{thm:decomposition}}
From Property 3 of Lemma \ref{lem:MeekPropertiesPCCG}, we have: $$M_{124}(\E) = M_{14}(M_{2}(\E))$$
Substituting $M_{2}(\E)$ from Lemma \ref{lem:MeekProperties} in the above equation, we have:
$$M_{124}(\E) = M_{14} \left (\underset{\SET{C}_i \in \SET{C}_{M_2}(\v)}{\bigsqcup} M_{2}(\SET{C}_i) \sqcup \E \right ),$$
where the $\SET{C}_{M_2}(\v)$ is set of all candidate sub-graphs for Meek rule 2 in $\E[Neigh(\v)]$. 
From Property 4, we can infer that:
$$M_{14} \left (\underset{\SET{C}_i \in \SET{C}_{M_2}(\v)}{\bigsqcup} M_{2}(\SET{C}_i) \sqcup \E \right ) = M_{14}(\E) \sqcup \left ( \underset{\SET{C}_i \in \SET{C}_{M_2}(\v)}{\bigsqcup} M_{2}(\SET{C}_i) \right ).$$
Thus, we can write: $$M_{124}(\E) = M_{14}(\E)  \Rcup \left ( \underset{\SET{C}_i \in \SET{C}_{M_2}(\v)}{\bigsqcup} M_{2}(\SET{C}_i) \right ).$$
Finally, plugging $M_{14}(\E)$ from Property 2 of Lemma \ref{lem:MeekPropertiesPCCG} into the above equation, we have:
\begin{align*}
M_{124}(\E) = \left ( \underset{e \in \overrightarrow{\E}}{\bigsqcup} M_{14} \left( \{e\} \Rcup \overline{\E} \right ) \right ) \Rcup \left ( \underset{\SET{C}_i \in \SET{C}_{M_2}(\v)}{\bigsqcup} M_{2}(\SET{C}_i) \right ).
\end{align*}
Therefore, the proof is complete.

\subsection*{M. Proof of Proposition \ref{Prop:DP}}
Consider a PCCG $G=(\V,\E)$ with set of nodes $\V$ and set of edges $\E$. Suppose that  $\v_s \rightarrow \v_d$ is the only oriented edge in the set $\E$, i.e., $\overrightarrow{\E} = \{\v_s \rightarrow \v_d\}$. By defining $\SET{T} \triangleq Neigh(\v_d) $, we have:
\begin{align*}
M_{14} \left ( \{\v_s \rightarrow \v_d\} \sqcup \overline{\E}
\right )
= M_{14} \left ( 
\{\v_s \rightarrow \v_d\} \sqcup \E^\prime \sqcup \overline{\E}
\right )
= \E_{M_{14}},
\end{align*}
where $\E^\prime = M_{14}(\{ \v_s \rightarrow \v_d \} \sqcup \overline{\E}[\SET{T}] ) \backslash \{ \v_s \rightarrow \v_d\} $. From Property 2 of Lemma \ref{lem:MeekPropertiesPCCG}, we can write:
\begin{align*}
M_{14} \left ( 
\{\v_s \rightarrow \v_d\} \sqcup \E^\prime \sqcup \overline{\E}
\right ) = 
M_{14} \left (\E^\prime \sqcup \overline{\E} \right )
\sqcup
M_{14} \left (\{\v_s \rightarrow \v_d\} \sqcup \overline{\E} \right ).
\end{align*}
Furthermore, we have:
\begin{align*}
M_{14} \left ( 
\{\v_s \rightarrow \v_d\} \sqcup \overline{\E}
\right ) \backslash \{\v_s \rightarrow \v_d\}
&= \E_{M_{14}}  \backslash \{\v_s \rightarrow \v_d\}
\\ M_{14} \left (\E^\prime \sqcup \overline{\E} \right ) &= \E_{M_{14}}  \backslash \{\v_s \rightarrow \v_d\}.
\end{align*}
Thus, we have:
\begin{align*}
M_{14} \left ( \{\v_s \rightarrow \v_d\} \sqcup \overline{\E}
\right ) = 
M_{14} \left (\E^\prime \sqcup \overline{\E} \right )
\sqcup
\{\v_s \rightarrow \v_d\}.
\end{align*}
Finally, we obtain from Property 2 in Lemma \ref{lem:MeekPropertiesPCCG}:
\begin{align*}
DP [\v_s \rightarrow \v_d]
= \left ( \underset{\v_l \rightarrow \v_k  \in \overrightarrow{\E^\prime} }{\bigsqcup} DP [\v_l \rightarrow \v_k ] 
\right )
\bigsqcup
\left \{ 
\v_s \rightarrow \v_d
\right \}.
\end{align*}

\subsection*{N. Proof of Theorem  \ref{the:TimeComplexityDP}}
We know that calling function in Algorithm \ref{Alg:DPONE}, may cause multiple another calls of this function in a recursive manner. Here, we prove that calling this function with directed edge $\v_s \rightarrow \v_d$ as input cannot occurred inside in calling another function with the same input as $\v_s \rightarrow \v_d$. We prove this by using the following lemma:

\begin{lemma}
	\label{lem:PathFromSource} Given a PCCG $G=(\V,\E)$ with set of nodes $\V$ and set of edges $\E$. For any $v_s\in \SET{V}$ and for any edge $v_r \rightarrow v_k \in M_{14}(\overline{\E} \sqcup \{\v_s \rightarrow \v_d\})$, there exists a directed path from $\v_s$ to $v_r$. 
\end{lemma} 
\begin{proof}
    We prove this statement by an induction on the number of steps taken by Meek function $M_{14}$. In the base case, if we apply Meek rule 1, then we orient some edge $\v_d \rightarrow \v_j$. Thus, there is the directed path $\v_s \rightarrow \v_d \rightarrow \v_j$ from $\v_s$ to $\v_j$. Moreover, if we apply Meek rule 4 in the first step, two edges like $\v_d \rightarrow \v_j$ and $\v_k \rightarrow \v_j$ will be oriented (see Table \ref{table:Meek}) and there is the directed path $\v_s \rightarrow \v_d \rightarrow \v_j$ from $\v_s$ to $\v_j$. Now, for the induction step, assume that induction hypothesis holds up to step r, i.e., for any oriented edge like $\v_i \rightarrow \v_r$ from Meek function $M_{1}$, there is a directed path from $\v_s$ to $\v_r$. Suppose that in step r+1, a new edge like $\v_r \rightarrow \v_k$ is oriented by applying Meek rule 1 on some edge $\v_i \rightarrow \v_r$. Since we know that there is a directed path from $\v_s$ to $\v_r$, the same holds from $\v_s$ to $\v_k$. In the case that new edge is oriented by applying Meek function $M_4$ on some previous oriented edge like $\v_i \rightarrow \v_r$, it can be easily seen that (see Table \ref{table:Meek}), there would be a directed edge from $\v_r$ to the head vertex of the newly oriented edge. Based on the induction hypothesis we can infer that there is a directed path from $\v_s$ to head vertex of the new oriented edge and the proof is complete.
\end{proof}
Based on Lemma \ref{lem:PathFromSource}, we will get a directed cycle if such a mentioned recursive call is happened. In this scenario, we will have a directed path from node $\v_s$ to itself. As the underlying graph is a causal DAG, that cannot be happened. Thus, we have not multiple function calls with the same input. Additionally, as we have stored the resulted oriented edges after applying Meek rules on that edge, we did not call that edge again. Therefore, we call this function for each edge once. For filling all entries of DP table, we must call this function $2|\E|$ times. Moreover, we have a nested for-loop in each call which takes $\mathcal{O}(\Delta^2)$ operations. 
Hence, in total, the computational complexity would be in the order of $\mathcal{O}(|\E|\Delta^2)$.

\subsection*{O. Proof of Lemma \ref{Lem:InterventionM14}}
We know $\SET{S}= \{\v_i \rightarrow \v |  \v_i \in \SET{I} \} \sqcup \{ \v \rightarrow \v_o |  \v_o \in \SET{O} \}\subseteq \E$. From Property 2 of Lemma \ref{lem:MeekPropertiesPCCG}, we have:
	\begin{align*} 
	M_{14}(\E) = 
	\left(\underset{\v_i \in \SET{I} }{\bigsqcup} M_{14} \left ( \left\{ \v_i \rightarrow \v \right \} \sqcup \overline{\E} \right )\right) \bigsqcup
	\left(\underset{\v_o \in \SET{O} }{\bigsqcup} M_{14} \left ( \left\{ \v \rightarrow \v_o \right \} \sqcup \overline{\E} \right )\right).
	\end{align*}
We define set of nodes $\SET{O}_c$ as $\{\v_o |  \v_o \in \SET{O}, \SET{I} \subseteq neigh(\v_o) \}$. Hence, we rewrite the above equation as the following:
	\begin{align*}
	M_{14}(\E) = 
	\left(\underset{\v_i \in \SET{I} }{\bigsqcup} M_{14} \left ( \left\{ \v_i \rightarrow \v \right \} \sqcup \overline{\E} \right )\right) &\bigsqcup
	\left(\underset{ \v_o \in \SET{O}_c }{\bigsqcup} M_{14} \left ( \left\{ \v \rightarrow \v_o \right \} \sqcup \overline{\E} \right )\right)\\
	&\bigsqcup
	\left(\underset{\v_o \in \SET{O}\backslash\SET{O}_c }{\bigsqcup} M_{14} \left ( \left\{ \v \rightarrow \v_o \right \} \sqcup \overline{\E} \right )\right).
	\end{align*}
For each node $\v_o \in \SET{O}\backslash\SET{O}_c$, there exists a node $\v_k \in \SET{I}$  such that $\v_k \notin neigh(\v_o)$, thus we have $\v \rightarrow \v_o \in M_{14}( \{ \v_k \rightarrow \v \} \sqcup \overline{\E})$. Therefore, we have the following relation:
	\begin{align*}
	M_{14} \left ( \left\{ \v \rightarrow \v_o \right \} \sqcup \overline{\E} \right ) \subseteq M_{14} \left ( \left\{ \v_k \rightarrow \v \right \} \sqcup \overline{\E} \right ) \subseteq 	\underset{ \v_i \in \SET{I} }{\bigsqcup} M_{14} \left ( \left\{ \v_i \rightarrow \v \right \} \sqcup \overline{\E} \right ).
	\end{align*}
Thus, we have:
	\begin{align*}
	\underset{\v_o  \in \SET{O}\backslash\SET{O}_c }{\bigsqcup} M_{14} \left ( \left\{ \v \rightarrow \v_o \right \} \sqcup \overline{\E} \right ) 
	\subseteq 
	\underset{\v_i \in \SET{I} }{\bigsqcup} M_{14} \left ( \left\{ \v_i \rightarrow \v \right \} \sqcup \overline{\E} \right).
	\end{align*}
Hence, we can write:
	\begin{align*}
	M_{14}(\E) &= 
	\left(\underset{\v_i \in \SET{I} }{\bigsqcup} M_{14} \left ( \left\{ \v_i \rightarrow \v \right \} \sqcup \overline{\E} \right )\right) \bigsqcup
	\left(\underset{\v_o \in \SET{O}_c }{\bigsqcup} M_{14} \left ( \left\{ \v \rightarrow \v_o \right \} \sqcup \overline{\E} \right )\right) \\
	&= \underset{\v_l \rightarrow \v_k  \in \SET{S^\prime} }{\bigsqcup} M_{14} \left ( \left\{ \v_l \rightarrow \v_k \right \} \sqcup \overline{\E} \right ) \\&= \underset{\v_l \rightarrow \v_k \in \SET{S^\prime} }{\bigsqcup} DP [\v_l \rightarrow \v_k],
	\end{align*}
where $\SET{S^\prime} = \{ \v_i \rightarrow \v |  \v_i \in \SET{I} \} \sqcup \{ \v \rightarrow \v_o |  \v_o \in \SET{O}, \SET{I} \subseteq neigh(\v_o) \}$.

\subsection*{P. Proof of Theorem \ref{Thm:InterventionM124}}
Having $\SET{S}= \{\v_i \rightarrow \v |  \v_i \in \SET{I} \} \sqcup \{ \v \rightarrow \v_o |  \v_o \in \SET{O} \}$ and $\overrightarrow{\SET{S}} = \overrightarrow{\E}$, from Theorem \ref{thm:decomposition}, we have:
	\begin{align*}
	M_{124}(\E) = \left ( \underset{e \in \overrightarrow{\E}}{\bigsqcup} M_{14}( \{e\} \Rcup \overline{\E} ) \right ) \Rcup \left ( \underset{\SET{C}_i \in \SET{C}_{M_2}(\v)}{\bigsqcup} M_{2}(\SET{C}_i) \right ).
	\end{align*}
	Substituting the first term with its equivalent term in Lemma \ref{Lem:InterventionM14}, we have:
	\begin{align*}
	&M_{124} ( \E ) = \left ( \underset{\v_l \rightarrow \v_k \in \SET{S^\prime} }{\bigsqcup} DP [\v_l \rightarrow \v_k] \right ) \bigsqcup  \left ( \underset{\SET{C}_i \in \SET{C}_{M_2}(\v)}{\bigsqcup} M_{2}(\SET{C}_i) \right ),
	\end{align*}
where $\SET{S^\prime} = \{ \v_i \rightarrow \v |  \v_i \in \SET{I} \} \sqcup  \{ \v \rightarrow \v_o |  \v_o \in \SET{O}_c\}$ and $\SET{O}_c = \{\v_o |  \v_o \in \SET{O}, \SET{I} \subseteq neigh(\v_o) \}$. By defining $\SET{I}_{neigh(\v_o)}= \{\v_i|\v_i \in \SET{I} \cap neigh(\v_o)\}$, we have the following equation:
	\begin{align*}
	&\underset{\SET{C}_i \in \SET{C}_{M_2}(\v)}{\bigsqcup} M_{2}(\SET{C}_i) =  \left \{\v_i \rightarrow \v_o |  \v_i \in \SET{I}_{neigh(\v_o)},  \v_o \in \SET{O}_c \right \} \bigsqcup  \left \{\v_i \rightarrow \v_o |  \v_i \in \SET{I}_{neigh(\v_o)},  \v_o \in \SET{O}\backslash \SET{O}_c \right \}.
	\end{align*}
From the definition of $\SET{O}_c$, we know that for each $\v_o \in \SET{O}_c$, we have $\SET{I} \subseteq neigh(\v_o)$. Hence, for each $\v_o \in \SET{O}_c$, we have $\SET{I}_{neigh(\v_o)} = \SET{I}$. Thus, for the first term in the above equation, we have:
$$\left \{\v_i \rightarrow \v_o |  \v_i \in \SET{I}_{neigh(\v_o)},  \v_o \in \SET{O}_c \right \} = \left \{\v_i \rightarrow \v_o |  \v_i \in \SET{I},  \v_o \in \SET{O}_c \right \}.$$
Therefore, it suffices to show that:$$ \left \{\v_i \rightarrow \v_o |  \v_i \in \SET{I}_{neigh(\v_o)},  \v_o \in \SET{O}\backslash \SET{O}_c \right \} \subseteq \underset{\v_l \rightarrow \v_k \in \SET{S^\prime} }{\bigsqcup} DP [\v_l \rightarrow \v_k] $$
Having $\v_{i} \in \SET{I}_{neigh(\v_o)}$ and $\v_o \in \SET{O}\backslash \SET{O}_c$, for each edge $\v_i \line \v_o \in \overline{\E}$, there exists a node $\v_{t} \in \SET{I} \backslash \SET{I}_{neigh(\v_o)}$, and a sub-graph $\E[\{\v_i,\v_o,\v_t,\v\}] = \{\v_t \rightarrow \v,\v \rightarrow \v_o,\v_i \rightarrow \v, \v_t \line \v_i, \v_i \line \v_o\}$ in the graph. Thus, we will have:
$$ \v_i \rightarrow \v_o \in 
DP [\v_t \rightarrow \v].$$
Furthermore, we have $\{\v_t \rightarrow \v\} \subset \SET{S^\prime}$. Thus, we can write:
$$ \{\v_i \rightarrow \v_o\} \subseteq
DP [\v_t \rightarrow \v] \subseteq 
\underset{\v_l \rightarrow \v_k \in \SET{S^\prime} }{\bigsqcup} DP [\v_l \rightarrow \v_k].$$
Hence, we can conclude that:
\begin{align*}
&M_{124} (\E ) = \left \{ \underset{\v_l \rightarrow \v_k \in \SET{S^\prime} }{\bigsqcup} DP [\v_l \rightarrow \v_k ] \right \} \bigsqcup  \left (\v_i \rightarrow \v_o |  \v_i \in I,  \v_o \in \SET{O}_c \right )
\end{align*}
where $\SET{S^\prime} =\{ \v_i \rightarrow \v |  \v_i \in \SET{I} \} \sqcup \{ \v \rightarrow \v_o |  \v_o \in \SET{O}, \SET{I} \subseteq neigh(\v_o) \}$.

\subsection*{Q. Proof of Lemma \ref{lem:LBMC}}
We denote the obtained ICCGs with $G=(\V,\E)$. We know that $\overrightarrow{\E} = \SET{S}= \{\v_i \rightarrow \v |  \v_i \in \SET{I} \} \sqcup \{ \v \rightarrow \v_o |  \v_o \in \SET{O} \}$ in which $\SET{I} \subseteq \SET{C}_k$. We divide the problem of calculating the lower bound into two cases. In the first case, all the edges are from nodes in maximal clique toward the intervened node, $\v$, i.e., $\SET{I} = \SET{C}_k$. In the second case, we consider that there is an edge from a node in maximal clique toward node $\v$ and also, there is an edge from intervened node $\v$ toward a node in maximal clique, i.e., $\SET{I} \subset \SET{C}_k$.
\newline
We use Theorem \ref{Thm:InterventionM124} for calculating number of oriented edges in the first case. We have $\SET{I} = \SET{C}_k$. Thus, we know $\SET{O}_c= \emptyset$ and  $\SET{S} =\{ \v_i \rightarrow \v |  \v_i \in \SET{C}_k \}$. Hence, in this case, number of oriented edges after applying Meek function $M_{124}$, $|M_{124}(\E)|$, is equal to the following:
\begin{align*}
&|M_{124}(\E)| =L_{I} =  \left |\underset{\v_i \in \SET{I}}{\bigsqcup} DP [\v_i \rightarrow \v] \right|.
\end{align*}
Now, we will investigate the second case. For sake of simplicity, first, we assume that the set of nodes that have edges from nodes in  maximal clique $\SET{C}_k$ to intervened node $\v$ is known. Then, we calculate the minimum number of oriented edges for any possible orientation of edges between maximal clique $\SET{C}_k$ and intervened node $\v$. After intervention, we have:
\begin{align*}
M_{124}(\E) 
=& M_{124} \left ( \left ( \underset{\v_i \in \SET{I}}{\bigsqcup}\{\v_i \rightarrow \v\} \right )\bigsqcup \left (  \underset{\v_o \in neigh(\v)\backslash\SET{I}}{\bigsqcup}\{\v \rightarrow \v_o\} \sqcup \overline{\E}\right ) \right ) \\
\end{align*}
Using Theorem \ref{thm:decomposition}, we decompose Meek function $M_{124}$ to Meek functions $M_2$ and $M_{14}$:
\begin{align*}
M_{124}(\E) = &M_{14} \left (  \left ( \underset{\v_i \in \SET{I}}{\bigsqcup}\{\v_i \rightarrow \v\} \right )\bigsqcup \left (  \underset{\v_o \in neigh(\v)\backslash\SET{I}}{\bigsqcup}\{\v \rightarrow \v_o\} \sqcup \overline{\E}\right ) \right ) \bigsqcup \\
&M_{2} \left ( \left ( \underset{\v_i \in \SET{I}}{\bigsqcup}\{\v_i \rightarrow \v\} \right )\bigsqcup \left (  \underset{\v_o \in neigh(\v)\backslash\SET{I}}{\bigsqcup}\{\v \rightarrow \v_o\} \sqcup \overline{\E}\right ) \right ). \\
\end{align*}
Based on Property 2 in Lemma \ref{lem:MeekPropertiesPCCG}, we extract mixed edge union from inside of Meek function $M_{14}$:
\begin{align*}
M_{124}(\E) = & \left ( \underset{\v_i \in \SET{I}}{\bigsqcup}M_{14}(\{\v_i \rightarrow \v\} \sqcup \overline{\E})\right)\bigsqcup \left( \underset{\v_o \in neigh(\v)\backslash\SET{I}}{\bigsqcup}M_{14}( \{\v \rightarrow \v_o\} \sqcup \overline{\E}) \right ) \bigsqcup \\ 
M_{2} &\left ( \left ( \underset{\v_i \in \SET{I}}{\bigsqcup}\{\v_i \rightarrow \v\} \right )\bigsqcup \left (  \underset{\v_o \in neigh(\v)\backslash\SET{I}}{\bigsqcup}\{\v \rightarrow \v_o\} \sqcup \overline{\E}\right ) \right ). \\
\end{align*}
We know that $\SET{I} \subseteq \SET{C}_k$, therefore we have:
\begin{align*}
 neigh(\v)\backslash\SET{I} = neigh(\v)\backslash\SET{C}_k \cup \SET{C}_k\backslash\SET{I}.
\end{align*}
 Thus, we can write:
\begin{align*}
 \underset{\v_o \in neigh(\v)\backslash\SET{I}}{\bigsqcup}M_{14}(\{\v \rightarrow \v_o\} \sqcup \overline{\E}) = \left ( \underset{\v_o \in neigh(\v)\backslash\SET{C}_k}{\bigsqcup}M_{14}(\{\v \rightarrow \v_o\} \sqcup \overline{\E}) \right ) \bigsqcup \left (\underset{\v_o \in \SET{C}_k\backslash\SET{I}}{\bigsqcup} M_{14}( \{\v \rightarrow \v_o\} \sqcup \overline{\E})\right ).
\end{align*}
Substituting $M_{14}$ function with the DP value in Proposition \ref{Prop:DP}, we will have:
\begin{align*}
 \underset{\v_o \in neigh(\v)\backslash\SET{I}}{\bigsqcup}M_{14}(\{\v \rightarrow \v_o\} \sqcup \overline{\E})  = \left (\underset{\v_o \in neigh(\v)\backslash\SET{C}_k}{\bigsqcup} DP [\v \rightarrow \v_o ] \right)\bigsqcup
\left (\underset{\v_o \in \SET{C}_k\backslash\SET{I}}{\bigsqcup} DP [\v \rightarrow \v_o ]\right).
\end{align*}
In the other hand, we have:
\begin{align*}
\underset{\v_i \in \SET{I}}{\bigsqcup}M_{14}(\{\v_i \rightarrow \v\} \sqcup \overline{\E}) = \underset{\v_i \in \SET{I}}{\bigsqcup} DP [\v_i \rightarrow \v ].
\end{align*}
Therefore, we can write:
 \begin{align*}
 M_{124}(\E) =&\left(\underset{\v_o \in neigh(\v)\backslash\SET{C}_k}{\bigsqcup} DP [\v \rightarrow \v_o ] \right) \bigsqcup \left(\underset{\v_o \in \SET{C}_k\backslash\SET{I}}{\bigsqcup} DP \left [ \{ \v \rightarrow \v_o  \}\right ] \right) \bigsqcup\\ 
 &\left(\underset{\v_i \in \SET{I}}{\bigsqcup} DP [\v_i \rightarrow \v ] \right) \bigsqcup M_{2} \left ( \left ( \underset{\v_i \in \SET{I}}{\bigsqcup}\{\v_i \rightarrow \v\} \right )\bigsqcup \left (  \underset{\v_o \in neigh(\v)\backslash\SET{I}}{\bigsqcup}\{\v \rightarrow \v_o\} \sqcup \overline{\E}\right ) \right ).
 \end{align*}
According to Lemma \ref{lem:MeekProperties}, we can write:
\begin{align*}
&M_{2} \left ( \left ( \underset{\v_i \in \SET{I}}{\bigsqcup}\{\v_i \rightarrow \v\} \right )\bigsqcup \left (  \underset{\v_o \in neigh(\v)\backslash\SET{I}}{\bigsqcup}\{\v \rightarrow \v_o\} \sqcup \overline{\E}\right ) \right ) = \\
&\left \{\v_i \rightarrow \v_o |  \v_i \in I, \v_o \in  \SET{C}_k\backslash\SET{I} \right \} \sqcup \left \{\v_i \rightarrow \v_o |  \v_i \in I, \v_o \in  neigh(\v) \cap neigh(\v_i)\backslash\SET{C}_k \right \}.
\end{align*}
Therefore, we can obtain the result of applying Meek function $M_{124}$ on set $\E$ with the following equation:
\begin{align*}
&M_{124}(\E) = \SET{T}_1 \sqcup \SET{T}_2 \sqcup \SET{T}_3 \sqcup \SET{T}_4.
\end{align*}
where,
\begin{align*}
&\SET{T}_1 = \underset{\v_o \in neigh(\v)\backslash\SET{C}_k}{\bigsqcup} DP [\v \rightarrow \v_o], \\
&\SET{T}_2 = \left \{\v_i \rightarrow \v_o |  \v_i \in \SET{I}, \v_o \in  \SET{C}_k\backslash\SET{I} \right \},\\
&\SET{T}_3 = \left(\underset{\v_i \in \SET{I}}{\bigsqcup} DP [\v_i \rightarrow \v] \bigsqcup \left \{\v_i \rightarrow \v_o |  \v_i \in I, \v_o \in  neigh(\v) \cap neigh(\v_i)\backslash\SET{C}_k \right \} \right ) \backslash \SET{T}_1, \\ 
&\SET{T}_4 = \left(\underset{\v_o \in \SET{C}_k\backslash\SET{I}}{\bigsqcup} DP [\v \rightarrow \v_o ]\right)\backslash \SET{T}_1.
\end{align*}
Now, our goal is to calculate $|M_{124}(\E)|$. To do so, we want to prove that the intersection of any two sets selected from the sets $\SET{T}_1,\SET{T}_2,\SET{T}_3,\SET{T}_4$ is empty. In order to show this, we first state the following Lemma \ref{lem:entercliqueorientation}.

\begin{lemma}
	\label{lem:entercliqueorientation}
	Consider a UCCG $G=(\V,\E)$ with set of nodes $\V$ and set of edges $\E$. Suppose that there exists a maximal clique $\SET{C}_k$ in the neighborhood of node $\v$. The following relations hold:
	\begin{itemize}
		\item For any  $\v_1 \line \v_2 \in \overline{\E}[\SET{C}_k]$ and $\v_i \in \SET{C}_k$, we have $ \v_1 \rightarrow \v_2 \notin DP [\v_i \rightarrow \v ]$.
		\item For any  $\v_1 \line \v_2 \in \overline{\E}[Neigh(v)]$ and  $\v_o \in neigh(\v)$, we have $\v_1 \rightarrow \v_2 \notin DP [\v \rightarrow \v_o ]\backslash \{\v \rightarrow \v_o \}$.
		\item For any  $\{\v_1 \line \v_2,\v_3 \line \v_4\} \subseteq \{\v \line \v_t|\v_t \in \SET{C}_k\} $, we have $\v_1 \rightarrow \v_2 \notin DP [\v_3 \rightarrow \v_4]$.

	\end{itemize}	
\end{lemma} 

\begin{proof}
	It can be shown that there exist a DAG in the Markov equivalence class of graph $G$, such that this DAG consists of some edges that are oriented from nodes in a clique in the neighborhood of node $\v$ to this node (see Proposition 6 in \cite{Hauser14}). We call such orientations in which there exist a DAG in the MEC as valid orientations. Based on this fact, we prove the relations by contradiction. For proving the first relation, assume that we have  $ \v_1 \rightarrow \v_2 \in DP [\v_i \rightarrow \v ]$. Thus, according to Lemma \ref{lem:PathFromSource}, there exists a directed path from node $\v$ to node $\v_2$. As $v_2$ and $v_i$ are in the maximal clique $\SET{C}_k$, we can have valid orientations of $v_2\rightarrow v$ and $v_i\rightarrow v$ without making a new v-structure. Hence, there would be a directed cycle $v\rightarrow \cdots \rightarrow v_2 \rightarrow v$ which is a contradiction. Therefore, there exists no DAG and we can imply that: $ \v_1 \rightarrow \v_2 \notin DP [\v_i \rightarrow \v ]$.
	
	
	In order to show the second relation, suppose that we have  $ \v_1 \rightarrow \v_2 \in DP [\v \rightarrow \v_o ]$. We partition the problem in three cases: $\v_2 = \v$, $\v_2 = \v_o$ and $\v_2 \in neigh(\v)\backslash \v_o$. In the first case, according to Lemma \ref{lem:PathFromSource}, there will be a directed path $\v \rightarrow \v_o ... \rightarrow \v$ which is a directed cycle. In the second case, according to Lemma \ref{lem:PathFromSource}, there will be a directed path $\v_o \rightarrow ... \rightarrow \v_1 \rightarrow \v_o$ which is a directed cycle. In the third case, having the valid orientations $v_2\rightarrow v$ and $v \rightarrow v_o$, we have a directed path $\v \rightarrow \v_o ...  \rightarrow \v_2  \rightarrow \v$ which is a directed cycle and there exists no DAG. This is a contradiction.
	
	For proving the third relation, assume we have  $\v_1 \rightarrow \v_2 \in DP [\v_3 \rightarrow \v_4]$. This means that orienting the edge $\v_3 \rightarrow \v_4$ orients the edge $\v_1 \rightarrow \v_2$, and we cannot have a DAG including directed edges $\v_3 \rightarrow \v_4$ and $\v_2 \rightarrow \v_1$. However, according to Proposition 6 in \cite{Hauser14}, such orientations are valid orientations and there exists a DAG, including those directed edges. Thus, this is a contradiction and the proof is complete.
\end{proof}

According to Lemma \ref{lem:entercliqueorientation}, we will show that intersection of sets $\SET{T}_1,\SET{T}_2,\SET{T}_3,\SET{T}_4$ is empty. Hence, the number of elements in $M_{124}(\E)$, i.e., $|M_{124}(\E) |$, can be computed by adding the number of elements in each of the sets $\SET{T}_1,\SET{T}_2,\SET{T}_3,\SET{T}_4$.

First, it can be seen that $\overline{\SET{T}}_2 \subseteq \overline{\E}[\SET{C}_k]$. Furthermore, based on the second relation in Lemma \ref{lem:entercliqueorientation}, we cannot find any edge $\v_1 \rightarrow \v_2 \in \SET{T}_1$ for nodes $\v_1,\v_2\in \SET{C}_k$. Therefore, we have $\overline{\SET{T}}_1 \cap \overline{\E}[\SET{C}_k] = \emptyset$. Additionally, we already removed all $\SET{T}_1$ elements from the sets $\SET{T}_3$ and $\SET{T}_4$. Thus, we can write:
\begin{align*}
&\SET{T}_1 \cap \left ( \SET{T}_2 \sqcup \SET{T}_3 \sqcup \SET{T}_4 \right) = \emptyset.
\end{align*}
In the following we prove that $\overline{\SET{T}}_3 \cap \overline{\E}[\SET{C}_k] = \emptyset$ and $\overline{\SET{T}}_4 \cap \overline{\E}[\SET{C}_k] = \emptyset$. Based on the second relation in Lemma \ref{lem:entercliqueorientation}, we cannot find any edge $\v_1 \rightarrow \v_2 \in \SET{T}_4$ for set of nodes $\v_1,\v_2\in \SET{C}_k$. Moreover, we can define the following two sets $\SET{T}_{3}^1$ and $\SET{T}_{3}^2$ such that $\SET{T}_3= \SET{T}_{3}^1\sqcup \SET{T}_{3}^2$:
\begin{align*}
&\SET{T}_{3}^1 = \left(\underset{\v_i \in \SET{I}}{\bigsqcup} DP [\v_i \rightarrow \v ]  \right ) \backslash \SET{T}_1, \\ 
&\SET{T}_{3}^2 = \left( \bigsqcup \left \{\v_i \rightarrow \v_o |  \v_i \in I, \v_o \in  neigh(\v) \cap neigh(\v_i)\backslash\SET{C}_k \right \} \right ) \backslash \SET{T}_1. \\ 
\end{align*}
We know that $\overline{\SET{T}}_{3}^2 \cap \overline{\E}[\SET{C}_k] = \emptyset$. It suffices to show $\overline{\SET{T}}_{3}^1 \cap \overline{\E}[\SET{C}_k] = \emptyset$. Again based on the first relation in Lemma \ref{lem:entercliqueorientation}, we cannot find any edge $\v_1 \rightarrow \v_2 \in \SET{T}_{3}^1$ for nodes $\v_1,\v_2 \in \SET{C}_k$. Hence, we have $\overline{\SET{T}}_3 \cap \overline{\E}[\SET{C}_k] = \emptyset$ and we can write:
\begin{align*}
&\SET{T}_2 \cap \left ( \SET{T}_3 \sqcup \SET{T}_4 \right) = \emptyset. \\
\end{align*}
Finally, we will show that: $
\SET{T}_3 \cap \SET{T}_4 = \emptyset.$
To do so, we prove the following two relations:
\begin{align}
&\overline{\SET{T}}_3 \subseteq \overline{\E}[Neigh(v)] \backslash \{ \v \line \v_o |\v_o \in \SET{C}_k\backslash\SET{I}\},\\
&\overline{\SET{T}}_4 \subseteq \{ \v \line \v_o |\v_o \in \SET{C}_k\backslash\SET{I}\}\sqcup \overline{\E}\backslash\overline{\E}[Neigh(v)].
\end{align}
For proving the relation (5), consider set $\SET{T}_3$. We know $\overline{\SET{T}}_3^2 \subseteq \overline{\E}[Neigh(v)] \backslash \{ \v \line \v_o |\v_o \in \SET{C}_k\backslash\SET{I}\}$. Thus, it suffices to show: $\overline{\SET{T}}_3^1  \subseteq \overline{\E}[Neigh(v)] \backslash \{ \v \line \v_o |\v_o \in \SET{C}_k\backslash\SET{I}\}$. Based on Proposition \ref{Prop:DP}, in the first step of DP calculation for each edge $\v_i \rightarrow \v$, we orient the edges inside the sub-graph  $\overline{\E}[Neigh(v)]$. The edges that are in $DP[v_i\rightarrow v]$ and $\overline{\E}[Neigh(v)]$, can be categorized into two types. In the first type, at least one of the endpoint of the edge is $v$. According to third relation in Lemma \ref{lem:entercliqueorientation}, for these edges like $v\rightarrow v_2$,   we cannot orient edges $v\line v_2$ for any $\v_2 \in \SET{C}_k$. Thus, we conclude that: $M_{14}(\{\v \rightarrow \v_2\} \sqcup \overline{\E}) \subseteq \SET{T}_1$. 
For the second type of the edges, none of the endpoints of edges is the node $v$. For this type of edges like $v_1\rightarrow v_2$, we cannot have $\{\v_1,\v_2\} \subseteq \SET{C}_k$ due to the first relation in Lemma \ref{lem:entercliqueorientation}. Furthermore, we cannot have $v_2 \in \SET{C}_k$. Since if we have $v_2 \in \SET{C}_k$, then $v_1\not\in \SET{C}_k$. Moreover, based on Proposition 6 in \cite{Hauser14}, by considering the valid orientation $v_2 \rightarrow \v$, we will have a cycle as $\{\v \rightarrow \v_1,\v_1 \rightarrow \v_2,\v_2 \rightarrow \v\}$. Additionally, for each edge of this type there is always an edge $v\rightarrow v_2$, where $v_2 \in neigh(v)\backslash\SET{C}_k$.
Therefore, it suffices to show that, $\E_1 \backslash \E_1[Neigh(v)] \subseteq M_{14}(\{\v \rightarrow \v_2\} \sqcup \overline{\E})$, where $\E_1 = M_{14}(\{\v_1 \rightarrow \v_2\} \sqcup \overline{\E})$. We depict all possible skeletons in Table \ref{Table:5}. Note that we consider the cases in which applying Meek function orients edges outside the sub-graph  $\overline{\E}[Neigh(v)]$. Therefore, in all of these cases $\v_3 \notin neigh(\v)$. Additionally, in all of these cases, the red edges are the ones that have been oriented as a result of applying Meek function $M_{14}$ on $\v_1 \rightarrow \v_2 $, and as can be seen, these edges are also oriented as a result of applying Meek functions $M_{14}$ on $\v \rightarrow \v_2 $. Based on the Proposition \ref{Prop:DP}, for discovering more edges' orientation outside the sub-graph  $\overline{\E}[Neigh(v)]$, we must compute DP function in the sub-graph $\overline{\E}[Neigh(v_3)]$ ,where only red edges are oriented. Thus, as we depict all possible cases, we conclude that every edges that will be oriented  outside the sub-graph  $\overline{\E}[Neigh(v)]$ as a result of $M_{14}(\{\v_1 \rightarrow \v_2\} \sqcup \overline{\E})$ is also in $M_{14}(\{\v \rightarrow \v_2\} \sqcup \overline{\E}) \subseteq \SET{T}_1$. Thus, we have $\overline{\SET{T}}_3 \subseteq \overline{\E}[Neigh(v)]$.
 In addition, based on the third relation in Lemma \ref{lem:entercliqueorientation}, no edge that is between node $\v$ and clique $\SET{C}_k$ will be oriented in $\SET{T}_3^1$. Thus, we can write:
\begin{align*}
\overline{\SET{T}}_3 \subseteq \overline{\E}[Neigh(v)] \backslash \{ \v \line \v_o |\v_o \in \SET{C}_k\backslash\SET{I}\}
\end{align*}
 

\begin{table}[]
	\begin{center}
		\caption{Different orientations for applying Meek function on $\v_i \rightarrow \v$}
		\label{Table:5}
		\begin{tabular}{|c|c|c|}
			\hline	
			\multicolumn{1}{|c|}{Case 1} & \multicolumn{1}{|c|}{Case 2} & \multicolumn{1}{|c|}{Case 3} \\ 
			\hline
			
			\begin{tikzpicture}[thick, scale=.6]
			\node (a) at (1,2.35) {$\v_i$};
			\node (a) at (-.45,1) {$\v_1$};
			\node (a) at (2.45,1) {$\v$};
			\node (a) at (-1.45,0) {$\v_4$};
			\node (a) at (1.6,0) {$\v_2$};
			\node (a) at (0,-1.45) {$\v_3$};
			\draw[] (1,2) circle (1.5pt);
			\draw[] (0,1) circle (1.5pt);
			\draw[] (2,1) circle (1.5pt);
			\draw[] (-1,0) circle (1.5pt);
			\draw[] (1,0) circle (1.5pt);
			\draw[] (0,-1) circle (1.5pt);
			\draw[midarrow]				(1,2) -- (2,1);
			\draw[]				(1,2) -- (0,1);
			\draw[midarrow]				(2,1) -- (1,0);
			\draw[]		(0,1) -- (2,1);
			\draw[]				(1,2) -- (2,1);
			\draw[midarrow]				(0,1) -- (1,0);
			\draw[midarrow=red]				(1,0) -- (0,-1);
			\draw[midarrow=red]				(-1,0) -- (0,-1);
			\draw[]				(-1,0) -- (1,0);
			\draw[midarrow]				(0,1) -- (-1,0);
			\end{tikzpicture}
			
			&
			
			\begin{tikzpicture}[thick, scale=.6]
			\node (a) at (1,2.35) {$\v_i$};
			\node (a) at (-.45,1) {$\v_1$};
			\node (a) at (2.45,1) {$\v$};
			\node (a) at (-1.45,0) {$\v_4$};
			\node (a) at (1.6,0) {$\v_2$};
			\node (a) at (0,-1.45) {$\v_3$};
			\draw[] (1,2) circle (1.5pt);
			\draw[] (0,1) circle (1.5pt);
			\draw[] (2,1) circle (1.5pt);
			\draw[] (-1,0) circle (1.5pt);
			\draw[] (1,0) circle (1.5pt);
			\draw[] (0,-1) circle (1.5pt);
			\draw[midarrow]				(1,2) -- (2,1);
			\draw[]				(1,2) -- (0,1);
			\draw[midarrow]				(2,1) -- (1,0);
			\draw[]		(0,1) -- (2,1);
			\draw[]				(1,2) -- (2,1);
			\draw[midarrowminus]				(0,1) -- (1,0);
			\draw[midarrow=red]				(1,0) -- (0,-1);
			\draw[midarrow=red]				(-1,0) -- (0,-1);
			\draw[]				(-1,0) -- (1,0);
			\draw[]				(0,1) -- (-1,0);
			\draw[]				(-1,0) -- (2,1);
			\end{tikzpicture}
			&
			\begin{tikzpicture}[thick, scale=.6]
			\node (a) at (1,2.35) {$\v_i$};
			\node (a) at (-.45,1) {$\v_1$};
			\node (a) at (2.45,1) {$\v$};
			\node (a) at (1.6,0) {$\v_2$};
			\node (a) at (0,-1.45) {$\v_3$};
			\draw[] (1,2) circle (1.5pt);
			\draw[] (0,1) circle (1.5pt);
			\draw[] (2,1) circle (1.5pt);
			\draw[] (1,0) circle (1.5pt);
			\draw[] (0,-1) circle (1.5pt);
			\draw[midarrow]				(1,2) -- (2,1);
			\draw[]				(1,2) -- (0,1);
			\draw[midarrow]				(2,1) -- (1,0);
			\draw[]		(0,1) -- (2,1);
			\draw[]				(1,2) -- (2,1);
			\draw[midarrow]				(0,1) -- (1,0);
			\draw[midarrow=red]				(1,0) -- (0,-1);
			\end{tikzpicture}
			\\
			\hline	
		\end{tabular}
	\end{center}
\end{table}

For relation (6), based on Lemma \ref{lem:entercliqueorientation}, the result of applying Meek function $M_{14}$ on set $\{\v \rightarrow \v_o\} \cup \overline{\E}$ cannot orient any edge inside the sub-graph $\overline{\E}[Neigh(v)]\backslash \{\v \rightarrow\v_o\}$. According to these two relations, we can write:
\begin{align*}
&\SET{T}_3 \cap \SET{T}_4 = \emptyset.
\end{align*}
To wrap up, the following equation holds:
\begin{align*}
&\SET{T}_1 \cap \left ( \SET{T}_2 \sqcup \SET{T}_3 \sqcup \SET{T}_4 \right) = \emptyset, \\
&\SET{T}_2 \cap \left ( \SET{T}_3 \sqcup \SET{T}_4 \right) = \emptyset, \\
&\SET{T}_3 \cap \SET{T}_4 = \emptyset.
\end{align*}
Hence, we will have:
\begin{align*}
|M_{124}(\E)| = |\SET{T}_1|+|\SET{T}_2| +|\SET{T}_3|+|\SET{T}_4|.
\end{align*}
Now, we can exactly compute the number of oriented edges after applying Meek function $M_{124}$ on set $\E$. We can compute $\SET{T}_1$ from the DP table. $\SET{T}_2$ can also be obtained easily. Now, we want to find the minimum value for the size of sets $\SET{T}_3$ and $\SET{T}_4$. First, we consider the set $\SET{T}_3$. Recall that the definitions of $P_j$ and $Q_j$ are:
\begin{align*}
&P_j \triangleq \underset{ |\SET{I}|=j,\SET{I}\subset \SET{C}_k\;\;\; } {min}\underset{\v_i \in \SET{I}} {max \;}\left | \left (DP [\v_i \rightarrow \v ] \underset{\v_o \in neigh(\v_i)\cap neigh(\v)\backslash\SET{C}_k}{\bigsqcup}  \{ \v_i \rightarrow \v_o  \}\right) \backslash\SET{T}_1 \right| \\
&Q_j \triangleq \underset{|\SET{O}|=j,\SET{O}\subset \SET{C}_k\;\;\; } {min}\underset{\v_o \in \SET{O}} {max \;} | DP [\v \rightarrow \v_o]\backslash\SET{T}_1|.
\end{align*}

Now, it can be seen that:
\begin{align*}
\underset{\v_i \in \SET{I}} {max \;} &\left | \left (DP [\v_i \rightarrow \v ] \underset{\v_o \in neigh(\v_i)\cap neigh(\v)\backslash\SET{C}_k}{\bigsqcup}  \{ \v_i \rightarrow \v_o  \}\right) \backslash\SET{T}_1 \right| \leq \\
&\left |  \left(\underset{\v_i \in \SET{I}}{\bigsqcup} DP [\v_i \rightarrow \v ] \bigsqcup \left \{\v_i \rightarrow \v_o |  \v_i \in I, \v_o \in  neigh(\v) \cap neigh(\v_i)\backslash\SET{C}_k \right \} \right ) \backslash \SET{T}_1 \right |.
\end{align*}

But the right hand side of the above inequality is $|\SET{T}_3|$. Hence, we can imply that $P_j\leq |\SET{T}_3|$. Similarly, we can show that $Q_j\leq |\SET{T}_4|$.


Finally, we will have:
\begin{align*}
&\left| \SET{T}_1 \right| = |\SET{R}|,\\
&\left| \SET{T}_2\right | = |\SET{I}|(|\SET{C}_k|-|\SET{I}|),\\
&\left| \SET{T}_3 \right| \geq P_j, \\ 
&\left|\SET{T}_4  \right| \geq Q_j.
\end{align*}
Thus, when the set $\SET{I}$ is known, we will have:
 \begin{align*}
&|M_{124}(\E)| \geq |\SET{R}| + P_{|I|} + Q_{(|\SET{C}_k|-|I|)} + |I|(|\SET{C}_k|-|I|)+(|\SET{C}_k|-2).
\end{align*}
Note that for obtaining $P_{|I|}$ and $Q_{(|\SET{C}_k|-|I|)}$, we just consider two edges between the $\SET{C}_k$ and $v$. Therefore, the last term is for those remaining $|\SET{C}_k|-2$ edges that are between clique and node $\v$ and they  will definitely be oriented after intervention. Note that, based on the second and third relation, these edges have not been considered in computing $Q_j$ and $P_j$ values and also they have not been oriented in set $\SET{R}$. If we consider all possible sets for set $\SET{I}$, where $|\SET{I}|=l$, we will have:
\begin{align*}
&|M_{124}(\E)| \geq L_{C} = |\SET{R}| +  \underset{l = \{1,...,|\SET{C}_k|-1\} } {\min    } P_l + Q_{|\SET{C}_k|-l} + |l|(|\SET{C}_k|-|l|)+(|\SET{C}_k|-2).
\end{align*}
Finally we can infer that:
\begin{align*}
&L(\SET{C}_k,v) = \min(L_{I} ,L_{C}).
\end{align*}

\subsection*{R. Proof of Theorem \ref{thm:LowerNode}}
After intervention on node $\v$, we will discover orientations of edges in the neighborhood of intervened node $\v$. We denote the obtained ICCG by $G^\prime=(\V,\E^\prime)$. Additionally, we denote set of nodes that have edges toward node $\v$ with $\SET{I}$. Furthermore, we denote the set of all maximal cliques in the neighborhood of node $\v$ by $\SET{C}(\v)$. As edges that are orientated from intervention do not generate a new v-structure in the graph, there exists a maximal clique $\SET{C}_k \in \SET{C}(\v)$, such that $\SET{I} \subseteq \SET{C}_k$.

After intervention, two cases can be considered: (a) All edges are outgoing from node $\v$ toward the nodes in the neighborhood of node $\v$ (b) At least there exists one edge from a node in the neighborhood of node $\v$ to node $\v$. In the first case, i.e., $\SET{I} = \emptyset$, using Theorem \ref{Thm:InterventionM124}, the number of oriented edges after intervention is equal to the following equation:
\begin{center}
	\begin{align*}
	|M_{124}(\E^\prime)| = \Bigg| \underset{\v_o \in neigh(\v)}{\bigsqcup} DP [\v \rightarrow \v_o]  \Bigg|.
	\end{align*}
\end{center}

In the second case, i.e., $\SET{I} \subseteq \SET{C}_k$ and $|\SET{I}|>0$, based on Lemma \ref{lem:LBMC}, the lower bound on number of oriented edges can be written as follows:
\begin{center}
	\begin{align*}
	|M_{124}(\E^\prime)| \geq \underset{\SET{C}_k \in \SET{C}(\v)}{min} L(\SET{C}_k,\v).
	\end{align*}
\end{center}

Hence, the lower bound on number of oriented edges after intervention on node $\v$, $L(v)$, can be calculated as the following:
\begin{center}
	\begin{align*}
	&|M_{124}(\E^\prime)| \geq L(v) = min \Bigg(   \Bigg| \underset{\v_o \in neigh(\v)}{\bigsqcup} DP [\v \rightarrow \v_o]  \Bigg|, 
	\underset{\SET{C}_k \in \SET{C}(\v)}{min} L(\SET{C}_k,\v)
	\Bigg).
	\end{align*}
\end{center}

\subsection*{S. Proof of Theorem \ref{thm:Consistency}}
We know the PCCG $G$ consists of the combination of both directed and undirected edges. For the ``only if" part, we know that there is a causal DAG with the same skeleton as graph $G$, including directed edges in $\E$. We can imply that there is no cycle in $\E$ and all edges in $\E$ are consistent. For the "if" part, we know that there is no cycle in $\E$ and the set $\E$ is consistent. We want to show that there will be a causal DAG with the same skeleton as graph $G$, including directed edges in $\E$.

To prove this, we will use an algorithm that constructs a DAG from our graph $G$, if the mentioned constraints are satisfied. Algorithm \ref{Alg:DAGConstruction} takes $G$ as its input and returns a causal DAG, including no v-structure if the constraints are satisfied. This algorithm guarantees that generated DAG contains all directed edges in $\E$ and has the same skeleton as $\overline{\E}$ and it has no v-structure. In Line 5 in Algorithm \ref{Alg:DAGConstruction}, we apply Meek functions $M_{124}$ on set $\E$. Through Lines 6-9, we orient undirected edges that are existed in set $\E$. In Line 7, we select an arbitrary node, and in Line 8, we orient all undirected edges in the neighborhood of this node as out-going edges. Thus, all the edges in the neighborhood of this node will be oriented and we denote the set of these oriented edges by $\SET{S}$. In Line 9, we obtain the result of applying Meek functions $M_{124}$ on set $\SET{S} \sqcup \overline{\E}$. Also, we add the oriented edges to the output DAG $\mathscr{G}$.

\begin{algorithm}
	\caption{Causal DAG Construction Algorithm}
	\label{Alg:DAGConstruction}
	\begin{algorithmic}[1]
		\State \textbf{Input:} PCCG $G=(\V,\E)$		
		\State \textbf{Output:} DAG $\mathscr{G}=(\V_\mathscr{G},\E_\mathscr{G})$

		\Function{$PDAG2DAG$}{$G$}
		\State $\V_\mathscr{G} \leftarrow \V$
		\State $\E_{\mathscr{G}} \leftarrow M_{124}(\E)$
		\While{$|\E| \neq | \overrightarrow{\E}_\mathscr{G}|$} 
		\State $v \leftarrow $ an arbitrary node from $\{v_l|   \exists\v_k \in {\V}_\mathscr{G},v_l \line v_k \in {\E}_\mathscr{G} \}$
		\State $\SET{S}\leftarrow\{\v_i\rightarrow \v|\v_i \rightarrow \v \in \E_\mathscr{G} \}\sqcup \{\v \rightarrow \v_o|\v \rightarrow \v_o \in \E_\mathscr{G} \text{ or } \v \line \v_o \in \E_\mathscr{G}\}$
		\State $\E_\mathscr{G} \leftarrow M_{124}({\E}_\mathscr{G}\sqcup \SET{S})$
		\EndWhile 
		\State \Return $\mathscr{G}$
		\EndFunction
	\end{algorithmic}
\end{algorithm}

Based on our assumptions, we know that PDAG $G$ has no cycle. Additionally, the following lemma expresses it has no v-structure.

\begin{lemma}
	\label{lem:vstructureconsistency} For any graph $G=(\V,\E)$, there is no v-structure in $\E$ if $\E$ is a consistent set.
\end{lemma} 
	
\begin{proof}
We prove this by contradiction. Suppose a v-structure like $\{v_i \rightarrow \v_j,\v_j \leftarrow \v_k\}$ exists in $\E$. Thus, we have $v_j \rightarrow \v_k \in DP [v_i \rightarrow \v_j]$. This is a contradiction because $\E$ is a consistent set.
\end{proof}
It suffices to show no new cycle and no new v-structure will be generated during running Algorithm \ref{Alg:DAGConstruction}. First, we will show that no v-structure will be appeared in the output graph of Algorithm \ref{Alg:DAGConstruction}. 

 To do so, we will prove that orienting further edges, either by applying Meek function $M_{124}$ on set $\E$ or adding new directed edges such as those in the constructing set $\SET{S}$ in Algorithm \ref{Alg:DAGConstruction}, does not generate any new v-structure. Before providing the proof, we state the following lemma that helps in proving mentioned statement.

\begin{figure}[ht]
	\begin{center}
		\begin{tabular}{ccc}
			
			\\
			\begin{tabular}{c}
				\begin{tikzpicture}[thick, scale=.6]
				\node (a) at (-.45,0) {$\v_s$};
				\node (b) at (1.5,1.4) {$\v_d$};
				\node (d) at (3.4,0) {$v_j$};
				\draw[fill=black] (0,0) circle (1.5pt);
				\draw[fill=black] (1.5,1) circle (1.5pt);
				\draw[fill=black] (3,0) circle (1.5pt);
				\draw[midarrow]   (0,0) -- (1.5,1);
				\draw[midarrow=blue]			  (1.5,1) -- (3,0);
				\end{tikzpicture}
			\end{tabular}	
			&
			
			\begin{tabular}{c}
				\begin{tikzpicture}[thick, scale=.6]
				\node (a) at (-.45,0) {$\v_s$};
				\node (b) at (1.5,1.4) {$\v_d$};
				\node (d) at (3.4,0) {$v_j$};
				\draw[fill=black] (0,0) circle (1.5pt);
				\draw[fill=black] (1.5,1) circle (1.5pt);
				\draw[fill=black] (3,0) circle (1.5pt);
				\draw[midarrow,blue] (1.5,1) -- (0,0);
				\draw[midarrow]		 (3,0) -- 	  (1.5,1);
				\end{tikzpicture}
			\end{tabular}	
			&
			
			\begin{tabular}{c}
				\begin{tikzpicture}[thick, scale=.6]
				\node (a) at (-.45,0) {$\v_s$};
				\node (b) at (1.5,1.4) {$\v_d$};
				\node (c) at (1.5,-1.4) {$v_i$};
				\node (d) at (3.4,0) {$v_j$};
				\draw[fill=black] (0,0) circle (1.5pt);
				\draw[fill=black] (1.5,-1) circle (1.5pt);
				\draw[fill=black] (1.5,1) circle (1.5pt);
				\draw[fill=black] (3,0) circle (1.5pt);
				\draw[midarrow]   (0,0) -- (1.5,1);
				\draw[]	  (0,0) -- (1.5,-1);
				\draw[midarrow,blue]			  (1.5,1) -- (3,0);
				\draw[midarrow,blue]			  (1.5,-1) -- (3,0);
				\draw[]			  (1.5,-1) -- (1.5,1);
				\end{tikzpicture}
			\end{tabular}	
			\\
			(a) & (b) & (c)
			\\
					\begin{tabular}{c}
				\begin{tikzpicture}[thick, scale=.6]
				\node (a) at (-.45,0) {$\v_s$};
				\node (b) at (1.5,1.4) {$\v_d$};
				\node (c) at (1.5,-1.4) {$v_i$};
				\node (d) at (3.4,0) {$v_j$};
				\draw[fill=black] (0,0) circle (1.5pt);
				\draw[fill=black] (1.5,-1) circle (1.5pt);
				\draw[fill=black] (1.5,1) circle (1.5pt);
				\draw[fill=black] (3,0) circle (1.5pt);
				\draw[midarrow,blue]   (1.5,1) -- (0,0);
				\draw[midarrow,blue]   (1.5,-1) -- 	  (0,0);
				\draw[midarrow]			(3,0)   --   (1.5,1);
				\draw[]			  (1.5,-1) -- (3,0);
				\draw[]			  (1.5,-1) -- (1.5,1);
				\end{tikzpicture}
			\end{tabular}
			&
			
			\begin{tabular}{c}
				\begin{tikzpicture}[thick, scale=.6]
				\node (a) at (-.45,0) {$\v_s$};
				\node (b) at (1.5,1.4) {$\v_d$};
				\node (c) at (1.5,-1.4) {$v_i$};
				\node (d) at (3.4,0) {$v_j$};
				\draw[fill=black] (0,0) circle (1.5pt);
				\draw[fill=black] (1.5,-1) circle (1.5pt);
				\draw[fill=black] (1.5,1) circle (1.5pt);
				\draw[fill=black] (3,0) circle (1.5pt);
				\draw[midarrow,blue]   (1.5,1) -- (0,0);
				\draw[midarrow,blue]   (1.5,-1) -- 	  (0,0);
				\draw[]			(3,0)   --   (1.5,1);
				\draw[midarrow]  (3,0)  -- (1.5,-1);
				\draw[]			  (1.5,-1) -- (1.5,1);
				\end{tikzpicture}
			\end{tabular}
	
			&
			\\
			(d) & (e) & 
		\end{tabular}
		\caption{reverse orientation of edges oriented by applying function $M_1$ or $M_4$}
		\label{Fig:OppositeDirection}
	\end{center}
\end{figure}
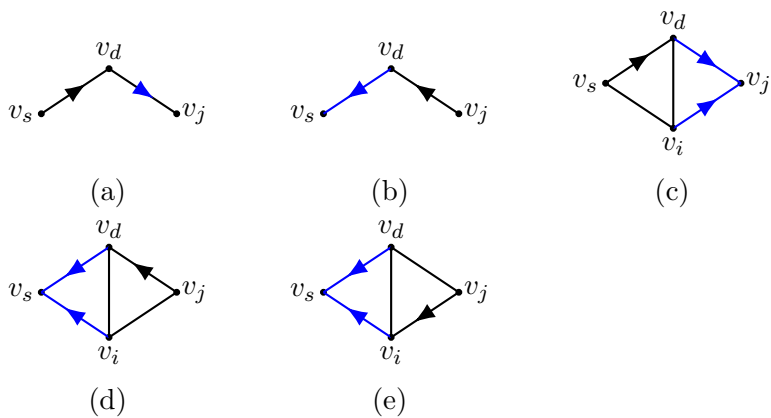

\begin{lemma}
	\label{lem:OppositeDirection} Consider a graph $G=(\V,\E)$ with set of nodes $\V$ and set of edges $\E$. If we have $\v_i \rightarrow \v_j \in DP [\v_k \rightarrow \v_l]$, we will have $\v_l \rightarrow \v_k \in DP [\v_j \rightarrow \v_i]$.
\end{lemma} 

\begin{proof}
	First, we prove that this lemma holds for candidate sub-graphs for Meek rule 1 and 4. Then, we extend this result for applying Meek function $M_{14}$ on a general graph. In Figure \ref{Fig:OppositeDirection}, the blue edges are oriented as the result of applying Meek function on black edges. We depict the candidate sub-graph for Meek rule 1 in the Figure \ref{Fig:OppositeDirection}(a). In this graph, the directed edge $\v_s \rightarrow \v_d$ orients the edge $\v_d \rightarrow \v_j$. According to Figure \ref{Fig:OppositeDirection}(b), if we orient the edge $\v_d \rightarrow \v_j$ in reverse direction, i.e., $\v_j \rightarrow \v_d$, the edge $\v_s \rightarrow \v_d$ will be oriented in the reverse direction, i.e., $\v_d \rightarrow \v_s$. 
	
	The candidate sub-graph for Meek rule 4 is depicted in Figure \ref{Fig:OppositeDirection}(c). In this sub-graph, by applying Meek function $M_4$, directed edge $\v_s \rightarrow \v_d$ will orient two edges $\v_i \rightarrow \v_j$ and $\v_d \rightarrow \v_j$. According to the graphs in Figure \ref{Fig:OppositeDirection}(d)(Figure \ref{Fig:OppositeDirection}(e)), if we orient the edge $\v_d \rightarrow \v_j$($\v_i \rightarrow \v_j$) in reverse direction, i.e., $\v_j \rightarrow \v_d$($\v_j \rightarrow \v_i$), the edge $\v_s \rightarrow \v_d$ will be oriented in the reverse direction, i.e., $\v_d \rightarrow \v_s$. Hence, we proved that this lemma holds for candidate sub-graphs of Meek function $M_1$ and $M_4$.
	
Based on the above observation, we will show the statement in the lemma holds in a general graph. Assume we know $\v_i \rightarrow \v_j \in DP [\v_k \rightarrow \v_l]$. The edge $\v_i \rightarrow \v_j$ has been oriented in a repetitive sequence of applying Meek function $M_1$ and $M_4$ in the corresponding candidate sub-graphs. Based on the above results, we can reveres the sequence of applying Meek functions on each candidate sub-graphs in order to orient the edge $\v_l \line \v_k$ as $\v_l \rightarrow \v_k$, and the proof is complete.
\end{proof}

Based on the Lemma \ref{lem:vstructureconsistency}, if we show that the set of directed edges after execution of Algorithm \ref{Alg:DAGConstruction} is consistent, we can conclude no v-structure exists in the output DAG. To do so, we prove following three facts: 1) The set of directed edges as a result of applying Meek function $M_{124}$ in Line 5 is consistent.
2) The set $\SET{S}$ is consistent with already directed edges. 3) The set $M_{124}({\E}_\mathscr{G}\sqcup \SET{S})$ is consistent.

1) We prove this by induction on the number of steps taken by applying Meek function $M_{124}$. For the base case, it is obvious that the edges in $\E$ are consistent. For the induction step, suppose that the directed edges up to step $r$ are consistent. We show that the directed edges will be remained consistent after step $r+1$.
We prove this by contradiction. Assume that we orient new edge $\v_i \rightarrow \v_j$ in step $r+1$ by applying one of the Meek functions and there exists an already directed edge $\v_r \rightarrow \v_t$, such that for some edges $\v_l \line \v_k \in \overline{\E}$  we have $\v_l \rightarrow \v_k \in DP [\v_i \rightarrow \v_j]$, but $\v_k \rightarrow \v_l \in DP [\v_r \rightarrow \v_t]$. Two cases can be considered:
\begin{itemize}
	\item In the first case, we have: $\v_i \rightarrow \v_j \in M_{14}(\E)$. Thus, there exists an edge $\v_q \rightarrow \v_w$ such that $\v_i \rightarrow \v_j \in DP [\v_q \rightarrow \v_w]$ and as a result $\v_l \rightarrow \v_k \in DP [\v_q \rightarrow \v_w]$. This is a contradiction, because we know the edges $\v_q \rightarrow \v_w$ and $\v_r \rightarrow \v_t$ are consistent.
	\item In this case, there exists a Meek rule 2 candidate sub-graph, such that we have:  $\v_i \rightarrow \v_j \in M_{2}(\{\v_i \rightarrow \v_m,\v_m \rightarrow \v_j,\v_i \line \v_j\})$. It can be shown that if $\v_l \rightarrow \v_k \in DP [\v_i \rightarrow \v_j]$, we will have $\v_l \rightarrow \v_k \in DP [\v_i \rightarrow \v_m] \sqcup DP [\v_m \rightarrow \v_j]$ according to Property 4 in Lemma \ref{lem:MeekPropertiesPCCG}. This is a contradiction, because the edges $\v_i \rightarrow \v_m$, $\v_m \rightarrow \v_j$ and $\v_r \rightarrow \v_t$ are consistent.
	\end{itemize}
	
	2) In this case, we want to prove that the oriented edges in $\SET{S}$ are consistent with already directed edges. We prove this by contradiction.  Assume we orient the edge  $\v \line \v_j$ as $\v \rightarrow \v_j$ in the procedure of orienting edges in $\SET{S}$ and there exists an already directed edge $\v_r \rightarrow \v_t$, such that $\v_l \rightarrow \v_k \in DP [\v \rightarrow \v_j]$ and $\v_k \rightarrow \v_l \in DP [\v_r \rightarrow \v_t]$. Based on the Lemma \ref{lem:OppositeDirection}, if we have $\v_l \rightarrow \v_k \in DP [\v \rightarrow \v_j]$, we will have $\v_j \rightarrow \v \in DP [\v_k \rightarrow \v_l]$. Having the relations $\v_j \rightarrow \v \in DP [\v_k \rightarrow \v_l]$ and $\v_k \rightarrow \v_l \in DP [\v_r \rightarrow \v_t]$, $\v_j \line \v$ should be already oriented as $\v_j \rightarrow \v$ when we apply Meek function $M_{124}$ in Line 5 and Line 9, which is a contradiction.
	
	3) In part 2, we stated that the union of set $\SET{S}$ and set ${\E}_\mathscr{G}$ is consistent. Thus, based on part 1, oriented edges as a result of applying Meek function $M_{124}$ on ${\E}_\mathscr{G} \sqcup \E$ is also consistent. Therefore, the proof is complete.

Next, we will show that no cycle will be generated in the output graph $\E_{\mathscr{G}}$ during the execution of Algorithm \ref{Alg:DAGConstruction}.
In the following lemma, we first prove that we do not have any directed cycle in this graph if there is no directed cycle of length three.

\begin{lemma}
	\label{lem:CycleInChordal} There is no directed cycle in a chordal graph if there exists no directed cycle of length three in that chordal graph.
\end{lemma} 
	
\begin{proof}
	Given a graph $G=(\V,\E)$ with set of nodes $\V$ and set of edges $\E$. Suppose there is a directed cycle $(\v_1,...,\v_i,...,\v_j,...,\v_1)$ with length $N$ in graph $G$. Without loss of generality, we assume that there is a chord between nodes $\v_i$ and $\v_j$. Thus, based on the direction of edges between nodes $\v_i$ and $\v_j$, we know one of the paths $(\v_1,...,\v_i,\v_j,...,\v_1)$ or $(\v_i,...,\v_j,\v_i)$ will be a directed cycle with length less than $N$. Thus, we can obtain a cycle with length less than $N$ from the original directed cycle. We can repeat this procedure until end up with a directed cycle of size three. Hence we can conclude that a cycle of size three exists in a choral graph with directed cycle.
\end{proof}

We know there is no cycle of length 3 in the set $\E$, and we will show that no such cycle will be generated during the execution of Algorithm 6. To do so, we need to prove three facts: 1) No cycle will be generated in Line 5.
2) No cycle will be generated during constructing $\SET{S}$ in Line 8. 3) No cycle will be created as the result of applying $M_{124}$ on ${\E}_\mathscr{G}\sqcup \SET{S}$. We will prove them in the sequel.

 Recall that the directed edges in $\E_{\mathscr{G}}$ remain consistent during execution of Algorithm \ref{Alg:DAGConstruction}. We will use this statement in the following proofs. 
 

1) We prove this by induction on the number of steps taken by applying Meek function $M_{124}$. For the base case, it is obvious that there is no cycle of length three in the set $\E$. For the induction step, suppose that there is no directed cycle of length three up to step $r$. We show that no such a cycle will be generated in step $r+1$.
We show this by contradiction. First, consider Meek function $M_2$. Applying this function on a Meek candidate sub-graph does not generate a cycle on that sub-graph. Therefore, we only consider the case in which the oriented edge as the result of applying Meek function $M_2$ in a candidate sub-graph makes cycle in another sub-graph. In that case we will have the following structure as $\{\v_i \rightarrow \v_j,\v_j \rightarrow \v_k,\v_k \rightarrow \v_l,\v_l \rightarrow \v_i,\v_j \line \v_l\}$. In such sub-graph, we have $\v_l \rightarrow \v_k \in DP[\v_i \rightarrow \v_j]$. This means the mentioned orientations violates the consistency, which is a contradiction.


The all sub-graphs that are able to make a new cycle after applying Meek function $M_1$ or $M_4$ are depicted in Table \ref{Table:ConstencycycleM14}. We check each of these cases and show that none of them can make a cycle. Case 1 and 2 belong to Meek function $M_1$, and the remaining two cases are  for Meek function $M_4$. Black edges are already oriented and the red edge is the one that applying specific Meek function on that edge will orient green edges. We intend to check whether the green edge makes a cycle or not. As the first case violates the consistency, it cannot be happened. In case 2, no cycle is generated. In case 3, similar to case 1, the consistency is violated. Finally, no cycle will be generated in case 4. 

\begin{table}[]
	\begin{center}
		\caption{Different cases that make cycle after applying Meek function $M_1$ or $M_4$}
		\label{Table:ConstencycycleM14}
		\begin{tabular}{|c|c|c|c|}
			\hline	
			\multicolumn{1}{|c|}{Case 1} &\multicolumn{1}{|c|}{Case 2} & \multicolumn{1}{|c|}{Case 3} & \multicolumn{1}{|c|}{Case 4} \\ 
			\hline	
			\multicolumn{2}{|c|}{Meek function $M_1$}  & \multicolumn{2}{|c|}{Meek function $M_4$} \\ 
			\hline
	
			\begin{tikzpicture}[thick, scale=.6]
			\node (a) at (0,1.45) {$\v_1$};
			\node (a) at (-1.45,0) {$\v_4$};
			\node (a) at (1.6,0) {$\v_2$};
			\node (a) at (0,-1.45) {$\v_3$};
			\draw[] (0,1) circle (1.5pt);
			\draw[] (-1,0) circle (1.5pt);
			\draw[] (1,0) circle (1.5pt);
			\draw[] (0,-1) circle (1.5pt);
			\draw[midarrow=red]				(0,1) -- (1,0);
			\draw[midarrow=green]				(1,0) -- (0,-1);
			\draw[midarrow]				(0,-1) -- (-1,0);
			\draw[midarrowminus]				(-1,0) -- (1,0);
			\draw[midarrowminus=green]			(1,0) -- 	(-1,0);
			\end{tikzpicture}
			
			&
					
			\begin{tikzpicture}[thick, scale=.6]
			\node (a) at (0,1.45) {$\v_1$};
			\node (a) at (-1.45,0) {$\v_4$};
			\node (a) at (1.6,0) {$\v_2$};
			\node (a) at (0,-1.45) {$\v_3$};
			\draw[] (0,1) circle (1.5pt);
			\draw[] (-1,0) circle (1.5pt);
			\draw[] (1,0) circle (1.5pt);
			\draw[] (0,-1) circle (1.5pt);
			\draw[midarrow=red]				(0,1) -- (1,0);
			\draw[midarrow=green]				(1,0) -- (0,-1);
			\draw[midarrow]				(0,-1) -- (-1,0);
			\draw[midarrow=green]			(1,0) -- 	(-1,0);
			\end{tikzpicture}
			
			&
		\begin{tikzpicture}[thick, scale=.6]
			\node (a) at (0,1.45) {$\v_1$};
			\node (a) at (-1.45,0) {$\v_4$};
			\node (a) at (1.6,0) {$\v_2$};
			\node (a) at (0,-1.45) {$\v_3$};
			\draw[] (0,1) circle (1.5pt);
			\draw[] (-1,0) circle (1.5pt);
			\draw[] (1,0) circle (1.5pt);
			\draw[] (0,-1) circle (1.5pt);
			\draw[midarrow=red]				(0,1) -- (1,0);
			\draw[midarrowminus=green]				(1,0) -- (0,-1);
			\draw[midarrowminus]				(0,-1) -- (-1,0);
			\draw[midarrowminus=green]			(-1,0)	 -- (0,-1);
			\draw[midarrow]				(-1,0) -- (1,0);
			\draw[]				(0,1) -- (-1,0);
			\end{tikzpicture}
			
			&
		\begin{tikzpicture}[thick, scale=.6]
			\node (a) at (0,1.45) {$\v_1$};
			\node (a) at (-1.45,0) {$\v_4$};
			\node (a) at (1.6,0) {$\v_2$};
			\node (a) at (0,-1.45) {$\v_3$};
			\draw[] (0,1) circle (1.5pt);
			\draw[] (-1,0) circle (1.5pt);
			\draw[] (1,0) circle (1.5pt);
			\draw[] (0,-1) circle (1.5pt);
			\draw[midarrow=red]				(0,1) -- (1,0);
			\draw[midarrow=green]				(1,0) -- (0,-1);
			\draw[midarrow=green]				(-1,0) -- (0,-1);
			\draw[midarrow]				(-1,0) -- (1,0);
			\draw[]				(0,1) -- (-1,0);
			\end{tikzpicture}
			
			\\
			\hline	
		\end{tabular}
	\end{center}
\end{table}

2) In this case, we orient some edges connected to node $v$. Thus, for making a  cycle $C$ by orienting some edges during the $\SET{S}$ construction, two case can be considered. In addition of already directed edges, one (the first case) or two (the second case) edges are oriented in $C$ in the procedure of $\SET{S}$ construction. The first case cannot be occurred because it means we have a Meek candidate sub-graph $M_2$ while the Meek functions $M_{124}$ in Line 5 or Line 9 has already been applied. In the second case, if both edges are oriented, they are out-going edges and cannot make a cycle.

3) In this part, we prove that the set of directed edges in the result of applying Meek function $M_{124}$ on $\SET{S}\sqcup {\E}_\mathscr{G}$ cannot generate a cycle. To do so, we recall from part 1 that applying Meek function $M_{124}$ on a set with the consistent directed edges cannot make a cycle. As we showed, the union of set $\SET{S}$ and set ${\E}_\mathscr{G}$ is consistent, no cycle exists in the result of applying Meek function $M_{124}$ on this set.

All in all, we proved that for any $\E$ as an input that is satisfying our assumptions, the set of directed edges after each iteration of the algorithm are subset or equal of a causal DAG in true MEC. As the number of edges are finite, the algorithm terminates and return a causal DAG in true MEC and the proof is complete.

\newpage
\vskip 0.2in

\bibliography{main}

\end{document}